\pdfoutput=1
\documentclass{article}
\usepackage{tkai}
\usepackage{float}
\usepackage[T1]{fontenc}
\usepackage[utf8]{inputenc}
\usepackage{times}  
\usepackage{helvet} 
\usepackage{courier}  
\usepackage[hyphens]{url}  
\usepackage{graphicx} 
\usepackage{caption} 
\usepackage{caption} 
\usepackage{algorithm}
\usepackage{algorithmic}
\usepackage{enumitem}
\usepackage{booktabs}       
\usepackage{amsfonts}       
\usepackage{nicefrac}      
\usepackage{xcolor}         
\usepackage{amsmath}
\usepackage{amssymb}
\usepackage{color}
\usepackage{graphics}
\usepackage{lipsum}
\usepackage{subcaption}

\usepackage{breqn}
\newenvironment{proof}{\paragraph{Proof:}}{\hfill$\square$}

\usepackage{newfloat}
\usepackage{listings}
\usepackage{bibentry}
\usepackage{titlesec}
\DeclareCaptionStyle{ruled}{labelfont=normalfont,labelsep=colon,strut=off} 
\floatstyle{ruled}
\newfloat{listing}{tb}{lst}{}
\floatname{listing}{Listing}

\titleclass{\subsubsubsection}{straight}[\subsubsection]
\newcounter{subsubsubsection}[subsubsection]
\renewcommand\thesubsubsubsection{\thesubsubsection.\arabic{subsubsubsection}}
\titleformat{\subsubsubsection}
  {\normalfont\normalsize\itshape}{\thesubsubsubsection}{1em}{}
\titlespacing*{\subsubsubsection}{0pt}{0.5em}{0.5em}

\title{TeLU Activation Function for Fast and Stable Deep Learning}
\author{%
  Alfredo Fernandez \\
  University of South Florida\\
  \texttt{alfredo.fernandez2021@outlook.com}
  \And
  Ankur Mali \\
  University of South Florida\\
  \texttt{ankurarjunmali@usf.edu}
}

\newtheorem{Thm}{Theorem}[section]
\newtheorem{Lem}{Lemma}[section]

\newtheorem{Def}{Definition}[section]
\newtheorem{Prop}{Property}[section]

\begin{document}

\maketitle

\begin{abstract}
We propose the Hyperbolic Tangent Exponential Linear Unit (TeLU), a neural network hidden activation function defined as $TeLU(x)=x \cdot tanh(e^x)$. TeLU’s design is grounded in the core principles of key activation functions, achieving strong convergence by closely approximating the identity function in its active region while effectively mitigating the vanishing gradient problem in its saturating region. Its simple formulation enhances computational efficiency, leading to improvements in scalability and convergence speed. Unlike many modern activation functions, TeLU seamlessly combines the simplicity and effectiveness of ReLU with the smoothness and analytic properties essential for learning stability in deep neural networks. TeLU’s ability to mimic the behavior and optimal hyperparameter settings of ReLU, while introducing the benefits of smoothness and curvature, makes it an ideal drop-in replacement. Its analytic nature positions TeLU as a powerful universal approximator, enhancing both robustness and generalization across a multitude of experiments. We rigorously validate these claims through theoretical analysis and experimental validation, demonstrating TeLU's performance across challenging benchmarks\footnotemark{}; including ResNet18 on ImageNet, Dynamic-Pooling Transformers on Text8, and Recurrent Neural Networks (RNNs) on the Penn TreeBank dataset. These results highlight TeLU’s potential to set a new standard in activation functions, driving more efficient and stable learning in deep neural networks, thereby accelerating scientific discoveries across various fields.\footnotetext{The complete code for all experiments is available at https://github.com/alfredofernandez2021/TeLU-Experiments.}

\end{abstract}

\newpage
\section{Introduction}
\label{intro}
The rapid advancements in deep learning have significantly expanded the capabilities of machine learning, enabling machines to perform tasks ranging from image recognition to natural language processing with remarkable accuracy \cite{lecun2015deep, StrongGradients1, Transformer}. Central to these advancements are neural networks, computational models inspired by the human brain, which consist of layers of interconnected neurons. Each neuron applies a mathematical operation to its inputs and passes the result to the next layer, enabling the network to learn and generalize from data \cite{Rumelhart1986LearningRB}. As neural networks have grown in complexity and depth, their success has increasingly depended on the subtle choices made in their architecture, including the selection of activation functions \cite{glorot_init}.

Activation functions play a pivotal role in neural networks by introducing non-linearity into the model, allowing it to learn complex patterns and representations that would be impossible for a purely linear model \cite{sharma2017activation}. Traditional activation functions such as the logistic sigmoid \cite{Rumelhart1986LearningRB}, hyperbolic tangent (tanh), and rectified linear unit (ReLU) \cite{ReLU} have been foundational in the growth of deep learning. However, despite their widespread use, these functions come with significant limitations. Sigmoid and Tanh functions are prone to vanishing gradient problems, particularly in deep networks, which can stall the learning process \cite{vanishinggradientsigmoid, BengioVanishingGradient}. ReLU, while exhibiting improved convergence due to its linear active region, is prone to dead neurons; where units stop contributing to the learning process due to negative inputs being zeroed out \cite{LReLU, dyingrelu}.

Modern activation functions, such as the Exponential Linear Unit (ELU) \cite{elu}, mitigate the dying neuron issue and reduce the forward bias shift effects present in positive semi-definite non-linearities such as ReLU by approaching an odd symmetry near the origin. The utilization of negative activation pushes the expected activation closer to 0, resulting in effects that resemble batch normalization and thereby speeds up learning \cite{Amari1998}. The Gaussian Error Linear Unit (GELU) \cite{GELU} uses its non-monotonic growth to deliver a comparable symmetry near the origin with the added benefit of having a deactivation that saturates to zero. This gradual deactivation functions like a deterministic Gaussian form of standard neuron dropout \cite{dropout}, where neurons selectively participate based on the presence of features rather than their absence, enhancing the model's robustness \cite{xie2021smoothadversarialtraining}. 

Despite these innovative developments, ReLU remains the most popular general-purpose nonlinearity, largely due to its simplicity and rapid convergence \cite{ReLUtrends, relusparse}. We believe that ReLU's widespread use in feed-forward neural networks has established it as the standard benchmark against which newer activation functions are compared. However, these newer functions often lack ReLU’s computational efficiency and effective gradient propagation \cite{DUBEY202292}. We also observed that recently proposed activation functions typically do not perform optimally under the same hyperparameter settings as ReLU \cite{Smish}. As a result, these modern nonlinearities can appear as inefficient alternatives that do not offer consistent improvements over ReLU.

Driven by this hypothesis, we set out to develop an activation function that not only preserves the enduring strengths of ReLU but also introduces innovations that exceed the benefits provided by modern activation functions. Recognizing the robustness and optimizer compatibility of smooth functions \cite{xie2021smoothadversarialtraining, SecondOrder1}, we focused on discovering an analytic non-linearity \cite{analytic}. In this paper, we present the Hyperbolic Tangent Exponential Linear Unit (TeLU), an activation function that retains ReLU’s rapid convergence while addressing key challenges such as the vanishing gradient problem and learning instability. TeLU also enhances robustness and provides greater stability during learning, making it a powerful and versatile tool for advancing deep neural networks.

We validate these claims through extensive empirical evaluations across various benchmark datasets and architectures. We demonstrate that TeLU outperforms traditional activation functions in terms of both training efficiency and final model accuracy. Moreover, our analysis reveals that this new activation function offers a more robust and stable training process, particularly in deep and complex networks. These findings suggest that TeLU could be a valuable addition to the toolkit of deep learning practitioners, paving the way for more powerful and effective neural network models.

The contributions of this paper are as follows:

\begin{itemize}
\item \textit{Proposal of the Hyperbolic Tangent Exponential Linear Unit (TeLU):} An artificial neuron activation function that integrates and optimizes the beneficial characteristics of previous nonlinearities.
\item \textit{Comprehensive Theoretical Analysis:} A detailed examination and comparison of popular activation functions' near-linear behavior in the active region, persistent gradient in the saturation region, runtime efficiency, universal approximation, and stability properties that theoretically demonstrate TeLU's advantages.
\item \textit{Extensive Experimental Evaluation:} A wide-ranging set of experimental validations of TeLU's distinctive theoretical advantages conducted on Multi-Layer Perceptron (MLP), Convolutional Neural Network (CNN), dynamic-pooling transformer, Recurrent Neural Network (RNN), and Variational AutoEncoder (VAE) architectures using competitive datasets such as ImageNet and Text8.
 \end{itemize}


Section 2 explores related activation functions, detailing the sequence of innovations in activation functions. Section 3 discusses the limitations of existing activation functions, highlighting the need for an alternative solution. Section 4 outlines our design process and introduces the formulation of TeLU. Section 5 delves the theoretical basis for TeLU's advantages. Subsection 5.1 examines the vanishing gradients of popular activation functions for increasingly negative inputs and describes how TeLU mitigates this issue. Subsection 5.2 explains how TeLU approximates the identity function for positive inputs, resulting in faster convergence compared to other activation functions. Subsection 5.3 demonstrates the computational efficiency of TeLU’s simple formulation. Subsection 5.4 discusses TeLU’s operational similarities to ReLU, enabling compatibility with existing ReLU projects. Subsection 5.5 proves that TeLU architectures are analytic universal approximators, outlining the corresponding stability and robustness benefits. Finally, Subsection 5.6 addresses learning stability, drawing on heuristics from literature to illustrate how TeLU enhances stability in deep neural network training.

Section 6 presents experimental validation of the theoretical results discussed in Section 5. Subsection 6.1 shows that TeLU's mitigation of the vanishing gradient problem is crucial for recovering from strong vanishing gradient conditions as well as improving accuracy in convolutional neural networks. Subsection 6.2 provides extensive evidence of TeLU's enhanced convergence properties using the ImageNet dataset with the ResNet34 architecture and the Text8 dataset with the Dynamic-Pooling Transformer architecture. Subsection 6.3 benchmarks the runtime of popular activation functions, illustrating TeLU's computational efficiency across various systems. Subsection 6.4 highlights the tuning similarities between TeLU and ReLU, demonstrating that configurations optimized for ReLU yield superior accuracy when applied to TeLU. Subsection 6.5 showcases the benefits of TeLU’s analytic universal approximation in Variational AutoEncoders (VAEs), Recurrent Neural Networks (RNNs), and robustness benchmarks. Finally, Subsection 6.6 demonstrates TeLU's unmatched learning stability across diverse architectural variations, initialization methods, and optimizers.


Section 7 provides a detailed discussion of the value of the TeLU activation function, combining both theoretical insights and experimental findings to illustrate its advantages in neural network training. This section delves into how TeLU addresses challenges such as improved convergence, vanishing gradient mitigation, computational efficiency, and compatibility with existing training configurations, highlighting its practical implications across various architectures. Building on this, Section 8 concludes the research by revisiting the main objectives and contributions, emphasizing TeLU’s role as an effective drop-in replacement for ReLU. Section 9 looks forward to future research directions, proposing efforts to expand theoretical guarantees and conduct further experimental validations. It suggests exploring TeLU’s performance in more diverse architectures and refining its mathematical properties, setting the stage for ongoing advancements in the deep learning community. Appendix A includes supplementary tables related to our theoretical analysis and experimental setup. Appendix B provides additional theorems that offer deeper insights into the approximation, stability, and convergence guarantees of TeLU.

\newpage
\section{Related Work}

This section provides an abridged history of activation function development, starting with the early models like the step and sigmoid functions, which paved the way for more sophisticated approaches. It explores how innovations such as the Rectified Linear Unit (ReLU) overcame key limitations of previous models, improved training efficiency and convergence in deep networks. Subsequently, we will examine more recent advancements, such as the Exponential Linear Unit (ELU), Sigmoid Linear Unit (SiLU), and Gaussian Error Linear Unit (GELU), which further addressed issues like dead neurons and output bias with smooth nonlinearities. Each innovation reflects a progression in the understanding of how activation functions can enhance neural network learning dynamics. By evaluating these historical developments and their contributions, we gain valuable insight into the broader landscape of activation functions and their critical role in driving the ongoing advancements in deep learning.

\subsection{Early Activation Functions}

In neuroscience, it was once believed that biological neurons activated according to the Heaviside unit step function \cite{McPitts}, which is either active with output 1 or inactive with output 0. As a differentiable approximation to the Heaviside function, the Logistic Sigmoid function $\sigma (x) = \frac{1}{1+e^{-x}}$ \cite{Rumelhart1986LearningRB} became prevalent in early neural networks, offering a practical activation function capable of universal approximation \cite{Cybenko1989Approximation, HORNIK1989359, FUNAHASHI1989183}. However, the Logistic Sigmoid function suffers from the vanishing gradient problem as pre-activations approach $\pm \infty$ and activations saturate toward 0 or 1. The vanishing gradient problem can significantly hinder the training of deep architectures. During backpropagation, the small gradients from upstream neurons shrink the gradients of downstream activations due to the chain rule in calculus \cite{backprop}. Consequently, parameter updates in the earlier layers become negligible when enough upstream neurons have gradients that approach zero.

The Hyperbolic Tangent (tanh) activation function, defined as $tanh(x)=\frac{e^x - e^{-x}}{e^x + e^{-x}}$, helps mitigate the vanishing gradient problem by providing larger gradients than the Logistic function for pre-activation values around 0. This effect is nearly guaranteed in practice when appropriate L1 or L2 weight regularization is used. Additionally, the Hyperbolic Tangent is symmetric about the origin, resulting in an average activation value closer to 0. This symmetry leads to improved convergence efficiency compared to positive functions like the Logistic Sigmoid \cite{LecunSymmetricAdvantage.66.2396, LeCun1998, raiko2012deep, Schraudolph2012}.

\subsection{Rectified Linear Units}

The Rectified Linear Unit (ReLU), defined as $ReLU(x) = max(0,x)$ and shown in figure \ref{fig:RELUFigure}, was originally discovered by Fukushima \cite{ReLU}. It was later recognized for its benefits in accelerating convergence due to its computational efficiency, near-linearity, and reduced vanishing gradient problem \cite{relusparse}. ReLU's simplicity in computation and differentiation leads to faster training times and improved scalability of neural architectures. Its linear behavior for positive inputs provides strong gradients and a faster convergence rate, though this comes with the trade-off of significant output bias \cite{LecunSymmetricAdvantage.66.2396}.

In the negative domain, ReLU neurons are inactive, resulting in sparse activations that help reduce overfitting. This effectively acts as a form of built-in deterministic dropout regularization \cite{dropout}. However, this inactivity can also lead to a permanent reduction in the network's representational capacity, a phenomenon known as the "dying ReLU" problem. Research by Lu et al. \cite{lu2019dying} shows that the likelihood of a ReLU network experiencing neuron inactivity increases with network depth and decreases with network width. This makes deep ReLU architectures particularly challenging to train unless width is scaled accordingly, which in turn increases computational costs \cite{complexitydeeplearning}. Additionally, the piecewise linear nature of ReLU introduces robustness challenges \cite{goodfellow2015explainingharnessingadversarialexamples}: Input perturbations can deactivate essential neurons or activate previously inactive neurons that did not participate in training, potentially compromising the stability and performance of the model.

\begin{figure}
    \centering
    \includegraphics[width=0.75\linewidth]{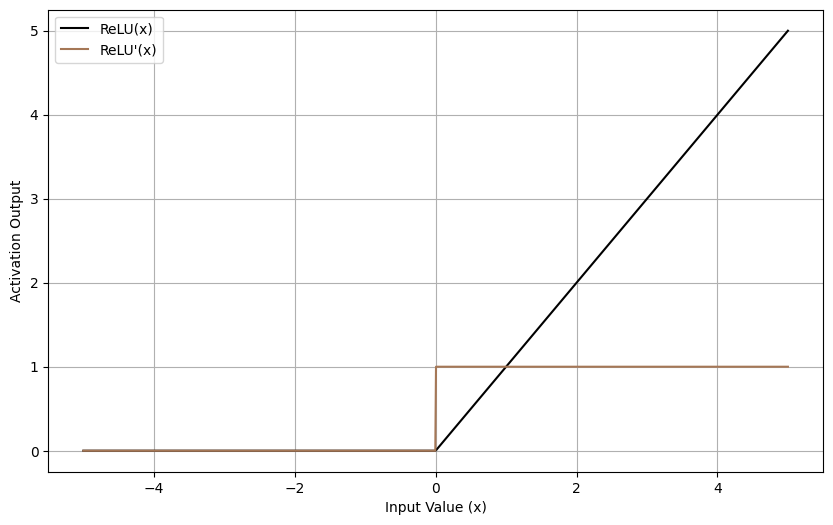}
    \caption[ReLU Activation Function and its First-Order Derivative.]%
    {\textbf{ReLU Activation Function and its First-Order Derivative.} Plot of $ReLU(x) = max(0,x) = \{x<0 : 0 , x\}$ and its first order derivative. For negative inputs, ReLU is 0. For positive inputs, ReLU is defined as the identity $x$.}
    \label{fig:RELUFigure}
    \vspace{1.0\baselineskip}
\end{figure}

\subsection{ReLU Variants}

Despite its drawbacks, ReLU remains a popular default choice for hidden activation functions and has inspired numerous variants. The Leaky ReLU (LReLU) \cite{LReLU}, defined as $LReLU(x)=max(0.01x,x)$, addresses the issue of inactive neurons by assigning a small constant slope to negative inputs. Similarly, Parametric ReLU (PReLU) \cite{PReLU}, $LReLU(x)=max(\alpha x,x)$, treats the slope of negative inputs as learnable parameter $\alpha$. While introducing additional learnable parameters in PReLU can increase the risk of overfitting, this can be managed with aggressive regularization and data augmentation techniques.

Random Leaky ReLU (RReLU) \cite{RReLU}, $RReLU(x) = max( \beta x,x)$, where $ \beta \sim \mathcal{U}(l,u)$ and $l,u \in [0,1)$ avoids the overhead of additional overfitting prevention by randomly setting the slope of the negative region using a uniform distribution over small positive values. This randomness introduces an implicit form of regularization, allowing RReLU to achieve lower training accuracies than ReLU, LReLU, and PReLU. These variants trade ReLU's deactivation region for unbounded behavior in an attempt to avoid the dying neuron problem. However, piecewise linear functions like these have been found to exhibit poor robustness \cite{xie2021smoothadversarialtraining}, making them susceptible to adversarial noise. As a result, smooth approximations of ReLU, such as $Softplus(x) = ln(1+e^x)$, may be more suitable for building secure applications.

\subsection{Exponential Linear Units}

Smooth nonlinearities also offer a significant advantage in their compatibility with second-order optimization algorithms like Natural Gradient Descent \cite{SecondOrderOptimization, shrestha2023naturalgradientmethodsperspectives}. Unlike first-order optimizers, which rely only on the gradient of the error function, second-order optimizers utilize the Hessian matrix, providing curvature information about the error surface. This additional information allows for more precise updates to model parameters, resulting in faster and more efficient convergence during training \cite{backprop}. 

The Exponential Linear Unit (ELU) \cite{elu} attempts to approach the Natural Gradient \cite{shrestha2023naturalgradientmethodsperspectives} in learning without the use of second-order optimizers by reducing the output bias commonly found in popular activation functions. This reduction in output bias results in a Fisher information matrix \cite{fisher1925theory} with small off-diagonal entries \cite{raiko2012deep}. As a result, ELU helps bring neural networks closer to Fisher Optimal Learning \cite{Amari1998}, where learning is guided by the natural gradient for improved convergence \cite{Schraudolph2012, LeCun1998, Lecun1991}. ELU is defined as the identity function for positive inputs, preserving the rapid convergence of the Rectified Linear Unit (ReLU). For negative inputs, ELU is defined as $ELU(x) = e^x - 1$, producing a small negative value instead of zeroing out the activation. Consequently, ELU saturates toward -1 rather than 0 as preactivation grows negative, meaning that ELU neurons are excluded from inference only when their preactivation is close to 0. This design sacrifices ReLU’s sparsity and deactivation in favor of capturing negative features. Although ELU has a continuous first derivative, its second derivative is not continuous, making it unsuitable for second-order optimization methods.


\begin{figure}
    \centering
    \includegraphics[width=0.75\linewidth]{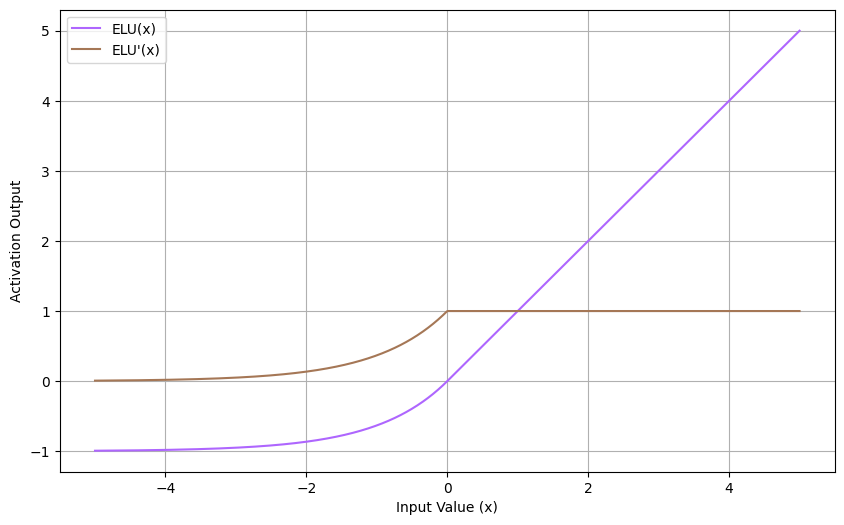}
    \caption[ELU Activation Function and its First-Order Derivative.]%
    {\textbf{ELU Activation Function and its First-Order Derivative.} Plot of $ELU(x)= {x<0 : e^x - 1 , x}$ and its first order derivative. For negative inputs, ELU is defined by $e^x - 1$ and thus saturates to an output of -1 as $x \to -\infty$. For positive inputs, ELU is defined as the identity $x$.}
    \vspace{1.0\baselineskip}
\end{figure}

\subsection{Sigmoid Linear Units}


Unlike ELU, smooth non-monotonic ReLU approximations like SiLU retain a dense deactivation region while offering the advantage of continuous, non-zero differentiability. The Sigmoid Linear Unit (SiLU)\cite{Ramachandran2017SwishAS}, also referred to as the Swish activation function, was discovered through Neural Architecture Search (NAS) \cite{NeuralArchitectureSearch}. Mathematically, SiLU is defined as $SiLU(x)=x \cdot \sigma(x)$, where $\sigma(x) = \frac{1}{1+e^{-x}}$. Its discovery involved constructing trees from univariate and bivariate functions, with a focus on balancing computational efficiency and expressivity. During the search, it was observed that simpler functions which could be generalized as linearly-scaled cumulative distribution functions (LSCDFs), were often favored. SiLU’s smooth, non-monotonic nature results in smoother optimization landscapes, aiding in the effective training of deeper networks by avoiding issues like sharp transitions and local minima. SiLU's seminal research \cite{Ramachandran2017SwishAS} has shown that SiLU outperforms traditional activation functions, such as ReLU and Leaky ReLU, across various tasks. These findings sparked the growing recognition of smooth, non-monotonic activation functions in deep learning \cite{GELU, Mish}.

\begin{figure}
    \centering
    \includegraphics[width=0.75\linewidth]{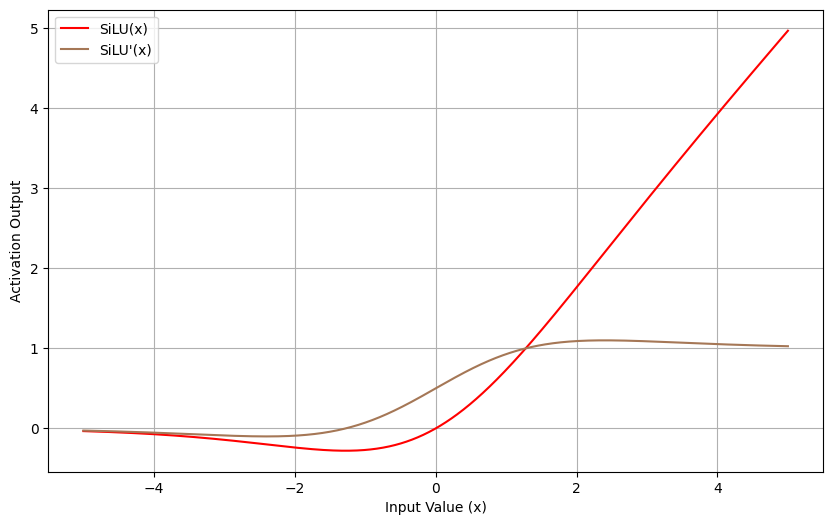}
    \caption[SiLU Activation Function and its First-Order Derivative.]
    {\textbf{SiLU Activation Function and its First-Order Derivative.} Plot of $SiLU(x)= x \cdot \frac{1}{1+e^{-x}}$ and its first order derivative. As input $x \to -\infty$, SiLU saturates to a deactivation output of 0. As input $x \to \infty$, SiLU's active region approximates the identity $x$.}
    \label{fig:SiLUFigure}
\end{figure}

\subsection{Gaussian Error Linear Units}

Hendrycks and Gimpel recognized the mathematical similarities between the inactive regions of LSCDFs, such as ReLU and SiLU, and the effects of dropout regularization \cite{GELU}. The formulation of GELU was a reinterpretation of the SiLU activation function, where a scaled error function was employed in place of the logistic cumulative distribution function found within SiLU. The result, $\frac{x}{2} \cdot [1+erf(x/ \sqrt{2})]$ theoretically allows GELU to leverage soft pseudo-ensembles \cite{pseudoensembles}, where negative activations are neglected and positive activations are preserved. Pseudo-ensembles refer to subsets of a neural network that temporarily become inactive, enabling network subgraphs to learn to perform tasks independently like traditional ensemble methods \cite{ensemblesReview}. This approach tends to enhance a model's robustness by benefiting from the many-advisors principle, common in ensemble machine learning models. Supporting this strategy, the non-monotonic GELU activation function has been shown to outperform both ReLU and ELU in the MNIST \cite{MNIST}, CIFAR-10, and CIFAR-100 \cite{cifar} datasets.

\begin{figure}
    \centering
    \includegraphics[width=0.75\linewidth]{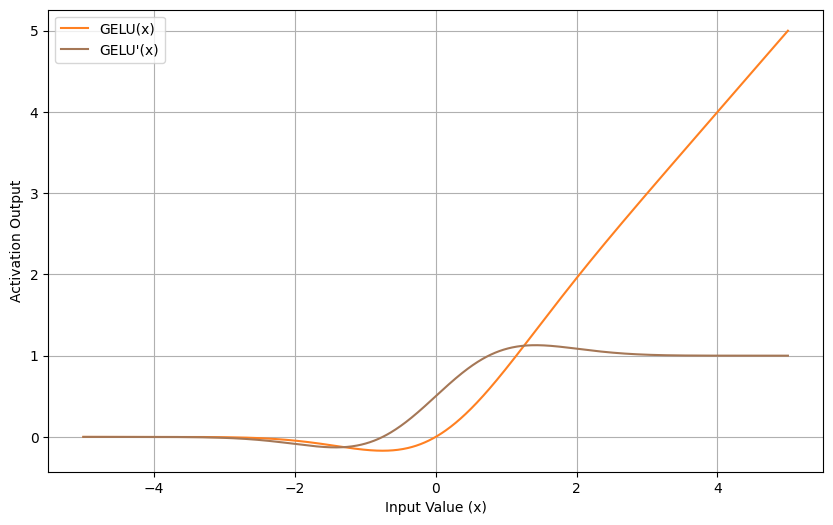}
    \caption[GELU Activation Function and its First-Order Derivative.]
    {\textbf{GELU Activation Function and its First-Order Derivative.} Plot of $GELU(x)= \frac{x}{2} \cdot (1+erf( \frac{x}{\sqrt{2}}))$ and its first order derivative. As input $x \to -\infty$, GELU saturates to a deactivation output of 0. As input $x \to \infty$, GELU's active region approximates the identity $x$.}
    \label{fig:GELUFigure}
    \vspace{1.0\baselineskip}
\end{figure}

\subsection{The Mish Activation Function}

The Mish activation function \cite{Mish}, defined as $Mish(x) = x \cdot tanh(ln(1+e^x))$, introduces a new LSCDF that is similar to SiLU and GELU. Mish was found while observing the causes of the beneficial SiLU properties, and was validated against other competing formulations with preliminary experimentation. The cumulative distribution function used, $tanh(ln(1+e^x))$, allows the Mish non-linearity to exhibit self-regularization and thus requires less external regularization by other means. This design choice allows Mish to demonstrate improvements over ReLU and SiLU, particularly in terms of stability when dealing with network depth and weight initialization strategies. The function also exhibits greater robustness to Gaussian noise than ReLU and SiLU in its introductory experiments \cite{Mish}.

\begin{figure}[h]
    \centering
    \includegraphics[width=0.75\linewidth]{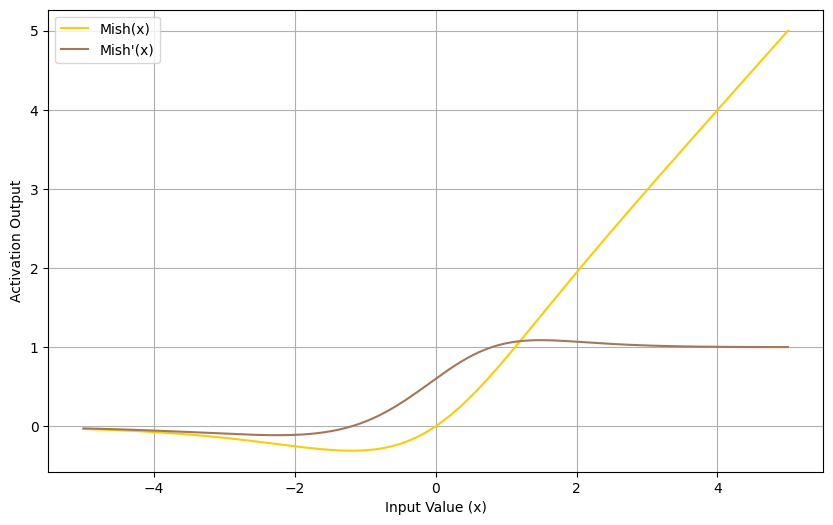}
    \caption[Mish Activation Function and its First-Order Derivative.]
    {\textbf{Mish Activation Function and its First-Order Derivative.} Plot of $Mish(x)= x \cdot tanh( ln(1 + e^x))$ and its first order derivative. As input $x \to -\infty$, Mish saturates to a deactivation output of 0. As input $x \to \infty$, Mish's active region approximates the identity $x$.}
    \label{fig:MishFigure}
    \vspace{1.0\baselineskip}
\end{figure}

\subsection{The Logish and Smish Activation Functions}

Similarly to Mish, the Logish function, $Logish(x) = x \cdot ln(1 + \frac{1}{1+e^{-x}}$ \cite{Logish}, achieves non-monotonic non-linearity with a self-regularizing effect, uniquely approaching a slope of $ln(2)$ as the input grows. The Smish function, $Smish(x) = x \cdot tanh(ln(1 + \frac{1}{1+e^{-x}})$ \cite{Smish}, further evolves Logish by applying the hyperbolic tangent to its smooth binary function, creating a new non-linearity that tends towards a slope of 0.6 as the input increases, though at a higher computational cost. The impact of slope reductions in Logish and Smish is still an area of active investigation, but early experimental results from their respective foundational studies indicate that both functions exhibit greater stability compared to widely used activation functions like ReLU, SiLU, and Mish. This is particularly evident on the CIFAR10 dataset, even under conditions where minimal L2 regularization is applied. 

Notably, the experiments demonstrating Smish's enhanced stability and performance were conducted without utilizing weight decay regularization. This absence of external regularization methods suggests that Smish possesses inherent self-regularization properties, achieving high accuracy without reliance on explicit regularization techniques. Consequently, Smish could be particularly advantageous in scenarios where stability and self-regularization are critical, but computational efficiency is negligible.


\begin{figure}
    \centering
    \includegraphics[width=0.75\linewidth]{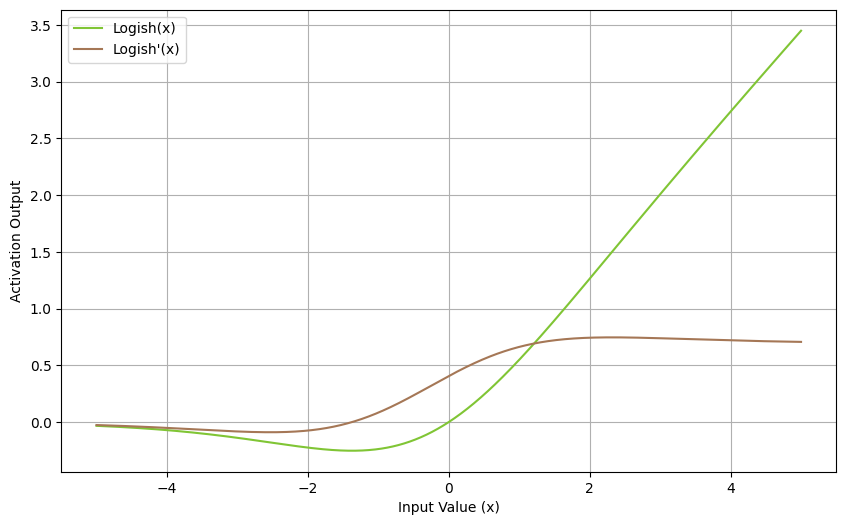}
    \caption[Logish Activation Function and its First-Order Derivative.]
    {\textbf{Logish Activation Function and its First-Order Derivative.} Plot of $Logish(x)= x \cdot ln(1 + \frac{1}{1+e^{-x}})$ and its first order derivative. As input $x \to -\infty$, Logish saturates to a deactivation output of 0. As input $x \to \infty$, Logish's active region approximates a linear growth with a slope of $ln(2)$.}
    \label{fig:LogishFigure}
    \vspace{4.0\baselineskip}
\end{figure}

\begin{figure}
    \centering
    \includegraphics[width=0.75\linewidth]{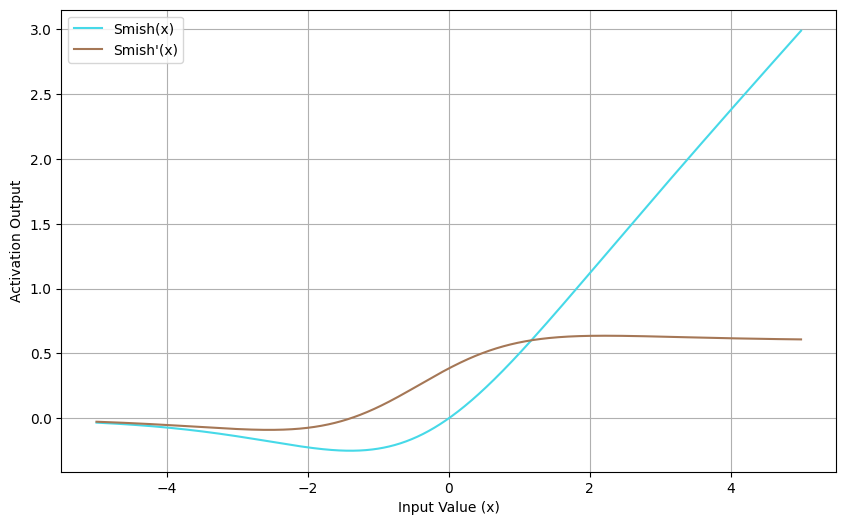}
    \caption[Smish Activation Function and its First-Order Derivative.]
    {\textbf{Smish Activation Function and its First-Order Derivative.} Plot of $Smish(x)= x \cdot tanh( ln(1 + \frac{1}{1+e^{-x}}))$ and its first order derivative. As input $x \to -\infty$, Smish saturates to a deactivation output of 0. As input $x \to \infty$, Smish's active region approximates a linear growth with a slope of 0.6.}
    \label{fig:SmishFigure}
    \vspace{1.0\baselineskip}
\end{figure}

\newpage
\section{Motivation}

Despite substantial progress in deep learning, the choice of activation functions has remained relatively unchanged, with ReLU continuing to dominate due to its computational efficiency and effective gradient propagation in the active region \cite{StrongGradients1, glorot_init, AFTrends, relusparse}. While ReLU’s simple formulation leads to rapid convergence in shallow networks, its limitations become apparent in deeper architectures. The well-known “dying ReLU” problem can cause significant capacity loss, hindering optimization and reducing overall model expressivity \cite{dyingrelu}. This issue becomes more pronounced as networks scale in depth and complexity, where maintaining stable gradient flow is critical for effective training \cite{lu2019dying}.

To address these issues, various alternatives such as ELU \cite{elu}, SiLU \cite{Ramachandran2017SwishAS}, GELU \cite{GELU}, and Mish \cite{Mish} have been proposed, each introducing modifications to smoothen activation transitions and reduce dead neurons. For example, ELU’s negative saturation encourages zero-centered activations, improving convergence stability. Similarly, non-monotonic functions like SiLU and GELU provide smoother gradients, preventing neuron saturation and improving gradient flow \cite{importantSmooth}. However, these designs often introduce increased computational complexity, and empirical evaluations have shown that their advantages are not universally consistent across diverse architectures and datasets \cite{AFTrends, inconsistent2, DUBEY202292, glorot_init, inconsistent3}.

Recent studies suggest that an ideal activation function should balance strong gradient propagation, zero-mean activations, smooth transitions, and computational efficiency \cite{SmoothBenefits, Ramachandran2017SwishAS, relusparse}. By incorporating negative activations, non-monotonic functions typically result in a more symmetrical distribution of activations around zero, which helps to accelerate convergence \cite{elu, DUBEY202292}. However, functions such as GELU and Mish, while possessing these desirable properties, are computationally demanding and produce weaker gradients, which slows down the learning process \cite{glorot_init}. This disconnect between theoretical advantages and practical performance suggests that there is still no single activation function that can effectively unify the strengths of existing methods.

This fragmentation in the design space motivates the need for an activation function that combines the strengths of ReLU’s learning efficiency \cite{relusparse} with the stability and generalization capabilities of smooth functions \cite{importantSmooth}. Our proposed activation function addresses these challenges by integrating smooth transitions, near-zero-mean activations, and robust gradient dynamics, enabling it to achieve consistent performance across a wide range of tasks. Supported by strong theoretical and empirical findings, our function is designed to overcome the limitations of current activations, leading to better convergence properties, enhanced stability, and improved generalization across both shallow and deep architectures.

By building on the strengths of its predecessors and addressing their individual weaknesses, our activation function establishes a new benchmark in activation design. It not only accelerates training and reduces computational overhead but also provides a more stable optimization landscape for large-scale models. This approach offers a unified solution that can set a new standard for activation functions, enabling more reliable and efficient learning in complex deep learning scenarios.

\newpage
\section{Formulation}

\subsection{Design Goals}


Our primary design goal for a hidden activation function is to provide neural network architectures with enough expressivity to approximate any unknown target function defined by the task at hand. Assuming these target functions are continuous, we can theoretically reduce this goal to achieving a universal approximator \cite{Cybenko1989Approximation}. Alongside expressivity, we seek an activation function that is computationally efficient for both forward and backward passes, reducing hardware demands and the duration of each training epoch. Additionally, we aim for an activation function that promotes fast convergence, lowering the number of epochs required for models to achieve success on hidden data. Combined, these requirements ensure training efficiency by minimizing both the time and number of epochs necessary to reduce loss on the training data.


In addition to training efficiency, maintaining training stability is crucial \cite{stabilitySGD}. This means that models using our activation function should achieve consistently favorable evaluations on both training and unseen data, even with slight variations in model and training configurations. Model configuration includes architectural choices like depth, layer width, and the types of layers used, while training configuration covers optimization-related decisions such as the choice of algorithm, learning rate, and weight decay coefficient. A particularly harmful form of instability occurs when a model’s parameters lead to numerically unstable outputs \cite{numericalStability}. When a neuron’s output overflows or underflows, it can cause the entire model to infer incorrectly, resulting in performance that barely exceeds random chance. Therefore, we aim to design an activation function that minimizes these instances of numerical instability, ensuring favorable and reliable training outcomes.


Our final set of goals focus on ensuring that the model's success on training data translates effectively to unseen data. To achieve this, we need to design an activation function that promotes strong generalization, allowing a neural network to maintain its performance beyond the training set. Traditionally, overfitting is managed at the model and optimization level by reducing model capacity \cite{DeepLearningTextbook} or using regularization techniques like dropout \cite{dropout}, weight decay \cite{Rumelhart1986LearningRB}, and batch normalization \cite{ioffe2015batchnormalizationacceleratingdeep}. It is well-known that combining various forms of regularization often leads to better generalization \cite{RethinkingGeneralization, goodfellow2013maxoutnetworks, dropout}. With this in mind, we aim to incorporate a modest yet meaningful level of self-regularization directly within our activation function, operating under the hypothesis that it could reduce the reliance on traditional methods. However, we must be cautious; introducing too much self-regularization could inadvertently limit the model's capacity. In addition, we seek to develop a robust activation function that enhances the model’s resilience, enabling it to perform reliably even when exposed to data with small perturbations.

\subsection{Strategies} 


Leshno et al. \cite{PolynomialUniversalApproximationLESHNO1993861} establish that a function can serve as a universal approximator if and only if it is continuous and non-polynomial. Thus, to satisfy our goal of creating a universal approximator, we must focus on developing a non-polynomial function. For computational efficiency, the activation function must also have a simple formulation with an easily differentiable first-order derivative. As demonstrated by $ReLU$, rapid convergence can be achieved through a straightforward function that maintains strong gradients in the active region, corresponding to positive outputs \cite{relusparse}. To fulfill this rapid convergence requirement, we can employ a linear unit that acts like the identity for positive inputs. However, we should avoid growth that exceeds the identity, as this could lead to exploding gradients and introduce numerical instability into the model. By maintaining this identity-like behavior for positive inputs, we reduce the risk of vanishing gradients in the active region. Additionally, it is crucial to sustain nonzero gradients in the saturation region of our linear unit, helping to prevent subgraphs in the neural network from stalling learning—a common issue in $ReLU$ architectures \cite{lu2019dying}.

To enhance learning stability in neural networks, we first analyze the strategy behind ELU's formulation \cite{elu}. ELU uses approximate odd symmetry near the origin, producing negative outputs for all negative inputs. This design enables the saturation region to counterbalance the linear region, pulling the expected output closer to zero for mean-zero inputs \cite{LecunSymmetricAdvantage.66.2396}. ELU achieves this by avoiding deactivation for negative inputs, adopting a piecewise definition that saturates to -1 as inputs approach $-\infty$. We hypothesize that an activation function with a single definition and gradual deactivation could further improve stability. Non-monotonic activations often exhibit odd symmetry near zero, reducing output bias for input distributions with small standard deviations, as seen with standard normalization techniques \cite{Norm1, Norm2}. Additionally, minimizing nonlinear terms may reduce numerical instability by limiting the chances of underflow or overflow. Thus, we aim for a simple formulation that aligns with our computational efficiency goals.



Recognizing that generalization is enhanced by combining various regularization techniques \cite{RethinkingGeneralization, goodfellow2013maxoutnetworks, dropout}, we aim to develop a self-regularizing activation function. Activation functions like GELU \cite{GELU}, Mish \cite{Mish}, Logish \cite{Logish}, and Smish \cite{Smish} are notable for their inherent self-regularization properties. For example, Hendrycks and Gimpel \cite{GELU} designed GELU as the expected value of a stochastic regularizer, while the self-regularization in Mish, Logish, and Smish arises from additional terms in their formulations. This observation leads us to hypothesize that a new activation function with implicit regularization and enhanced generalization is likely to belong to the class of smooth, non-monotonic nonlinearities \cite{Mish, Logish, Smish}. Furthermore, Rosca et al. \cite{importantSmooth} emphasize the critical role of smooth activation functions in improving both generalization and uncertainty estimates. By pursuing a smooth, non-monotonic activation function, we aim to harness the combined benefits of smoothness and implicit self-regularization.

Beyond traditional generalization, we aim for models using our activation function to also perform well on unseen data corrupted by small perturbations \cite{robustnessseminal}. Xie et al. \cite{xie2021smoothadversarialtraining} have demonstrated that smooth activation functions like GELU and SiLU offer both accuracy and robustness advantages over piecewise linear activation functions. This inherent smoothness, characterized by differentiability at every point, is a defining feature of analytic functions \cite{analytic}. An analytic function can be expressed as a convergent power series, which involves the representation of infinitely many derivatives. Analytic functions are also compatible with higher-order optimization techniques, such as Natural Gradient Descent \cite{shrestha2023naturalgradientmethodsperspectives}, which utilize the curvature of the loss function rather than just its value, leading to more stable learning with fewer steps until convergence \cite{SecondORder2, SecondOrder1, SecondOrder3, backprop}. However, as demonstrated by ELU, it is possible to approach the benefits of second-order optimization techniques without their direct use by incorporating a non-linearity that exhibits a degree of odd symmetry around the origin. This feature helps reduce forward output bias, thereby enhancing the learning efficiency by approach Fisher optimal learning \cite{LecunSymmetricAdvantage.66.2396, Schraudolph2012, LeCun1998, Lecun1991, Amari1998}.


\subsection{Design Requirements}

Based on our developed strategies, the search for an activation function will be directed by the following mathematical requirements:

\begin{itemize}
\item Should approximate linearity as inputs approach $\infty$
\item Must be analytic, or infinitely differentiable at all points, without collapsing to zero
\item Should exhibit near odd symmetry about $y = -x$ near the origin
\item Must include a dense inactive region that saturates towards zero
\item Should demonstrate minimal gradient decay as the output saturates towards zero
\item Should have low computational complexity
\item Must be non-polynomial to qualify as a universal approximator
\end{itemize}

In practice, we discovered that combining conflicting heuristics can lead to a more focused mathematical investigation. Thus, we present our goal summary in the form of the following dichotomies:

\begin{itemize}
\item The non-linearity must be nearly linear in the active region while remaining analytic. This ensures a convergence rate similar to ReLU and the identity, while being infinitely differentiable at all points to enhance stability and robustness.
\item The nonlinearity should activate with both positive and negative values while having a dense inactive region. This allows the gradient of the unit to align with the natural gradient, while still retaining the ability to deactivate or drop out a neuron during training.
\item The nonlinearity must asymptotically decrease towards 0 as inputs become negative, but should minimize the rate at which gradients approach 0 in the inactive region. This gives the optimizer the flexibility to deactivate a neuron without permanently losing the option to reactivate it later.
\item The nonlinearity should be definable with minimal mathematical complexity without being polynomial or suffering from exploding gradients. This results in a computationally efficient universal approximator that can be trained with stability.
\end{itemize}

Additionally, the new non-linearity must perform competitively with ReLU when used in configurations optimized for ReLU-based architectures. This practical requirement ensures that the new activation function can serve as a viable alternative to ReLU and be seamlessly integrated into existing deep learning projects.



\subsection{Proposed Implementation}


After developing several unstable activation functions, we began closely analyzing the Mish activation function \cite{Mish}. Mish, defined as $Mish(x)=x \cdot tanh(ln(1+e^x))$, was noted in its foundational research for offering impressive stability and robustness compared to earlier functions. However, its drawbacks, namely computational inefficiency and damped gradients, prevent it from surpassing ReLU in convergence efficiency in practice \cite{StrongGradients1, glorot_init, relusparse}. Recognizing Mish's potential, we identified that by removing its logarithmic term and the associated unit addition, we could significantly enhance its computational efficiency and gradient strength. 
This led to the creation of the Hyperbolic Tangent Exponential Linear Unit, $TeLU(x) = x \cdot \tanh(e^x)$, which demonstrated superior convergence efficiency \cite{our_old}. TeLU was independently discovered by empirical observations in \cite{tanhexp}. In contrast, our work not only provides theoretical foundations but also extensive empirical evidence demonstrating why TeLU achieves superior convergence efficiency and serves as a better drop-in replacement for ReLU-based systems. \footnote{In the discussion section, we show our key contributions and key differences compared to \cite{tanhexp} }. 
According to the taxonomy proposed by Apicella et al. \cite{APICELLA202114}, TeLU is a classic, fixed-shape activation function.


We found that our mathematical and experimental requirements are best met by TeLU, shown in Figure \ref{fig:TeLUFigure}. TeLU combines the linear properties of traditional activation functions with the nonlinear advantages of exponential and hyperbolic tangent functions. This blend strikes a balance between promoting efficient learning and avoiding exploding gradients. TeLU's design, utilizing the hyperbolic tangent of the exponential function, naturally moderates the output, keeping it within a manageable range. Unlike earlier activation functions, TeLU provides a smooth transition across the origin, enhancing gradient flow and contributing to more stable and consistent learning dynamics \cite{importantSmooth, DUBEY202292}. Figure \ref{fig:AllActivationsFigure} illustrates this behavior, highlighting TeLU's continuity and slower saturation compared to other state-of-the-art functions.

\begin{figure}
    \centering
    \includegraphics[width=0.75\linewidth]{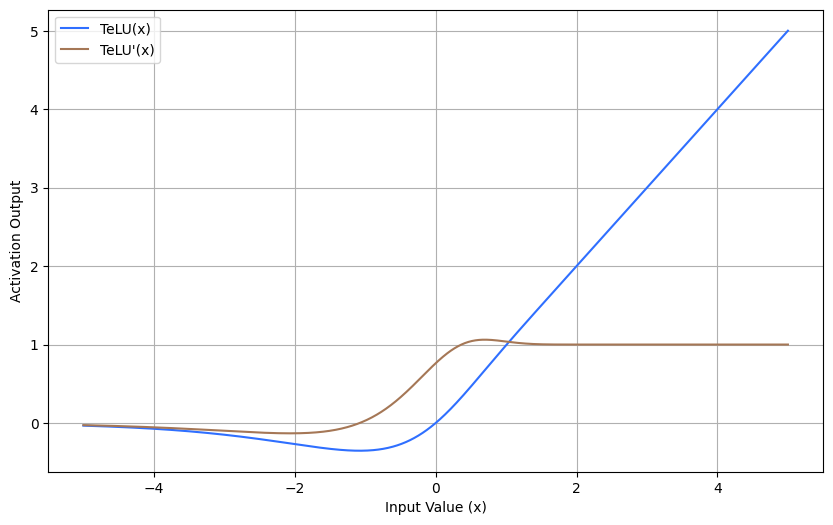}
    \caption[TeLU Activation Function and its First-Order Derivative.]
    {\textbf{TeLU Activation Function and its First-Order Derivative.} Plot of $TeLU(x)= x \cdot tanh( e^x )$ and its first order derivative. As input $x \to -\infty$, TeLU slowly saturates to a deactivation output of 0. As input $x \to \infty$, TeLU's active region quickly approximates the identity $x$.}
    \label{fig:TeLUFigure}
    \vspace{1.0\baselineskip}
\end{figure}



Our formulation is in line with the successful heuristics identified in the SiLU Neural Architecture Search \cite{Ramachandran2017SwishAS}, which emphasizes the effectiveness of linearly-scaled cumulative distribution functions with minimal terms. This category includes activation functions such as ReLU, SiLU, GELU, and Mish, which are all known for their reliable performance as validated by independent research \cite{DUBEY202292, Swish3rdParty}. The similarity of TeLU to both GELU and SiLU suggests that it may share their robustness properties, as indicated in studies on smooth activation functions \cite{xie2021smoothadversarialtraining, importantSmooth}. This consistency with existing literature is essential, as it strengthens the potential of TeLU to achieve comparable levels of success. The alignment with established results not only supports the validity of TeLU's design but also increases the likelihood of its adoption in diverse deep learning applications. To illustrate, Figure \ref{fig:AllActivationsFigure} presents a comparison of the main activation functions discussed in this paper, while Figure \ref{fig:AllDerivativesFigure} displays their corresponding first-order derivatives.



\begin{figure}
    \centering
    \includegraphics[width=0.75\linewidth]{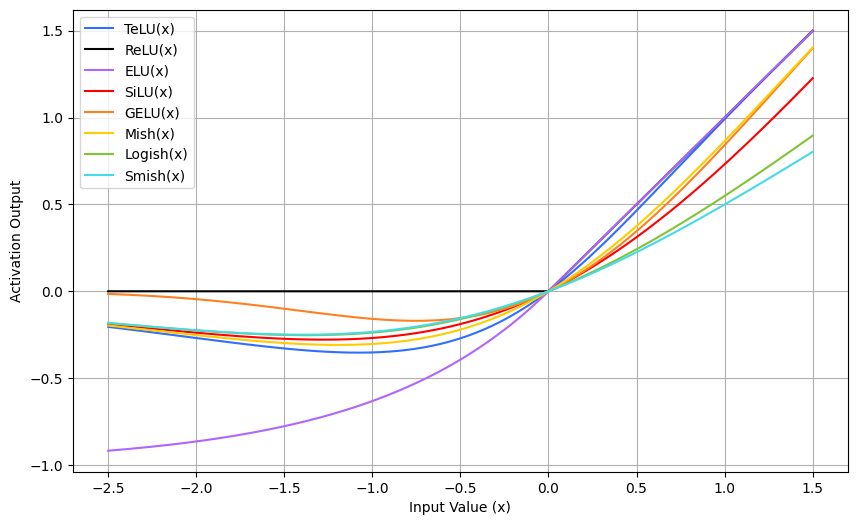}
    \caption[Focused Set of Linear Unit Activation Functions.]
    {\textbf{Focused Set of Linear Unit Activation Functions.} Plot depicts our focused group of linear unit activation functions that will commonly be used throughout analysis and experiments.}
    \label{fig:AllActivationsFigure}
    \vspace{1.0\baselineskip}
\end{figure}

\begin{figure}
    \centering
    \includegraphics[width=0.75\linewidth]{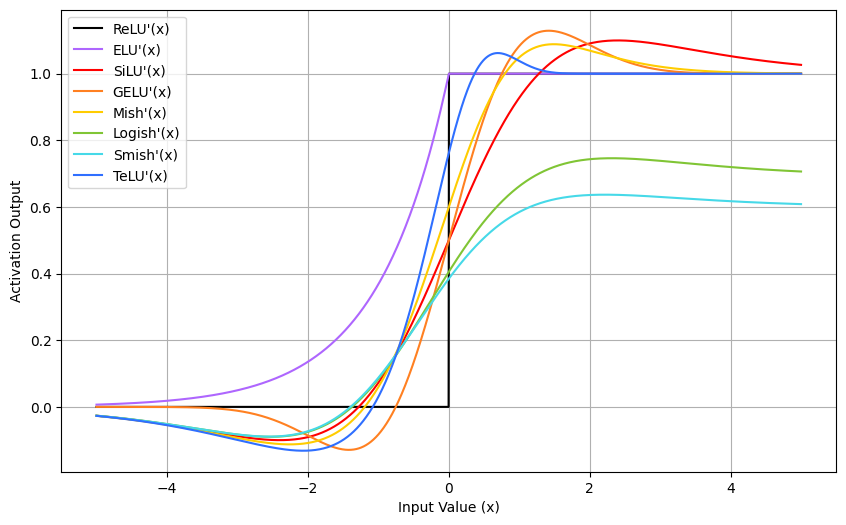}
    \caption[First Derivatives of Focused Set of Linear Unit Activation Functions.]
    {\textbf{First Derivatives of Focused Set of Linear Unit Activation Functions.} Plot depicts the first-order derivatives of our focused group of activation functions.}
    \label{fig:AllDerivativesFigure}
    \vspace{1.0\baselineskip}
\end{figure}

\newpage
\section{Theoretical Framework}
\label{sect:theorysection}


This section provides an in-depth examination of the theoretical advantages of the TeLU activation function, expanding on key principles of activation function design and their impact on neural network performance. It explores essential properties like gradient behavior, convergence speed, computational efficiency, and overall stability, laying the theoretical groundwork for TeLU’s potential benefits. This analysis demonstrates how TeLU addresses common deep learning challenges, including inefficient learning, the vanishing gradient problem, and computational complexity.

\subsection{Persistent Gradients of the Saturation Region}
\label{subsect:persistentgradtheory}

We have observed that lower-bounded activation functions are generally advantageous in deep learning. Apart from certain ReLU variants, all the activation functions we’ve discussed are lower-bounded. This aligns with the behavior of biological neurons, which either remain inactive or fire at a specific frequency. Interestingly, even without an initial bias toward lower-bounded functions, Nader and Azar’s evolutionary neural architecture search (NAS) demonstrated a preference for them across all tested datasets. From a many-advisors perspective, it seems that most tasks inherently favor lower-bounded functions. Moreover, we adhere to the heuristic that an activation function should saturate to an output of 0 when a neuron becomes inactive, making neuron inactivity analogous to dropout during inference. However, this presents the challenge of neurons being irreversibly dropped out during training if the gradient in the inactive region decays too quickly to zero.

To address this, our nonlinearity is designed to saturate to 0 when a neuron is inactive while ensuring that the gradient in the inactive region remains above zero for as long as possible. This approach maximizes the chances of the optimizer reactivating a neuron when needed. Ideally, just as we desire strong gradients in the active region, we also aim for relatively strong gradients in the saturating inactive region. To evaluate this aspect of each nonlinearity's curvature, we first examine the graph of its first derivative. Figure \ref{fig:PersGradFirstDervs} illustrates the absolute value of each activation function's first derivative. The plot it shown over negative values near the origin to give an idea about the relative rate that each function's derivative approaches 0. While some gradients clearly saturate faster than others, we sought a clearer comparison of their decay rates.

\begin{figure}
    \centering
    \includegraphics[width=0.75\linewidth]{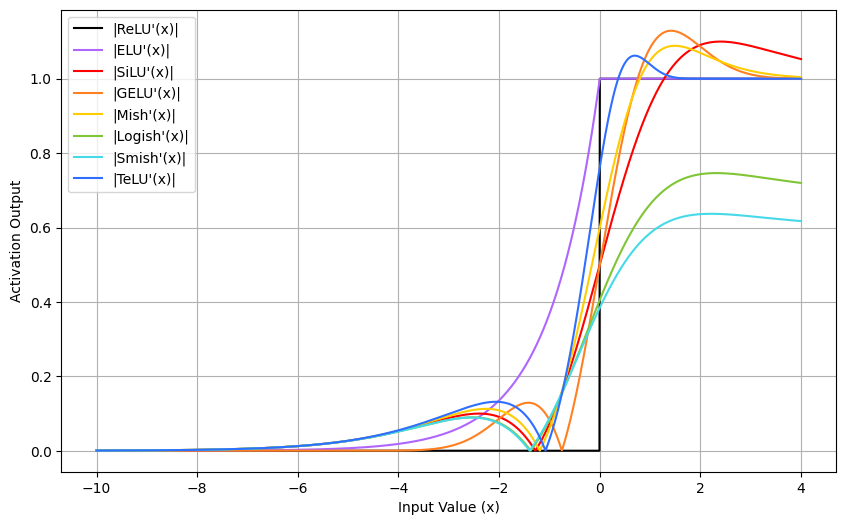}
    \caption[Saturation of Activation Function as Input Grows Negative.]
    {\textbf{Saturation of Activation Function as Input Grows Negative.} Shows the absolute value of the first derivatives of TeLU, ReLU, ELU, SiLU, GELU, Mish, Logish, and Smish. We notice how some derivatives approximate 0 faster than others, indicating that they decay towards 0 at varying rates.}
    \label{fig:PersGradFirstDervs}
    \vspace{1.0\baselineskip}
\end{figure}

\begin{figure}
    \centering
    \includegraphics[width=0.75\linewidth]{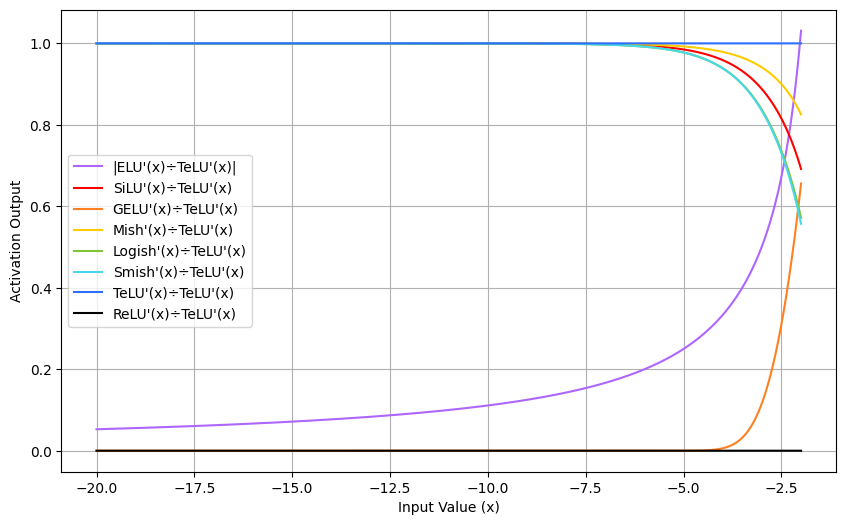}
    \caption[Quotient Comparisons of Vanishing Gradients as Input Grows Negative.]
    {\textbf{Quotient Comparisons of Vanishing Gradients as Input Grows Negative.} Shows quotient comparison comprising first derivatives of each activation divided by TeLU's first derivative. Plotted values less than one show that TeLU's derivatives saturate slower towards 0 than others at small negative values. The dark blue line shows the derivative of TeLU divided by itself, which is equivalent to 1, as a means to easily compare to other, lesser, quotients in this range.}
    \label{fig:RelativeGradients}
    \vspace{1.0\baselineskip}
\end{figure}

Figure \ref{fig:RelativeGradients} addresses this by showing the ratio of TeLU’s absolute derivative to each derivative. This comparison reveals that most non-monotonic nonlinearities maintain gradients that decay more slowly than their monotonic counterparts, albeit with the trade-off of a singular point at zero. This trade-off can be seen as a necessary compromise for gradient persistence. In practice, these solutions emerge primarily due to underflow in lower-precision floating-point variables, as they are likely irrational numbers. Additionally, Figure \ref{fig:RelativeGradients} indicates that TeLU exhibits stronger gradients at the onset of inactivity compared to other functions. To gain a clear understanding of the rates of gradient decay, we extend the concept of asymptotic growth classes to also encompass asymptotic decay. Therefore, we define the $O$, $\Omega$, and $\Theta$ asymptotic classes as:

\begin{Def}
    $f(n) \in O(g(n))$ if and only if there exists positive constants $c$ and $n_0$ such that $f(n) \leq c \cdot g(n)$ for all $n \geq n_0$; where $n \in \mathbb{Z}$
\end{Def}

\begin{Def}
    $f(n) \in \Omega(g(n))$ if and only if there exists positive constants $c$ and $n_0$ such that $c \cdot g(n) \leq f(n)$ for all $n \geq n_0$; where $n \in \mathbb{Z}$
\end{Def}

\begin{Def}
    $f(n) \in \Theta(g(n))$ if and only if there exists positive constants $c_1$, $c_2$, and $n_0$ such that $c_1 \cdot g(n) \leq f(n) \leq c_2 \cdot g(n)$ for all $n \geq n_0$; where $n \in \mathbb{Z}$
\end{Def}

In traditional terms, a function $f(x)$ belongs to the class $O(g(x))$ if it is bounded above by a scaled version of $g(x)$ beyond a certain point. Similarly, $f(x)$ is in the class $\Omega(g(x))$ if $g(x)$ serves as a lower bound for $f(x)$ beyond a fixed input value. The function $f(x)$ is classified as $\Theta(g(x))$ if and only if it belongs to $O(g(x))$ and $\Omega(g(x))$. Common asymptotic growth classes include $O(1)$, $O(x)$, $O(x^c)$, $O(c^x)$, $O(x!)$, where, $c$ is a positive constant. $a!$ expresses the factorial operation on variable $a$, which is $\prod_{i=a}^{1} i$ when $a \in \mathbb{Z}^+$. To extend the factorial operation to real numbers, we utilize the Gamma function $\Gamma (z) = \int_{0}^{\inf} t^{z-1}e^{-t}dt$.

To define asymptotic decay classes, we invert these growth classes, resulting in classifications like $\Theta(1)$, $\Theta(1/x)$, $\Theta(1/x^2)$, $\Theta(1/e^x)$, and $\Theta(1/x!)$. When we apply this framework to non-monotonic activation functions, we find that the limit of some functions divided by the corresponding $g(x)$ functions does not converge to a nonzero constant. Therefore, we extend the decay classes with $\Theta(x/e^x)$ and $\Theta(1/(x^2)!)$ to better describe the asymptotic decay of functions like TeLU, SiLU, Mish, Logish, and Smish. Interestingly, we discover that GELU does not belong $\Theta(1/x!)$ due to having a greater rate of decay, leading us to classify GELU as an element of $O(1/x!)$ and $\Omega(1/(x2)!)$. We summarize these asymptotic decay classifications in \ref{tab:Asymptotic}, with each function mirrored about x=0 to more clearly illustrate the mathematical notation. 

\begin{table}[]
\centering
 \caption[Asymptotic Decay Qualifications.]
 {\textbf{Asymptotic Decay Qualifications.} Shows the asymptotic decay classes that describes the rate at which each activation function's derivative approaches 0. If the saturation of an activation function $f(x)$ can be described by a tight bound, $\Theta (g(x))$, then the limit of $\frac{-f(x)}{g(x)}$ as $x \to \infty$ evaluates to a non-zero integer. We could not determine this tight bound for GELU, but instead found upper $O$ and lower $\Omega$ bounds. ReLU is simply defined as 0 for negative inputs, so it does not exhibit a particular decay. }
    \begin{tabular}{||c c c||} 
 \hline
 Function & TeLU Limit & Asymptotic Class \\
 \hline\hline
 TeLU & 1.0 & $\frac{d}{dx}TeLU(-x)\in\Theta(\frac{x}{e^x})$ \\[1.5pt] 
 \hline 
 \\[-1em]
 ReLU & N/A & N/A \\[1.5pt] 
 \hline
 \\[-1em]
 ELU & $x$ & $\frac{d}{dx}ELU(-x)\in\Theta(\frac{1}{e^x})$\\[1.5pt] 
 \hline
 \\[-1em]
 SiLU & 1.0 & $\frac{d}{dx}SiLU(-x)\in\Theta(\frac{x}{e^x})$\\[1.5pt] 
 \hline
 \\[-1em]
 GELU & $e^x$ & $\frac{d}{dx}GELU(-x)\in O(\frac{1}{x!}), \Omega(\frac{1}{(x^2)!})$ \\[1.5pt] 
 \hline
 \\[-1em]
 Mish & 1.0 & $\frac{d}{dx}Mish(-x)\in\Theta(\frac{x}{e^x})$\\[1.5pt] 
 \hline
 \\[-1em]
 Logish & 1.0 & $\frac{d}{dx}Logish(-x)\in\Theta(\frac{x}{e^x})$\\[1.5pt] 
 \hline
 \\[-1em]
 Smish & 1.0 & $\frac{d}{dx}Smish(-x)\in\Theta(\frac{x}{e^x})$\\[1.5pt]  
 \hline
\end{tabular}
\label{tab:Asymptotic}
\end{table}

We further demonstrate the relevance of these asymptotic decay classifications by identifying the approximate regions where each nonlinearity’s gradient underflows when using Pytorch's float32 precision. To ensure a fair comparison, we defined each function using its typical subfunctions as listed in Table 4. For the first derivative of each nonlinearity, we examined the domain subset $[-200,0]$, recording instances where the gradient reached zero with a step size of 0.0001. We record the results in \ref{tab:Numerical} with a precision of two floating point decimal places.

\begin{table}[]
\centering
 \caption[Numerical Inactivity of Activation Functions.]
 {\textbf{Numerical Inactivity of Activation Functions.} Table shows the domain at which each activation function evaluation results in numerical underflow when using a floating point precision of 32 bits. We also determine the evaluations of each activation function at $x=-10,-100$ to show the persistence of gradients both near and far from the origin. }
    \begin{tabular}{||c c c c||} 
 \hline
 $f(x)$ & Null Domain & $f'(-10)$ & $f'(-100)$ \\
 \hline\hline
 $TeLU(x)$ & $(-\infty,-103.98]$ & -4.2826e-06 & -1.4013e-42 \\ 
 \hline
 $ReLU(x)$ & $(-\infty,0.0]$ & 0.0 & 0.0 \\
 \hline
 $ELU(x)$ & $(-\infty,-103.98]$ & 7.5826e-10 & 3.7835e-44\\
 \hline
 $SiLU(x)$ & $(-\infty,-103.98]$ & -4.2826e-06 & -1.4013e-42\\
 \hline
 $GELU(x)$ & $(-\infty,-14.42]$ & -9.2447e-18 & 0.0 \\
 \hline
 $Mish(x)$ & $(-\infty,-103.98]$ & -4.2309e-06 & -1.4153e-42\\
 \hline
 $Logish(x)$ & $(-\infty,-103.98]$ & -4.2309e-06 & -1.4153e-42\\
 \hline
 $Smish(x)$ & $(-\infty,-103.98]$ & -4.2309e-06 & -1.4153e-42\\
 \hline
\end{tabular}
\label{tab:Numerical}
\end{table}

Interestingly, the non-monotonic zero-gradient singularities were not detected, suggesting that they are unlikely to be common during training. Additionally, it is evident that TeLU, SiLU, Mish, Logish, and Smish delay their underflow regions due to their slower asymptotic decay. In contrast, GELU and ELU exhibit broader zero-gradient regions as a result of numerical underflow. ReLU, on the other hand, consistently exhibits zero gradients throughout its entire inactive region, which is characteristic of the "dying ReLU" problem. However, the shared underflow regions observed in TeLU, SiLU, Mish, Logish, and Smish arise from their exponential subfunctions experiencing numerical underflow, which causes them to default to 0. As a result, the sigmoid component of these functions evaluates to 0, leading to a product of 0 when multiplied by the identity function. Similarly, the GELU activation function is driven to 0 due to the behavior of its error function implementation.

To better understand the comparative numerical decay of these nonlinearities, we examined the values of their derivatives at various points within their inactive regions. Although our full experiment considered 100 different points, for simplicity, we present the gradient values at inputs -100 and -10, which we believe sufficiently capture the overall behavior. These values confirm that TeLU and SiLU maintain stronger gradients at the onset of their post-singularity inactive regions. In contrast, Mish, Logish, and Smish initially exhibit weaker gradients but eventually lead to slightly stronger gradients. This relationship is similar to the comparison between the Hyperbolic Tangent and Sigmoid nonlinearities: while the Hyperbolic Tangent maintains stronger gradients earlier in its inactive region, the Logistic Sigmoid eventually shows slightly stronger gradients deeper into its inactive state.

\subsection{Near-Linearity of Active Region}
\label{subsect:nearlinsubsection}


Activation functions that mimic the identity function in their unbounded regions benefit from strong gradients in their active areas, resulting in more efficient training \cite{StrongGradients1, glorot_init}. These strong unit gradients prevent learnable parameter updates from being scaled down by the upstream partial derivatives of saturating functions during backpropagation, allowing the model to reach local minima in the loss landscape more rapidly. At this stage, learning rate schedulers can then lower the effective learning rate to facilitate precise convergence. In contrast, activation functions with sub-identity growth in the active region often require specialized learning rate schedulers to counterbalance their weaker gradients. However, this added complexity increases inter-dependencies among training hyperparameters, leading to a less modular design.


To achieve similar benefits in convergence rate and modularity with a smooth activation function, we aimed to approximate linearity in the active region. In this context, $TeLU(x) = x \cdot \tanh(e^x)$ can be viewed as a variant of $Mish(x) = x \cdot \tanh(\ln(1 + e^x))$, but with the logarithmic term removed, allowing the hyperbolic tangent to saturate to 1 more rapidly. The identity function then multiplies with this saturated hyperbolic tangent, resulting in near-linearity in the active region. As shown in Table \ref{tab:NearLinearity}, TeLU’s activation for positive inputs aligns more closely with the identity than any other smooth function considered, as measured by both L1 and L2 metrics. The L1 distance metric from linearity is calculated as $\int_0^\infty |f(x) - mx| dx$ for each activation function $f(x)$, where $m$ represents the slope that $f(x)$ approaches as input grows large. Similarly, the L2 distance metric is evaluated as $\int_0^\infty (f(x) - mx)^2 dx$. 


\begin{table}[]
\centering
 \caption[Near-Linearity Distance Calculations.]
 {\textbf{Near-Linearity Distance Calculations.} Calculating the sums of L1 and L2 distances between each activation function's active region and the linear function each non-linearity approximates throughout $0 \leq x < \infty$. Logish approximates a slope of $ln(2)$ as $x \to \infty$, while Smish approximates a slope of 0.6. All other activation functions either approximate or match the identity function $x$ in growth throughout their active regions. We use these respective slopes to calculate each distance.}
    \begin{tabular}{||c c c c||} 
 \hline
 Function & L1 Distance & L2 Distance & Max Slope \\
 \hline\hline
 TeLU & \textbf{0.0273} & \textbf{0.0008} & \textbf{1.0} \\ 
 \hline
 SiLU & 0.8225 & 0.1582 & \textbf{1.0} \\
 \hline
 GELU & 0.2500 & 0.0309 & \textbf{1.0} \\
 \hline
 Mish & 0.2407 & 0.0238 & \textbf{1.0} \\
 \hline
 Logish &  0.4289 & 0.0436 & $ln(2)$ \\
 \hline
 Smish & 0.2887 & 0.0201 & 0.6 \\
 \hline
\end{tabular}
\label{tab:NearLinearity}
\end{table}


Among the listed activation functions, TeLU stands out by exhibiting an active region that most efficiently and rapidly approximates the identity function. This is significant because, in theory, stronger gradients enhance gradient propagation during backpropagation, thereby minimizing the vanishing gradient problem. However, it is important to avoid gradients that persistently exceed 1, as they can lead to exploding gradients. Functions that grow linearly with a gradient of 1 effectively mitigate both vanishing and exploding gradient issues. Piecewise functions like ReLU and ELU are defined as the identity for positive inputs but lack persistent gradients in their negative regions. This deficiency can result in neurons becoming inactive or "dying," a common issue in ReLU-based networks when the gradient becomes zero for negative inputs. In contrast, TeLU offers persistent, non-zero gradients across its entire domain, reducing the likelihood of neuron inactivation. As discussed in Subsection \ref{subsect:persistentgradtheory}, the persistent gradients provided by TeLU facilitate more stable and consistent weight updates, leading to smoother training dynamics. The combination of persistent gradients in the deactivation region and strong unit gradients in the active region enables neural networks using TeLU to converge faster than those employing other activation functions. This makes TeLU a superior choice for achieving efficient and robust training outcomes.

\subsection{Runtime Efficiency}
\label{subsect:runtimeefficiencytheory}

The efficiency of a neuron during forward and backward passes can be significantly enhanced by a nonlinearity with a simple mathematical formulation. Many powerful optimization methods exist that allow major speedups in numerical calculations. Approaches such as dynamic programming \cite{DynamicProgHWACC}, linear interpolation \cite{NumericalOptimizationHWACC}, low-level language implementations \cite{ProgrammingLangSpeedupHWACC}, and machine learning hardware accelerators \cite{HWACCQuantum} \cite{HWACCTPU} \cite{HWACCQuantum} can each provide vast speedups to computations. Together, these optimization methods can provide indeterminately large performance speedups to activation function computation. These, however, come at engineering, hardware, or power costs. In this subsection, we abstract away such optimizations and focus on the baseline definitions of each non-linearity. We define these expressions to contain only common expressions, namely: $x$, $e^x$, $ln(x)$, $tanh(x)$, maximum, piece-wise, and the error function. We express each non-linearity in their common forms in \ref{table:baselineFormulations}, which reflects how each are defined in their initial works \cite{ReLU, elu, Ramachandran2017SwishAS, GELU, Mish, Logish, Smish}.

\begin{table}[]
\centering
 \caption[Baseline Definitions of Activation Functions.]
 {\textbf{Baseline Definitions of Activation Functions.} Base formulations of activation functions used to determine computational complexity in this subsection. Legal terms include only expression with universal recognition in mathematics, to provide a fair comparison.}
    \begin{tabular}{||c c||} 
 \hline
 Function & Algebraic Definition \\
 \hline\hline
 TeLU & $x \cdot tanh(e^x)$ \\ 
 \hline
 ReLU & $maximum(0,x)$ \\
 \hline
 ELU & if $ x<0$ : $e^x - 1$; else: $x$\\
 \hline
 SiLU & $x \cdot \frac{1}{1+e^{-x}}$\\
 \hline
 GELU & $x \cdot \frac{1}{2}(1+erf(\frac{x}{\sqrt{2}}))$ \\
 \hline
 Mish & $x \cdot tanh(ln(1+e^x))$\\
 \hline
 Logish & $x \cdot ln( 1 + \frac{1}{1+e^{-x}})$\\
 \hline
 Smish & $x \cdot tanh(ln( 1 + \frac{1}{1+e^{-x}}))$\\
 \hline
\end{tabular}
\label{table:baselineFormulations}
\end{table}

To assess the mathematical complexity of the activation functions in our study, we attempt to quantify their complexity. We hypothesize that a primary contributor to a function's computational complexity is the number of piecewise operations it contains. We differentiate piecewise operations from typical nonlinear operations, as they are normally implemented by a conditional flow of control rather than a single calculation. These are followed by the non-piecewise non-linear operations each activation function contains. For this, we count the instances of logarithms, exponentials, trigonometric functions, and error functions present in each nonlinearity. To further assess computational complexity, we also tally the number of multiplicative (multiplication or division) and additive (addition or subtraction) operations required by each nonlinearity. that occurs outside of the nonlinear functions previously accounted for. These operations often involve long vectors within neural networks, making them computationally expensive. Finally, we consider the number of constant-valued vectors that need to be allocated for calculating the nonlinearity. Although these allocations can be optimized or amortized in various ways, they still present an existing memory cost. We summarize these quantifications in Table \ref{table:ComputationComplexity}. As there are several plausible ways to express the first derivatives of these activation functions, we only perform such a heuristic assessment in Table \ref{table:DerivativeComplexity} in our Appendix.

\begin{table}[]
\centering
 \caption[Computational Costs of Activation Functions.]
 {\textbf{Computational Costs of Activation Functions.} Summarizes the number of instances of nonlinearity computations, arithmetic operations, and constant terms for each activation function. These heuristics provide us with a theoretical evaluation of the computational complexity of each non-linearity.}
    \begin{tabular}{||c c c c c c||} 
 \hline
 Function & Piecewise & Nonlinearity & Multiplicative & Additive & Constants \\
 \hline\hline
 TeLU & 0 & 2 & 1 & 0 & 0 \\ 
 \hline
 ReLU & 1 & 0 & 0 & 0 & 1 \\
 \hline
 ELU & 1 & 1 & 0 & 1 & 1\\
 \hline
 SiLU & 0 & 1 & 2 & 2 & 3\\
 \hline
 GELU & 0 & 1 & 2 & 1 & 3 \\
 \hline
 Mish & 0 & 3 & 1 & 1 & 2\\
 \hline
 Logish & 0 & 2 & 3 & 1 & 4\\
 \hline
 Smish & 0 & 3 & 3 & 1 & 4\\
 \hline
\end{tabular}
\label{table:ComputationComplexity}
\end{table}

\subsection{ReLU Compatability} 
\label{subsect:relucomptheory}
The widespread success of the Rectified Linear Unit (ReLU) activation function in stateless architectures, such as Convolutional Neural Networks (CNNs), Transformers, Autoencoders, and Diffusion Models, has firmly established it as the default choice in deep learning research and development. Its simplicity, computational efficiency, and ability to mitigate the vanishing gradient problem have made it indispensable for training a variety of deep architectures. Consequently, an activation function that closely mimics ReLU’s behavior while offering potential improvements is more likely to gain traction within the community. Such a function should not only preserve ReLU’s core characteristics but also enhance network stability, gradient flow, and expressivity across diverse architectures. 

To systematically identify an activation function that meets these criteria, we first define a rigorous evaluation metric that measures how closely a candidate function approximates ReLU’s behavior. Let $ r(x) = \max(0, x) $ denote the ReLU function and $ f(x) $ represent a candidate activation function. The total approximation error is defined as:

\[
\mathcal{L}_{\text{total}} = \int_{-\infty}^{\infty} \big| r(x) - f(x) \big| \, dx.
\]

This integral quantifies the overall deviation of $ f(x) $ from ReLU across the entire input domain. However, because ReLU operates in two distinct regions—\textit{inactive} ($ x < 0 $) and \textit{active} ($ x \geq 0 $)—we further decompose the error into contributions from each region:

\[
\mathcal{L}_{\text{inactive}} = \int_{-\infty}^{0} \big| r(x) - f(x) \big| \, dx, \quad \mathcal{L}_{\text{active}} = \int_{0}^{\infty} \big| r(x) - f(x) \big| \, dx.
\]

This decomposition allows us to analyze whether the approximation errors are concentrated in the suppression of negative values or in the scaling of positive values, which can indicate specific properties of $ f(x) $ that may be beneficial or detrimental for particular model classes. We apply these metrics to a comprehensive set of commonly used activation functions, including TeLU, ELU, SiLU, GELU, Mish, Logish, Smish, Leaky ReLU ($ LReLU(x) = \max(0.01x, x) $), and Softplus ($ \text{Softplus}(x) = \ln(1+e^x) $). Table \ref{tab:ReLUProximity} summarizes the computed proximity values, providing a robust comparison of how well each function replicates ReLU’s behavior in different regions. This systematic evaluation framework not only facilitates the selection of activation functions that can act as direct substitutes for ReLU but also guides the development of new functions that maintain ReLU’s favorable properties while potentially mitigating known issues such as the “dying ReLU” problem. Ultimately, our methodology supports the discovery of activation functions that can seamlessly integrate into diverse deep learning architectures, ranging from CNNs and Autoencoders to Transformers and emerging generative models like Diffusion Models, ensuring compatibility and enhancing overall model robustness.

To begin identifying an activation function that learns similarly to ReLU, we first assume that the evaluation of such an activation function closely approximates that of ReLU. We quantify the loss of this approximation by determining $\int_{-\infty}^{\infty} |r(x)-f(x)|dx$, where $r(x)=ReLU(x)$ and $f(x)$ represents each candidate function we are evaluating in terms of ReLU approximation. To gain further insight into this approximation, we also evaluate $\int_{-\infty}^{0} |r(x)-f(x)|dx$ and $\int_{0}^{\infty} |r(x)-f(x)|dx$ to quantify the loss of approximation for both the inactive and active region of ReLU. We performs these calculations on TeLU, ELU, SiLU, GELU, Mish, Logish, Smish, $LReLU = max(0.01x,x)$, and $Softplus(x) = ln(1+e^x)$ in Table \ref{tab:ReLUProximity}.


\begin{table}[]
\centering
 \caption[Sub-Domain Distance of Activation Functions to ReLU.]
 {\textbf{Sub-Domain Distance of Activation Functions to ReLU.} Calculating the sums of L1 distance between each activation function and ReLU along negative and positive subdomains}
    \begin{tabular}{||c c c||} 
 \hline
 Function & (-inf, 0] & [0, inf)  \\
 \hline\hline
 TeLU & 0.967 & 0.027  \\ 
 \hline
 LReLU & $\infty $ & 0.000 \\
 \hline
 Softplus & 0.822 & $\infty$ \\
 \hline
 ELU & $\infty$ & 0.000  \\
 \hline
 SiLU & 0.822 & 0.822  \\
 \hline
 GELU & 0.250 & 0.250  \\
 \hline
 Mish & 0.884 & 0.241  \\
 \hline
 Logish &  0.767 & $\infty$  \\
 \hline
 Smish & 0.760 & $\infty$  \\
 \hline
\end{tabular}
\label{tab:ReLUProximity}
\end{table}

In summary, our findings reveal that only non-monotonic activation functions within the candidate pool can approximate ReLU with a bounded $ L_1 $-norm error. Specifically, these non-monotonic functions share critical mathematical properties that align closely with ReLU’s behavior:

\begin{itemize}
\label{list:ReLUprops}
\item $f(0)=0$
\item $\lim_{x \to -\infty} f(x) = 0$
\item $\lim_{x \to \infty} f(x) = \infty$
\item $\lim_{x \to \infty} \frac{f(x)}{x} = 1$
\end{itemize}

These characteristics are not incidental but are fundamental requirements for any function intended to replicate ReLU’s unique role in deep learning architectures. First, having $ f(0) = 0 $ ensures that a neuron using such a function will remain inactive when the pre-activation input is zero, preserving the sparsity and inhibition properties essential for ReLU-based networks. Any deviation from this property would alter the activation pattern, impacting the network's overall inference dynamics. Second, the asymptotic behavior as $ x \to -\infty $ guarantees a dense deactivation region for large negative inputs, which helps induce sparsity by maintaining a large portion of neurons in an inactive state. Without this property, the activation function would produce spurious activations, undermining the benefits of sparse representations.

Moreover, the linear growth as $ x \to \infty $ is a hallmark of ReLU’s effectiveness, enabling strong gradients for large positive inputs. If the growth rate is slower than linear, the function would suffer from vanishing gradients, leading to poor convergence during training. Conversely, if the function grows faster than $ x $, it can exacerbate exploding gradients, resulting in numerical instability. Thus, a successful ReLU approximation must satisfy these growth conditions to ensure compatibility with ReLU-optimized architectures.

The observation that only non-monotonic functions meet these criteria suggests that non-monotonicity is not merely an artifact but a \textit{necessary} feature for accurately emulating ReLU’s behavior. This becomes clearer when we examine the structural similarities across ReLU and other widely adopted activations such as TeLU, SiLU, GELU, and Mish. Each of these functions can be expressed as a scaled cumulative distribution function of the form $ x \cdot \sigma(x) $, where $ \sigma(x) $ is a smooth, monotonically increasing function with a range of $ (0, 1) $. For ReLU, $ \sigma(x) $ is the Heaviside step function:

\[
\sigma(x) = 
\begin{cases} 
0 & \text{for } x < 0 \\ 
1 & \text{for } x \geq 0 
\end{cases}.
\]

In TeLU, this distribution is $ \sigma(x) = \tanh(e^x) $, while other activations have their own distinct, yet similar, cumulative distributions. From this perspective, we can interpret the non-monotonicity in these functions as a smooth approximation of ReLU’s abrupt transition at $ x = 0 $. The fact that ReLU can be viewed as a degenerate case of a generalized class of functions $ x \cdot \sigma(x) $ suggests that any smooth approximator to ReLU must inherently introduce non-monotonic behavior to bridge the gap between active and inactive regions.

In fact, we can prove that \textit{non-monotonicity} is the only class of single-definition functions capable of approximating ReLU while satisfying the essential properties listed in \ref{list:ReLUprops}. To formalize this, we employ a proof by contradiction. Assume that a purely monotonic function satisfies all the conditions of \ref{list:ReLUprops}. Then, by construction, it must either fail to deactivate sufficiently in the inactive region ($ x < 0 $) or fail to grow linearly in the active region ($ x > 0 $). This contradiction implies that monotonicity cannot achieve both the sparse deactivation and linear activation required to replicate ReLU. Thus, non-monotonicity is not just a feature—it is a prerequisite for any smooth approximation of ReLU.

This insight shifts the focus from simply searching for ReLU substitutes to designing non-monotonic activation functions that retain ReLU’s fundamental properties while offering enhanced numerical stability and gradient behavior. By framing the problem in this way, we establish a theoretical foundation for developing next-generation activations that are both mathematically principled and empirically effective.

\begin{Thm}
\label{theorem:ReLUcontradiction}
Let $ r(x) = \text{ReLU}(x) = \max(0, x) $ be defined as:
\[
r(x) = 
\begin{cases} 
0, & \text{if } x \leq 0, \\ 
x, & \text{if } x > 0.
\end{cases}
\]
This function has the following properties:
\begin{itemize}
    \item $ r(0)=0 $,
    \item $ \lim_{x \to -\infty} r(x) = 0 $,
    \item $ \lim_{x \to \infty} r(x) = \infty $.
\end{itemize}

Let $ f(x) $ be a single continuous and differentiable function defined on $ \mathbb{R} $ that satisfies the same properties as $ r(x) $ for all $ x \in \mathbb{R} $. Then, any such $ f(x) $ that approximates $ r(x) $ must be non-monotonic. In other words, $ f(x) $ must both increase and decrease throughout its domain.
\end{Thm}

\begin{proof}

\textit{Step 1:} Properties of ReLU and Its Derivative

The ReLU function $ r(x) = \max(0, x) $ has a piecewise linear form:

\[
r(x) = 
\begin{cases} 
0, & \text{if } x \leq 0, \\ 
x, & \text{if } x > 0.
\end{cases}
\]

Its derivative $ r'(x) $ is given by:

\[
r'(x) = 
\begin{cases} 
0, & \text{if } x < 0, \\ 
1, & \text{if } x > 0.
\end{cases}
\]

The derivative of ReLU changes abruptly at $ x = 0 $, making $ r(x) $ non-differentiable at that point. This sharp transition is what any smooth approximating function $ f(x) $ must capture to closely resemble $ r(x) $.

\textit{Step 2:} Conditions on $ f(x) $ for Smooth Approximation
Suppose $ f(x) $ is a smooth, continuous, and differentiable function defined for all $ x \in \mathbb{R} $. Furthermore, assume that $f(x)$ is a monotonic function, with an absence of either increasing or decreasing behavior along all  $ x \in \mathbb{R} $. To approximate $ r(x) $, $ f(x) $ must satisfy:

1. $ f(0) = 0 $,
2. $ \lim_{x \to -\infty} f(x) = 0 $,
3. $ \lim_{x \to \infty} f(x) = \infty $.


Since $f(x)$ is continuous and differentiable at all points, there exist only three possible cases that describe the growth of $f(x)$ along all $x \in (-\infty, 0]$:
\begin{enumerate}
\item $f(x)$ is constant-valued across all $x \in (-\infty, 0]$, without increasing or decreasing.
\item $f(x)$ increases throughout $a < x < b$, where  $a, b \in (-\infty, 0]$, by an amount $d \in \mathbb{R}^+$.
\item $f(x)$ decreases throughout $a < x < b$, where  $a, b \in (-\infty, 0]$, by an amount $d \in \mathbb{R}^+$.
\end{enumerate}

We proceed by following each case.

\textit{Case 1:} Suppose $ f(x) $ is constant for all $ x \in (-\infty, 0] $. For $ f(x) $ to remain constant over this interval, it must be defined as a constant value, say $ f(x) = 0 $, for all $ x \in (-\infty, 0] $. Now, if $ f(x) $ is to increase for $ x > 0 $, then at some point $ c > 0 $, $ f(x) $ must exhibit growth, i.e., $ f(x) $ must no longer remain constant. This implies that $ f(x) $ must change its behavior at $ x = 0 $, transitioning from being constant for $ x \in (-\infty, 0] $ to increasing for $ x > 0 $. For instance, we might redefine $ f(x) $ as an increasing function, say $ g(x) $, for $ x > 0 $. However, this construction results in a contradiction because the function $ f(x) $ is no longer defined by a single expression throughout all $ x \in \mathbb{R} $. Instead, it would require two distinct definitions: one for $ x \in (-\infty, 0] $, where $ f(x) = 0 $, and another for $ x > 0 $, where $ f(x) $ follows some increasing function. Such a piecewise definition contradicts the assumption that $ f(x) $ is a single continuous function across $ \mathbb{R} $.

In conclusion, for $ f(x) $ to be both constant for $ x \in (-\infty, 0] $ and increasing for $ x > 0 $, it would have to be defined by two distinct expressions, which contradicts the premise that $ f(x) $ is a single function over all $ x \in \mathbb{R} $. Therefore, $ f(x) $ cannot satisfy both conditions without being discontinuous or piecewise defined.



\textit{Case 2:} Let $ f(x) $ be a continuous and monotonically increasing function on the interval $ (a, b) $. By definition, this means that for any $ x_1, x_2 \in (a, b) $ with $ x_1 < x_2 $, we have $ f(x_1) \leq f(x_2) $. Therefore, $ f(x) $ does not exhibit any decreasing behavior on $ (a, b) $.  Assume further that $ \lim_{x \to -\infty} f(x) = 0 $ and that $ f $ increases by an amount $ d $ over the interval $ (a, b) $, meaning that $ f(b) = f(a) + d $. Thus, the relationship between $ f(a) $ and $ f(b) $ is given by:

\[
f(b) = f(a) + d,
\]

where $ d > 0 $ due to the increasing nature of $ f(x) $.  By the Intermediate Value Theorem (IVT), since $ f(x) $ is continuous on $ (a, b) $, for any value $ c \in (f(a), f(b)) $, there exists some $ x_0 \in (a, b) $ such that $ f(x_0) = c $. In particular, for any small $ \epsilon > 0 $, we can choose $ c = f(b) - \epsilon = d - \epsilon $. Thus, there exists some $ x_0 \in (a, b) $ such that:

\[
f(x_0) = d - \epsilon.
\]

This does not imply that $ f(x) $ decreases anywhere in $ (a, b) $. Instead, because $ f(x) $ is monotonically increasing, it passes through the value $ d - \epsilon $ while continuing to increase from $ f(a) $ to $ f(b) $.

To summarize, $ f(x) $ can take any value between $ f(a) $ and $ f(b) $ by the IVT, and since $ f(x) $ is increasing, it achieves these values without any decrease. Therefore, the fact that $ f(x) $ equals $ d - \epsilon $ for some $ x_0 \in (a, b) $ is entirely consistent with the premise that $ f(x) $ is a monotonically increasing function. Hence, we conclude that the function $ f(x) $ increases by $ d $ from $ f(a) $ to $ f(b) $, and the Intermediate Value Theorem guarantees that all values between $ f(a) $ and $ f(b) $ are attained in the interval $ (a, b) $, without any need for $ f(x) $ to decrease."

\textit{Case 3:} Let $ f(x) $ be a continuous and monotonically decreasing function on the interval $ (a, b) $. By definition, for any $ x_1, x_2 \in (a, b) $ with $ x_1 < x_2 $, we must have $ f(x_1) \geq f(x_2) $. Therefore, $ f(x) $ does not exhibit any increasing behavior on $ (a, b) $.

Now, assume that $ \lim_{x \to -\infty} f(x) = 0 $ and that $ f(x) $ decreases by an amount $ d $ throughout the interval $ a < x < b $. This implies that:

\[
f(b) = f(a) - d,
\]

where $ d > 0 $ because $ f(x) $ is decreasing. Therefore, $ f(b) $ is less than $ f(a) $. Additionally, assume that $ f(0) = 0 $ and $ d < f(0) = 0 $, meaning that $ f(b) = d $ for some negative value of $ d $, i.e., $ f(b) = d < 0 $.  By the Intermediate Value Theorem (IVT), since $ f(x) $ is continuous and decreasing, for any value $ c \in (f(b), f(a)) $, there must exist some $ x_0 \in (a, b) $ such that $ f(x_0) = c $. In particular, for any small $ \epsilon > 0 $, we can choose $ c = d + \epsilon $, so there exists some $ x_0 \in (a, b) $ such that:

\[
f(x_0) = d + \epsilon,
\]

where $d + \epsilon \in (d, 0) $.

However, this does not imply that $f(x) $ increases anywhere on $ (a, b) $. Since $ f(x) $ is monotonically decreasing, it achieves the value $ d + \epsilon $ while continuing to decrease from $ f(a) $ to $ f(b) $. Thus, there is no need for the function to increase at any point, and no contradiction arises from the IVT. The assertion that $ f(x) $ must increase at some point within $ (-\infty, 0] $ is incorrect. A monotonic function can still take values like $ d + \epsilon $ while remaining strictly decreasing. Therefore, the idea that  $f(x)$ must both decrease and increase within $(-\infty, 0]$ is false and does not follow from the given assumptions.

In conclusion, the assumption that $ f(x) $ is monotonic does not lead to any contradictions. As a result, the claim that $ f(x) $ must be non-monotonic is unfounded. $ f(x) $ can remain monotonically decreasing without violating the Intermediate Value Theorem or the given conditions.




\end{proof}

Having shown that a function that closely approximates ReLU must be non-monotonic, we proceed to compare our non-monotonic candidates in terms of ReLU approximation. We recall that in Table \ref{tab:ReLUProximity}, TeLU achieves the best approximation to ReLU along the active subdomain $x \in [0,\infty]$. On the other hand, GELU approximates ReLU best along the inactive subdomain $x \in (-\infty,0]$. We argue that closely approximating ReLU while it exhibits strong gradients in the active region outweighs closely approximating it along the near-zero inactive region. The active region exhibits stronger gradients, and therefore has a larger influence on the way a neural network learns. For this reason, we hypothesize that ReLU should operate more similarly to TeLU than GELU. We detail our reasoning in the proceeding lemma:

\begin{Lem}
Let $ r(x) = \text{ReLU}(x) = \max(0, x) $ be defined as:
\[
r(x) = 
\begin{cases} 
0, & \text{if } x \leq 0, \\ 
x, & \text{if } x > 0.
\end{cases}
\]

Define the active subdomain of $ r(x) $ as $ \mathcal{A} = [0, \infty) $ and the inactive subdomain as $ \mathcal{I} = (-\infty, 0) $.

Let $ t(x) = \text{TeLU}(x) $ and $ g(x) = \text{GELU}(x) $ be two smooth, continuous, and differentiable activation functions such that:

1. $ t(x) \approx r(x) $ for $ x \in \mathcal{A} $,
2. $ g(x) \approx r(x) $ for $ x \in \mathcal{I} $.

Furthermore, let $ f(x) $ be an arbitrary continuous and differentiable function that approximates $ r(x) $ over the entire domain $ x \in \mathbb{R} $.

Define the approximation error for $ f(x) $ over the active subdomain $ \mathcal{A} $ and inactive subdomain $ \mathcal{I} $ as:

\[
E_{\mathcal{A}}(f) = \int_{0}^{\infty} \left| f(x) - r(x) \right|^2 \, dx,
\]
\[
E_{\mathcal{I}}(f) = \int_{-\infty}^{0} \left| f(x) - r(x) \right|^2 \, dx.
\]

If $ E_{\mathcal{A}}(t) < E_{\mathcal{A}}(g) $ and $ E_{\mathcal{I}}(g) < E_{\mathcal{I}}(t) $, then $ t(x) $ provides a closer approximation to $ r(x) $ in the active region, while $ g(x) $ provides a closer approximation to $ r(x) $ in the inactive region.

Moreover, the influence of $ f(x) $ on neural network training is determined by the gradients in each subdomain:

\[
I_{\mathcal{A}}(f) = \int_{0}^{\infty} \left| f'(x) \right|^2 \, dx,
\]
\[
I_{\mathcal{I}}(f) = \int_{-\infty}^{0} \left| f'(x) \right|^2 \, dx.
\]

If $ I_{\mathcal{A}}(t) > I_{\mathcal{I}}(g) $, then $ t(x) $ has a stronger impact on training in the active region than $ g(x) $ in the inactive region.

If $ t(x) $ and $ g(x) $ are used to approximate $ r(x) $, then $ t(x) $ is a better surrogate for ReLU in neural network training due to the stronger gradients and smaller approximation error in the active subdomain $ \mathcal{A} $. Thus, for tasks where the active region influences learning more heavily, $ t(x) $ should behave more similarly to $ r(x) $ than $ g(x) $.

\end{Lem}

TeLU demonstrates gradient behavior that closely mirrors the strong identity growth seen in the ReLU activation function. Specifically, when the inputs are positive, both TeLU and ReLU produce active gradients that are approximately equal to one, leading to a convergence speed during training that is comparable for both functions.


This rapid convergence arises because strong gradients enable more significant parameter updates with each training iteration, helping the model learn faster. However, this property also implies that TeLU, like ReLU, does not impose much implicit regularization on the model. Implicit regularization refers to the tendency of some activation functions to naturally limit the magnitude of updates, thereby acting as a form of regulation on model complexity. In the case of ReLU and TeLU, the strong gradients reduce this implicit regularization effect, allowing the network to fit the data more efficiently.

In contrast, other activation functions, such as Logish or Smish, have weaker active gradients—especially as inputs move away from zero. This results in a "damping" effect where gradient updates are smaller, slowing down convergence. This behavior effectively regularizes the learning process by preventing overly large parameter updates, which can help in scenarios where controlling model complexity and preventing overfitting are desired. Thus, the distinct behavior of TeLU, with its strong gradients and rapid convergence, sets it apart from these other functions that inherently moderate the learning speed through gradient damping.

This characteristic leads to a more modular design of the neural network training procedure, as other components, such as the optimization algorithm and learning rate scheduler, can take on a more direct role in managing learning steps. With the activation function playing a lesser role in regularization, the optimization algorithm can be fine-tuned to control learning aggressiveness through parameters like momentum and decay rates, while the learning rate scheduler dynamically adjusts based on training progress. This approach allows for more precise tailoring of the training dynamics using external techniques, resulting in more efficient and effective training. Consequently, network regularization and convergence behavior become more configurable, with the activation function primarily introducing non-linearity rather than acting as a built-in regularization tool.

In summary, non-monotonic functions are well-suited to smoothly approximate the $ReLU = \max(0, x)$ nonlinearity while mimicking ReLU's saturating behavior as $x \to -\infty$, its deactivation at $x = 0$, and its unbounded linear growth as $x \to \infty$. The TeLU activation function, in particular, closely replicates the inference and learning dynamics of models using ReLU when substituted in its place. Both TeLU and ReLU exhibit strong gradients that help reduce the implicit regularization imposed by the activation function, enabling a more modular training design. These similarities make TeLU an excellent drop-in replacement for ReLU in deep neural networks.


\subsection{Analytic Universal Approximation}
\label{subsect:AUAtheory}

\begin{Def}
    A function $g: \mathbb{X} \to \mathbb{R}^n$ whose domain $\mathbb{X} \subset \mathbb{R}^n$ is a universal approximator over $\mathbb{X}$ if and only if there exist constant vectors $\alpha$, $w$, and $b \in \mathbb{R}^n$ such that $| \sum_{j=1}^{n} \alpha_j \cdot g(w_j^\intercal x + b_j) - f(x)| < \epsilon$ for all $x \in \mathbb{X}$, where $f: \mathbb{X} \to \mathbb{R}^n$ is any arbitrary continuous function and $\epsilon$ is any arbitrarily small positive error.
\end{Def}

In other words, a univariate function $g(x)$ is considered a universal approximator if it can be used in a finite linear combination of compositions, each with an affine transformation of the input, to approximate any multivariate continuous function over a bounded domain. In the context of deep learning, this concept guarantees that there exists a neural network with a single hidden layer and a universal approximation activation function that can effectively perform any task, provided the target function to be approximated is continuous and the input range is bounded. This principle underpins the capability of neural networks to model a vast range of complex functions with sufficient precision.

A pivotal breakthrough that underpins the use of sigmoidal activation functions in neural networks is the universal approximation theorem, first established independently by Cybenko (1989), Hornik et al. (1989), and Funahashi (1989). This foundational theorem asserts that a standard neural network with a single hidden layer of activation $\sigma : \mathbb{X} \to \mathbb{R}^n$ can uniformly approximate any continuous function over a bounded domain $\mathbb{X} \subset \mathbb{R}^n$ provided that $\sigma$ is sigmoidal. A function $\sigma$ is sigmoidal if and only if $\lim_{x \to -\infty} \sigma(x) = 0$ and $\lim_{x \to \infty} \sigma(x) = 1$. The compact domain $\mathbb{X} \subset \mathbb{R}^n$ constraint is a given within digital computer systems that use a finite number of bits to represent numbers. In essence, this means that such networks, despite their shallow structure, possess the remarkable ability to approximate any continuous functions with arbitrary precision.

Building on this foundation, Hornik (1991) expanded the universal approximation theorem to include any continuous, bounded, and non-constant activation function, thereby broadening the potential of neural networks as universal approximators. Further refinement came from the work of Leshno et al. (1993), who demonstrated that for continuous activation functions, a standard neural network with one hidden layer can uniformly approximate any continuous function if the activation function employed by the hidden layer is non-polynomial. This highlights the importance of non-polynomial activation functions in ensuring the universal approximation capabilities of neural networks. Since $TeLU$ cannot be expressed as a polynomial, it is a universal approximator. However, we will demonstrate that $TeLU$ is a universal approximator ourselves with support from Cybenko's (1989) approach to gain a mathematical handle on the concept. This will allow us to be able to prove further properties regarding the quality of universal approximation of $TeLU$. Now, we will state and prove that $TeLU(x) = x \cdot tanh(e^x)$ is a universal approximator:

\begin{Lem}
Let $ I_n = [a, b]^n $ be an $ n $-dimensional hypercube, and let $ C(I_n) $ denote the space of continuous functions on $ I_n $, where $ a, b \in \mathbb{R} $ and $ a < b $. Suppose $ f \in C(I_n) $ is a continuous function. Then, there exists a finite linear combination of the form
\[
g(x) = \sum_{i=1}^{M} c_i \cdot TeLU(w_i^\intercal x + b_i)
\]
where $ TeLU(x) = x \cdot \tanh(e^x) $, such that for any $ \epsilon > 0 $,
\[
\lvert f(x) - g(x) \rvert < \epsilon \quad \forall x \in I_n,
\]
where $ w_i \in \mathbb{R}^n $, and $ c_i, b_i \in \mathbb{R} $.
\end{Lem}

\begin{proof}
We first analyze the asymptotic behavior of the function $ TeLU(x) = x \cdot \tanh(e^x) $:

- As $ x \to -\infty $, we have $ e^x \to 0 $ and therefore $ \tanh(e^x) \to 0 $, so
  \[
  TeLU(x) = x \cdot \tanh(e^x) \to 0.
  \]
- As $ x \to \infty $, $ e^x \to \infty $ and $ \tanh(e^x) \to 1 $, thus
  \[
  TeLU(x) = x \cdot \tanh(e^x) \to x.
  \]

Now define the function $ \sigma(x) = TeLU(x) - TeLU(x - 1) $. We claim that $ \sigma(x) $ behaves as a sigmoidal function:

- As $ x \to -\infty $, both $ TeLU(x) \to 0 $ and $ TeLU(x-1) \to 0 $, so
  \[
  \sigma(x) \to 0.
  \]
- As $ x \to \infty $, $ TeLU(x) \to x $ and $ TeLU(x - 1) \to x - 1 $, thus
  \[
  \sigma(x) = TeLU(x) - TeLU(x - 1) \to 1.
  \]

Thus, $ \sigma(x) $ has the asymptotic properties of a sigmoidal function: $ \sigma(x) \to 0 $ as $ x \to -\infty $ and $ \sigma(x) \to 1 $ as $ x \to \infty $.

By Cybenko's universal approximation theorem, for any $ f \in C(I_n) $ and any $ \epsilon > 0 $, there exist coefficients $ \alpha_j \in \mathbb{R} $, vectors $ y_j \in \mathbb{R}^n $, and biases $ b_j \in \mathbb{R} $ such that
\[
f(x) \approx \sum_{j=1}^{N} \alpha_j \sigma(y_j^\intercal x + b_j),
\]
where $ \sigma(x) $ is a sigmoidal function. Substituting $ \sigma(x) = TeLU(x) - TeLU(x - 1) $, this becomes:
\[
f(x) \approx \sum_{j=1}^{N} \alpha_j \left[ TeLU(y_j^\intercal x + b_j) - TeLU(y_j^\intercal x + b_j - 1) \right].
\]

Now, we can rewrite the sum as a linear combination of $ TeLU $-based terms. Defining $ M = 2N $ and $\oplus$ as a concatenating operator, we represent the following vector concatenations:
\[
\hat{y} = y \oplus y, \quad \hat{\alpha} = \alpha \oplus -\alpha, \quad \hat{b} = b \oplus (b-1),
\]
which leads to the expression:
\[
g(x) = \sum_{j=1}^{M} \hat{\alpha}_j TeLU(\hat{y}_j^\intercal x + \hat{b}_j).
\]

By Cybenko’s theorem, there exists a sum $ g(x) $ such that
\[
\lvert f(x) - g(x) \rvert < \epsilon \quad \forall x \in I_n.
\]
Thus, $ TeLU(x) = x \cdot \tanh(e^x) $ is a universal approximator for continuous functions.

\end{proof}


Having proven that $TeLU$ is a universal approximator, we have shown that, in theory, a neural network with a single hidden layer that employs the $TeLU$ activation function will approximate any unknown target function. An architecture that is capable of universal approximation enjoys the practical benefits of flexible modeling, being able to be trained to reach adequate performance in any task that shares its input and output dimensions. Additionally, with a universal approximator, deep learning engineers do not need to worry about if their chosen activation functions are capable of fitting to a task. Instead, they can focus their design efforts to scaling the model's parameterization to better fit the desired task to prevent overfitting. Therefore, engineers that apply continuous non-polynomial functions to their fully connected layers may significantly reduce their search to tuning other hyperparameters such as network size, learning rate, and optimizer choice.

A universal approximator ensures that a network can approximate any unknown target function, but not that it will approximate it well enough without overfitting. Some universal approximators like $ReLU$ are only capable of providing piecewise linear representations of tasks, and therefore lack representation power. The limited expressiveness of these $ReLU$ approximations amounts to linear interpolation, which is prone to instability when attempting to model smooth relationships between input and output. Other non-linearities such as $ELU$ offer a continuous first derivative, but a discontinuous second derivative, making it incompatible with second order optimization procedures that improve convergence speed and stability during training. Therefore, we require an additional property beyond standard universal approximator to concisely state this quality of universal approximation. We turn our attention towards the property of a function being analytic:



\begin{Def}
A function $ f(x) $ is said to be analytic at a point $ x_0 $ if there exists a neighborhood around $ x_0 $ such that $ f(x) $ can be represented as a convergent power series:

\[
f(x) = \sum_{n=0}^{\infty} c_n (x - x_0)^n
\]

for some constants $ c_n $, where the series converges to $ f(x) $ for all $ x $ in this neighborhood.
\end{Def}

When a function is analytic, it is implied that it is indefinitely differentiable at all points. This property makes them highly compatible with second-order optimization strategies that utilize the Hessian matrix of the cost function. By incorporating the Hessian matrix, optimization algorithms can account for the curvature of the loss function, leading to more stable and efficient convergence. Moreover, analytic functions facilitate the smooth transfer of information through the network, resulting in more dispersed neuron activations, which contribute to more robust and generalized representations (CITE). Additionally, analytic functions are compatible with a broader range of mathematical analysis techniques, such as spectral analysis and perturbation methods, which can further aid in understanding and improving the behavior of neural networks during training and inference. Since $TeLU$ is a smooth function composed of analytic functions, we claim and demonstrate that it is also an analytic function. In the following lemma, we demonstrate that TeLU is indeed an analytic function:

\begin{Lem}
$TeLU(x) = x \cdot \tanh(e^x)$ is an analytic function, where $x \in \mathbb{R}$.
\end{Lem}

\begin{proof}
We begin by noting that $ TeLU(x) = x \cdot \tanh(e^x) $ is the product of the identity function and the composition of the hyperbolic tangent and the exponential functions.

First, consider the exponential function $ e^x $, which has the well-known Maclaurin series expansion:

\[
e^x = \sum_{n=0}^{\infty} \frac{x^n}{n!}
\]

This series converges for all $x \in \mathbb{R}$, showing that $e^x$ is analytic.

Next, we recall the Maclaurin series expansion for the hyperbolic tangent function $ \tanh(x) $:

\[
\tanh(x) = \sum_{n=1}^{\infty} \frac{2^{2n} (2^{2n} - 1) B_{2n}}{(2n)!} x^{2n-1}
\]

where $ B_{2n} $ are the Bernoulli numbers. This series also converges for all $ |x| < \frac{\pi}{2} $, establishing that $ \tanh(x) $ is analytic within this domain.

Since $ \tanh(e^x) $ is the composition of the analytic functions $ \tanh(x) $ and $ e^x $, $ \tanh(e^x) $ is analytic for all $x \in \mathbb{R}$.

Finally, $ TeLU(x) = x \cdot \tanh(e^x) $ is the product of the analytic function $ \tanh(e^x) $ and the identity function $ x $, which is itself analytic. By the closure properties of analytic functions, the product of two analytic functions is also analytic.

Thus, $ TeLU(x) = x \cdot \tanh(e^x) $ is an analytic function, as it can be expressed as a power series:

\[
TeLU(x) = \sum_{n=0}^{\infty} b_n x^{n+1}
\]

where $ b_n \in \mathbb{R} $ are constants derived from the power series of $ \tanh(e^x) $. This completes the proof.
\label{thm:analyticTeLU}
\end{proof}

Having established that $TeLU$ is analytic, it is important to revisit the practical benefits this property brings to neural network architectures. The smoothness and infinite differentiability of $TeLU$ enhance its compatibility with second order optimization methods such as Natural Gradient Descent (NGD) \cite{NGD}. This leads to more stable convergence and improved training efficiency, especially in complex models. Additionally, the smooth transfer of information throughout the network helps promote balanced neuron activations, allowing for robust feature representations and better generalization to unseen data. The analytic nature of TeLU also opens up opportunities for utilizing sophisticated mathematical tools, offering deeper insights and potential for further optimization during training and inference. We now show that $TeLU$, being a universal approximator that is also analytic, leads to analytic universal approximations:

\begin{Thm}
The function $ TeLU(x) = x \cdot \tanh(e^x) $ is an analytic universal approximator. In other words, any continuous function $ f(x) $ can be approximated by a linear combination of $ TeLU $ nonlinearities, and the resulting approximation will be analytic.
\end{Thm}

\begin{proof}
We have already established that $ TeLU(x) = x \cdot \tanh(e^x) $ is analytic. Let $ f(x) \in C(\mathbb{R}^n) $ be an arbitrary continuous function, and let $ g(x) $ be a linear combination of $ TeLU $ nonlinearities that approximates $ f(x) $. Thus, for any $ \epsilon > 0 $, there exists a function $ g(x) $ of the form

\[
g(x) = \sum_{j=1}^{M} \alpha_j \cdot TeLU(w_j^\intercal x + b_j)
\]

such that the approximation error satisfies

\[
|g(x) - f(x)| < \epsilon \quad \forall x \in \mathbb{R}^n.
\]

Since $ TeLU(x) $ is analytic, we can express each $ TeLU(w_j^\intercal x + b_j) $ as a convergent power series. Therefore, the function $ g(x) $ becomes a summation of power series expansions for each term:

\[
g(x) = \sum_{j=1}^{M} \alpha_j \cdot \sum_{n=0}^{\infty} c_n (w_j^\intercal x + b_j)^n
\]

where $ c_n $ are the coefficients of the power series expansion. Rewriting the expression, we have:

\[
g(x) = \sum_{j=1}^{M} \sum_{n=0}^{\infty} \alpha_j c_n (w_j^\intercal x + b_j)^n = \sum_{j=1}^{M} h_j(x),
\]

where $ h_j(x) $ represents the power series expansion of each individual term $ TeLU(w_j^\intercal x + b_j) $. Since the sum of analytic functions is also analytic, the approximation $ g(x) $ is a sum of $ M $ analytic functions $ h_j(x) $, and thus $ g(x) $ is analytic.

Consequently, the linear combination of $ TeLU(x) $ nonlinearities that approximates $ f(x) $ is itself an analytic function. Therefore, we conclude that $ TeLU(x) = x \cdot \tanh(e^x) $ is an analytic universal approximator.
\end{proof}

With the proof that an architecture employing TeLU as its hidden activation function is capable of analytic universal approximations, the model achieves several theoretical and practical advantages. Theoretically, this capability ensures that the architecture can approximate a wide range of continuous and differentiable functions with arbitrary precision, which is essential in modeling complex, nonlinear systems. Analytic universal approximation guarantees that not only can the model represent any function, but it can do so with a smooth, well-behaved structure, preserving differentiability across all layers. This property is crucial in tasks requiring optimization through gradient-based methods, as it mitigates issues related to non-smoothness and sharp transitions in the loss landscape. The analytic nature also allows the model to generalize better by capturing underlying relationships in data, rather than fitting to noise or outliers.

Practically, having an architecture with analytic universal approximation capabilities means that the model can be applied to a broader set of problems across domains such as physics, engineering, and finance, where smooth and continuous approximations are needed for high-fidelity simulations and predictions. For instance, in control systems or physical simulations, TeLU-enabled architectures can model complex dynamics with high precision, leading to better predictive power and stability. Additionally, this property often results in more efficient training, as the smoothness in activation and loss functions can lead to faster convergence and more reliable gradient flow, addressing problems like vanishing gradients.


Having established that $ TeLU(x) = x \cdot \tanh(e^x) $ is an analytic universal approximator, we now explore the theoretical and practical benefits of architectures capable of representing problems with analytic universal approximations.

\subsubsection{Theoretical Benefits}

\begin{Prop}
\textit{Guaranteed Approximability of Continuous Functions:} An architecture using analytic universal approximators can approximate any continuous function $ f(x) \in C(\mathbb{R}^n) $ to an arbitrary degree of accuracy. This is ensured by the universal approximation theorem.
\end{Prop}

Since $ TeLU(x) $ is analytic, it guarantees that smooth, continuous functions can be approximated with a high degree of precision. The smoothness of analytic functions, which are infinitely differentiable within their domain, enables the architecture to handle complex functional relationships.

\begin{Prop}
\textit{Convergence Guarantees:} Due to the well-behaved nature of analytic functions, architectures using $ TeLU(x) $ exhibit stable and efficient convergence properties when optimized using gradient-based methods. 
\end{Prop}

Since analytic functions have continuous derivatives, the resulting loss surface is smooth, allowing gradient-based optimization methods to converge faster and more reliably to optimal solutions without getting stuck in local minima.

\begin{Prop}
\textit{Higher-Order Information:} Analytic functions are infinitely differentiable, meaning that architectures with analytic approximators can capture not only the function's value but also higher-order derivatives, enabling more precise modeling of the underlying process.
\end{Prop}

This property is crucial in tasks that require higher-order information, such as physics simulations, where capturing the curvature of the function is essential. It also allows architectures to model complex dynamical systems that require higher-level functional representations.

\begin{Prop}
\textit{Avoidance of Pathological Behavior:} Non-analytic functions may exhibit discontinuities or oscillations, which are not desirable in many practical applications. Analytic approximators avoid these issues by ensuring smooth and well-behaved representations.
\end{Prop}

Non-analytic functions can introduce undesirable behaviors, such as discontinuities or oscillations, which may complicate the learning process and reduce the model’s effectiveness in practical applications. By contrast, architectures using analytic approximators, like TeLU, ensure that the approximated functions are smooth and continuous. This smoothness prevents pathological behaviors, enabling the model to produce well-behaved representations that are more reliable and predictable, particularly in real-world scenarios where stability and smooth transitions are critical.


\subsubsection{Practical Benefits}

\begin{Prop}
\textit{Improved Generalization:} An analytic universal approximator tends to generalize better to unseen data, as it avoids overfitting to noise or specific artifacts in the training data by leveraging smooth approximations.
\end{Prop}

In practice, smooth approximations help models capture the underlying patterns of the data without overfitting, making them more robust to noisy or incomplete datasets. This leads to improved performance on real-world tasks, particularly in domains with complex, non-linear relationships.

\begin{Prop}
\textit{Efficiency in Optimization:} Architectures using analytic functions provide smooth loss landscapes, allowing for more efficient optimization. This reduces the likelihood of getting trapped in local minima and speeds up convergence during training.
\end{Prop}

Since the gradients of analytic functions are stable and predictable, optimization techniques such as gradient descent perform more effectively. In practical terms, this means that training models with analytic approximators can be faster and require fewer iterations, making them suitable for large-scale machine learning tasks.

\begin{Prop}
\textit{Better Interpretability and Differentiability:} Analytic universal approximators offer a smooth, interpretable functional form, which makes them easier to understand and analyze. They are also well-suited for tasks requiring sensitivity analysis or gradient-based algorithms.
\end{Prop}

Because analytic functions are smooth and continuous, they are more interpretable. Additionally, their differentiability allows for the use of gradient-based optimization methods in tasks such as reinforcement learning or control systems, where differentiability is critical.

\begin{Prop}
\textit{Robustness to Numerical Precision:} Analytic functions, due to their smoothness, are less prone to numerical instability, making them more reliable when implemented on hardware with limited floating-point precision.
\end{Prop}

This is particularly important for edge computing or embedded systems, where precision is limited. Analytic universal approximators like $ TeLU(x) $ provide stable computations, reducing the risk of errors caused by floating-point approximations.

\begin{Prop}
\textit{Applicability in Physics-Informed Models:} Many real-world problems, such as those in physics or engineering, are governed by smooth, differentiable equations. Analytic approximators are well-suited for these domains because they can represent smooth transitions and capture physical laws more accurately.
\end{Prop}

In practical applications like fluid dynamics, electromagnetism, or mechanical systems, analytic functions provide a natural fit due to their inherent smoothness and ability to model physical systems with high precision.

\begin{Prop}
\textit{Smoothness in Adversarial Settings:} Analytic functions are naturally resistant to adversarial perturbations because small changes in input lead to small, predictable changes in output, making it harder for adversarial examples to exploit vulnerabilities in the model.
\end{Prop}

In adversarial machine learning, smoothness helps protect the model from being tricked by small, carefully crafted perturbations. This makes architectures based on analytic approximators more robust against adversarial attacks, enhancing security and reliability in sensitive applications.

Thus the architectures that utilize analytic universal approximators, such as $ TeLU(x) = x \cdot \tanh(e^x) $, offer significant theoretical and practical advantages. They guarantee smooth, continuous, and well-behaved approximations, provide better generalization, and are computationally efficient. These properties make analytic universal approximators highly valuable for a wide range of applications, from scientific computing and optimization to adversarial robustness and interpret-ability.

\subsection{Stability}
\label{subsect:stabilitytheory}





In deep neural networks, stability refers to the network's ability to produce consistent and reliable outputs when exposed to small perturbations or changes in input data, weights, or training conditions. Theoretical stability is crucial for ensuring that the network generalizes well to unseen data and avoids issues like exploding or vanishing gradients, which can lead to erratic behavior during training. Without proper attention to stability, networks may suffer from poor convergence, unstable training dynamics, or a lack of reliability in deployment.


Activation functions play a crucial role in influencing the stability of deep neural networks by determining how neurons respond to input signals and propagate gradients during training. Functions like ReLU \cite{ReLU} mitigate the vanishing gradient problem, enabling better gradient flow in deep networks, but can lead to dead neurons, causing instability. Conversely, sigmoid and tanh activations are bounded functions that may limit instability, but suffer from vanishing gradients which leads to slow convergence and unstable learning. More recent activation functions such as Leaky ReLU \cite{LReLU} and ELU \cite{elu}, are designed to balance gradient flow and prevent instability by reducing output bias and thereby approach Fisher ideal learning. Smooth activation functions have also been observed to promote the learning stability of a mode, as demonstrated by Zheng et al \cite{Softplus} and Ramachandran et al \cite{Ramachandran2017SwishAS}.



The vanishing gradient and exploding gradient problems directly impact the stability of deep neural networks by making weight updates during training inefficient or unstable. In the vanishing gradient problem, small gradients in deep layers prevent effective weight updates, causing the network to learn slowly or stop learning altogether. On the other hand, exploding gradients cause large weight updates, leading to oscillations and divergent training. By addressing these issues with activation functions that mitigate both underflow and overflow, deep networks maintain stable learning, ensuring consistent weight updates and convergence.


Activation functions that exhibit near-zero expected output contribute to neural network stability by reducing output bias, ensuring that neurons do not consistently produce large positive or negative values. This mitigates the risk of saturating activation functions like sigmoid or tanh, where extreme outputs lead to vanishing gradients. By centering activations around zero, these functions allow for more balanced weight updates during backpropagation. This behavior aligns with Fisher's Ideal Learning, which suggests that stable learning arises when gradients provide unbiased information about the parameter space, facilitating efficient and consistent training.


Activation functions that are smooth and analytic contribute to learning stability in neural networks by ensuring that gradients change gradually and predictably during training. The smoothness of these functions avoids sudden jumps in gradient values, which can destabilize learning by causing erratic weight updates. Their analytic properties enable continuous differentiability, which is compatible with second order optimization techniques that utilize the curvature of the loss landscape to promote stability of learning. These characteristics allow for more consistent convergence, especially in deep networks, by preserving valuable gradient information throughout the layers through dispersity of activation.


Activation functions that are simple to calculate, such as ReLU and Leaky ReLU, contribute to learning stability by reducing the computational complexity of training deep neural networks. Simpler calculations minimize the risk of numerical errors like underflow or overflow, which can occur when nonlinear operations are repeatedly applied for the activation of a single neuron. Once underflow or overflow emerges within a term of a complex function, outer functions are likely to be driven towards additional underflow or overflow. Subsequent layers may also propagate the numerical instability to all the outputs, directly impacting inference. By avoiding these issues, simple activation functions help maintain numerical precision and prevent cascading errors in gradient computations, ensuring more stable learning, especially in deep architectures.


\begin{table}[]
\centering
 \caption[Stability Heuristics of Activation Functions.]
 {\textbf{Stability Heuristics of Activation Functions.} Summary of heuristics on which properties of an activation function most impact the learning stability of a model. 'Nesting' quantifies the depth of nested non-linearities within each activation function which may be adverse to numerical stability. 'Vanishing' informs if a function saturates towards deactivation at an exponential rate or faster, resulting in more common vanishing gradients. 'Smooth' summarizes whether a function is differentiable at all points. 'Output Bias' measures the expected output of the activation function, given input taken from a standard Gaussian distribution.  }
    \begin{tabular}{||c c c c c ||} 
 \hline
 Function & Nesting & Vanishing & Smooth & Output Bias \\
 \hline\hline
 TeLU & 2 & No & Yes & 0.2621 \\ 
 \hline
 ReLU & 1 & Yes & No & 0.3989 \\
 \hline
 ELU & 2 & Yes & No & 0.1605\\
 \hline
 SiLU & 2 & No & Yes & 0.2066 \\
 \hline
 GELU & 2 & Yes & Yes & 0.2821 \\
 \hline
 Mish & 3 & No & Yes & 0.2404 \\
 \hline
 Logish & 2 & No & Yes & 0.1398 \\
 \hline
 Smish & 3 & No & Yes & 0.1201 \\
 \hline
\end{tabular}
\label{tab:StabilityTable}
\end{table}

We present a detailed summary of the heuristics that we believe most significantly influence the learning stability of a model, focusing on specific properties of activation functions. First, we consider the 'Nesting' column, which reflects the complexity introduced by the number of nested non-linearities within an activation function. Higher levels of nested non-linearities can adversely affect the numerical stability of a model, particularly as the network depth increases. Next, the 'Vanishing' column assesses the tendency of an activation function to approach saturation and deactivation. If a function's derivative exhibits exponential decay as inputs grow towards $-\infty$, it can result in a greater occurrence of vanishing gradients, hindering effective backpropagation \cite{backprop} and slowing down learning in deep networks. In ReLU, this effect is exaggerated with derivatives of 0 for any negative input in what is known as the dying ReLU problem. Lu et al \cite{lu2019dying} discuss the strong effects of the dying ReLU problem in deep and narrow architectures, which are often necessary for efficient learning of complex behavior. The asymptotic decay of each activation function as inputs grow negative can be viewed in Table \ref{tab:Asymptotic} within Subsection \ref{subsect:persistentgradtheory}.

Our 'Smooth' column refers to the non-zero differentiability of the activation function across its entire input domain. A function that is continuously differentiable at all points promotes smoother gradient flow and reduces potential disruptions during the training process. ReLU and ELU are not smooth functions due to their linear definition at positive inputs. Meanwhile, the non-monotonic functions are all smooth, as they can be indefinitely differentiated at all points without degenerating to zero. Finally, the 'Output Bias' column evaluates the expected output of the activation function when the input follows a standard Gaussian distribution. This is evaluated as $\int^{\infty}_{-\infty} p(x) \cdot f(x) dx$, where $p(x)$ is the unit Gaussian distribution $\mathcal{N}(0,1)$, or $\frac{1}{\sqrt{2 \pi}} \cdot e^{- \frac{x^2}{2}}$. $f(x)$ is the mathematical definition of each activation function. This measure helps us understand whether the function introduces an inherent bias in its outputs, which can affect weight initialization and, ultimately, learning dynamics. The closer a function's output bias is to zero, the more it enhances the model's convergence efficiency \cite{LecunSymmetricAdvantage.66.2396}. 



We do not explicitly list exploding gradients in the Table, as their behavior requires a more nuanced, detailed analysis. While the activation functions in our Table are primarily linear units, behaving linearly as inputs grow large, the real advantage of our smooth non-monotonic functions—TeLU, SiLU, GELU, Mish, Logish, and Smish—emerges at small positive inputs. These functions exhibit sub-linear growth, meaning their outputs increase more slowly than their inputs under these conditions. This sub-linearity is critical in controlling gradient behavior, as it naturally prevents the derivatives from becoming excessively large during backpropagation.
By reducing the likelihood of gradients exploding, non-monotonic activation functions contribute significantly to greater learning stability. This stability is essential for deep networks, particularly when training with complex data or architectures that are prone to issues with gradient propagation. As exploding gradients can severely hinder convergence and result in unstable weight updates, choosing activation functions that naturally limit gradient magnitudes offers a practical solution to this widespread problem.

Mitigating exploding gradients is not only about achieving smoother training but also about improving the network's overall learning efficiency. When gradients are kept within a controlled range, the model can train without sudden fluctuations in the loss function, resulting in more stable and predictable weight updates. This stability helps the optimization process stay on course, leading to faster convergence and better generalization to new, unseen data. Activation functions like TeLU and its non-monotonic counterparts play a crucial role in promoting this stability by naturally preventing the gradients from becoming excessively large. By employing smooth, analytic functions like TeLU, these architectures inherently avoid issues such as abrupt changes or sharp boundaries in the model's output. This makes them particularly well-suited for tasks where continuity and smoothness are critical, such as in control systems, robotics, or financial modeling, where precise, stable responses are essential for success.

This can be theoretically shown as follows:

\begin{Thm}
Let $ f(x) \in C(\mathbb{R}^n) $ be a continuous function. Suppose we approximate $ f(x) $ using two architectures: one based on an analytic universal approximator $ TeLU(x) = x \cdot \tanh(e^x) $, and the other based on the non-analytic activation function $ ReLU(x) = \max(0, x) $. The analytic universal approximator provides the following mathematically provable advantages:

\begin{enumerate}
    \item \textit{Smoothness and Continuity of Derivatives:} The analytic universal approximator is infinitely differentiable, whereas ReLU is piecewise linear and not differentiable at $ x = 0 $.
    
    \item \textit{Gradient Propagation Stability:} In deep networks, the gradient propagation using analytic functions remains bounded and stable, whereas ReLU introduces points where gradients vanish or explode, leading to poor training dynamics.
    
    \item \textit{Error Propagation in Deep Networks:} Analytic functions provide smoother error propagation, leading to better control of gradient norms, while ReLU propagates discontinuous gradients, leading to unstable optimization.
\end{enumerate}
\end{Thm}

\begin{proof}
We will prove each point in turn.

\textit{1. Smoothness and Continuity of Derivatives:}

As demonstrated in \ref{thm:analyticTeLU} $ TeLU(x) = x \cdot \tanh(e^x) $, TeLU is analytic. In being analytic, TeLU may be defined as a convergence power series and is therefore infinitely differentiable at all points.





\textit{2. Gradient Propagation Stability:}

Consider a deep network with $ L $ layers. Let $ g^{(l)}(x) $ represent the gradient of the output with respect to the input at layer $ l $, defined as:
\[
g^{(l)}(x) = \frac{d}{dx} \left( W^{(l)} \cdot a^{(l-1)}(x) \right),
\]
where $ W^{(l)} $ is the weight matrix and $ a^{(l-1)}(x) $ is the activation of layer $ l-1 $.

Using an analytic activation function like $ TeLU(x) $, the gradient propagation remains continuous and smooth:
\[
g^{(l)}(x) = W^{(l)} \cdot \frac{d}{dx} TeLU(x) = W^{(l)} \cdot \left( tanh(e^x) + x \cdot e^x \cdot sech^2(e^x) \right),
\]
which is bounded and continuous for all $ x \in \mathbb{R} $. The gradient never vanishes nor explodes, as the derivative of $ TeLU(x) $ is smooth for all $ x $.

In contrast, for $ ReLU(x) $, the gradient at layer $ l $ is:
\[
g^{(l)}(x) = W^{(l)} \cdot \frac{d}{dx} ReLU(x) = W^{(l)} \cdot 1(x > 0).
\]

Where $1(x > 0)$ evaluates to 1 if $x>0$ and 0 if $x \leq 0$. At $ x = 0 $, the gradient can be discontinuous or zero, leading to two major issues:
1. \textit{Vanishing gradients} when a large number of neurons have negative input, causing the gradient to become zero and halting learning.
2. \textit{Exploding gradients} in deep layers, especially when weights $ W^{(l)} $ are large, causing sharp changes in the gradient due to the piecewise linear nature of ReLU.

Thus, analytic universal approximators offer more stable gradient propagation across deep networks.

\textit{3. Error Propagation in Deep Networks:}

Let $ \mathcal{L}(x) $ represent the loss function of the network, and let the gradient of the loss with respect to the input at the final layer be $ \frac{\partial \mathcal{L}}{\partial x} $.

For the analytic activation $ TeLU(x) $, the gradient propagation follows from the smooth nature of $ TeLU $:
\[
\frac{\partial \mathcal{L}}{\partial x} = \frac{\partial \mathcal{L}}{\partial a^{(L)}} \cdot \frac{d}{dx} TeLU(x) = \frac{\partial \mathcal{L}}{\partial a^{(L)}} \cdot \left( tanh(e^x) + x \cdot e^x \cdot sech^2(e^x) \right).
\]
The gradient is well-controlled due to the smooth behavior of $ TeLU(x) $, ensuring that error propagation through layers remains stable and predictable.

For $ ReLU(x) $, the error propagation becomes problematic:
\[
\frac{\partial \mathcal{L}}{\partial x} = \frac{\partial \mathcal{L}}{\partial a^{(L)}} \cdot \frac{d}{dx} ReLU(x) = \frac{\partial \mathcal{L}}{\partial a^{(L)}} \cdot \mathbb{1}(x > 0).
\]
This creates sharp, discontinuous changes in the error gradient when $ x $ crosses zero, leading to unpredictable error propagation. Furthermore, the gradient can become zero when large numbers of neurons are inactive (i.e., in the dead zone), effectively halting the learning process for those weights.

Thus, analytic activations offer more stable and predictable error propagation in deep networks.

\end{proof}

Smoothness and continuity of derivatives, gradient propagation stability, and error propagation stability are crucial for ensuring the effective training of deep neural networks. Smooth activation functions with continuous derivatives allow for stable and predictable gradient updates, preventing erratic changes during backpropagation. This, in turn, promotes gradient propagation stability, which is essential for maintaining useful gradient values across many layers, avoiding the problems of vanishing or exploding gradients. Stable error propagation further ensures that the network can consistently learn from its mistakes, leading to more reliable convergence and better generalization to unseen data. Together, these properties form the foundation of stable and efficient learning in deep networks.

\newpage
\section{Experimental Validation}

This section presents a detailed experimental validation of the TeLU activation function, building on the theoretical foundation established in the previous section. Here, we provide empirical results that support and confirm TeLU's proposed benefits, such as enhanced gradient behavior, improved convergence speed, computational efficiency, and stability. By testing TeLU across various neural network architectures and datasets, we aim to demonstrate its effectiveness in addressing challenges like inefficient learning, the vanishing gradient problem, and computational complexity. These experiments serve to bridge the gap between theory and practice, offering concrete evidence of TeLU’s potential as a superior alternative to traditional activation functions.

\subsection{Persistent Gradients of the Saturation Region}
\label{subsect:persistentgradexp}

In Subsection \ref{subsect:persistentgradtheory}, we examined the rate at which the derivatives of various activation functions vanish as inputs approach $x \to -\infty$. By categorizing each function's vanishing gradient behavior using asymptotic decay classes, we identified TeLU as belonging to the class with the most persistent gradients, $\Theta(\frac{x}{e^x})$. Additionally, we noted that TeLU maintains stronger derivatives compared to other nonlinearities, particularly at the onset of its deactivation region.

\subsubsection{FashionMNIST Dataset on MLP with Negative Biases}

We demonstrate the benefit of these characteristics experimentally by attempting to train Multi-Layer Perceptrons (MLPs) where neurons have been initialized with a bias of -10. The large negative bias dictates the activation of the neuron towards the saturating region of lower-bounded nonlinearities. By introducing this large negative bias, we evaluate the ability of the architecture to recover from a point where backpropagations yields vanishing gradients. This extreme case lets us focus on the direct effects of each activation function's asymptotic saturation rate. We utilize two hidden layers and train on the FashionMNIST dataset \cite{MNIST} for 200 epochs. We optimize according to the SGD optimizer with a momentum of 0.9 and a weight decay of 0.0005. This reduced hidden layer count effectively reduces the number of epochs that we would need to run this experiment to witness activation function recovery. We outline our configuration in table \ref{tab:NegativeBiasHyperparams}.

The presence of weight decay does not significantly impact recovery in practice, as the bias is what is keeping the neuron activations in their saturation regions. We display the validation accuracies over the 200 epochs in Figure \ref{fig:neg10bias200} and testing accuracies in Table \ref{tab:negbiastestsummary}, with each data point averaged over 10 trials. We observe that TeLU tends to begin its successful stride towards recovery before any other nonlinearity, due to having stronger inactive gradients near the origin where expected inputs are likely to fall. We observe Mish, SiLU, Smish, and Logish begin their recovery sooner than ELU, GELU, and ReLU, which is what we would expect to see from our analysis. Furthermore, TeLU's accuracy approaches a competitive level at a faster rate than do the competing activation functions.

\begin{table}[b]
    \centering
    \caption[MLP FashionMNIST Negative-Bias Experiment Configuration.]
    {\textbf{MLP FashionMNIST Negative-Bias Experiment Configuration.} Configuration of experiments ran on an MLP model on the FashionMNIST dataset with biases initialized to -5 to test architecture ability to recover from vanishing gradients. }
    \label{tab:NegativeBiasHyperparams}
    \begin{tabular}{||c c||} 
 \hline
 Hyperparameter & Value \\
 \hline\hline
 Train/Val/Test Split & 50,000 / 10,000 / 10,000\\ 
 \hline
 Hidden Layer Count & 2\\ 
 \hline
 Hidden Layer Width & 128\\ 
 \hline
 Initialization & Xavier Uniform\\ 
 \hline
 Optimizer & Mini-batch SGD\\
 \hline
 Momentum & 0.9\\
 \hline
 Batch Size & 128\\
 \hline
 Learning Rate & 0.005\\
 \hline
 Weight Decay & 0.0005\\
 \hline
 Number of Epochs & 200\\
 \hline
 Number of Trials & 10\\
 \hline
\end{tabular}
\end{table}

\begin{figure}
    \centering
    \includegraphics[width=0.75\linewidth]{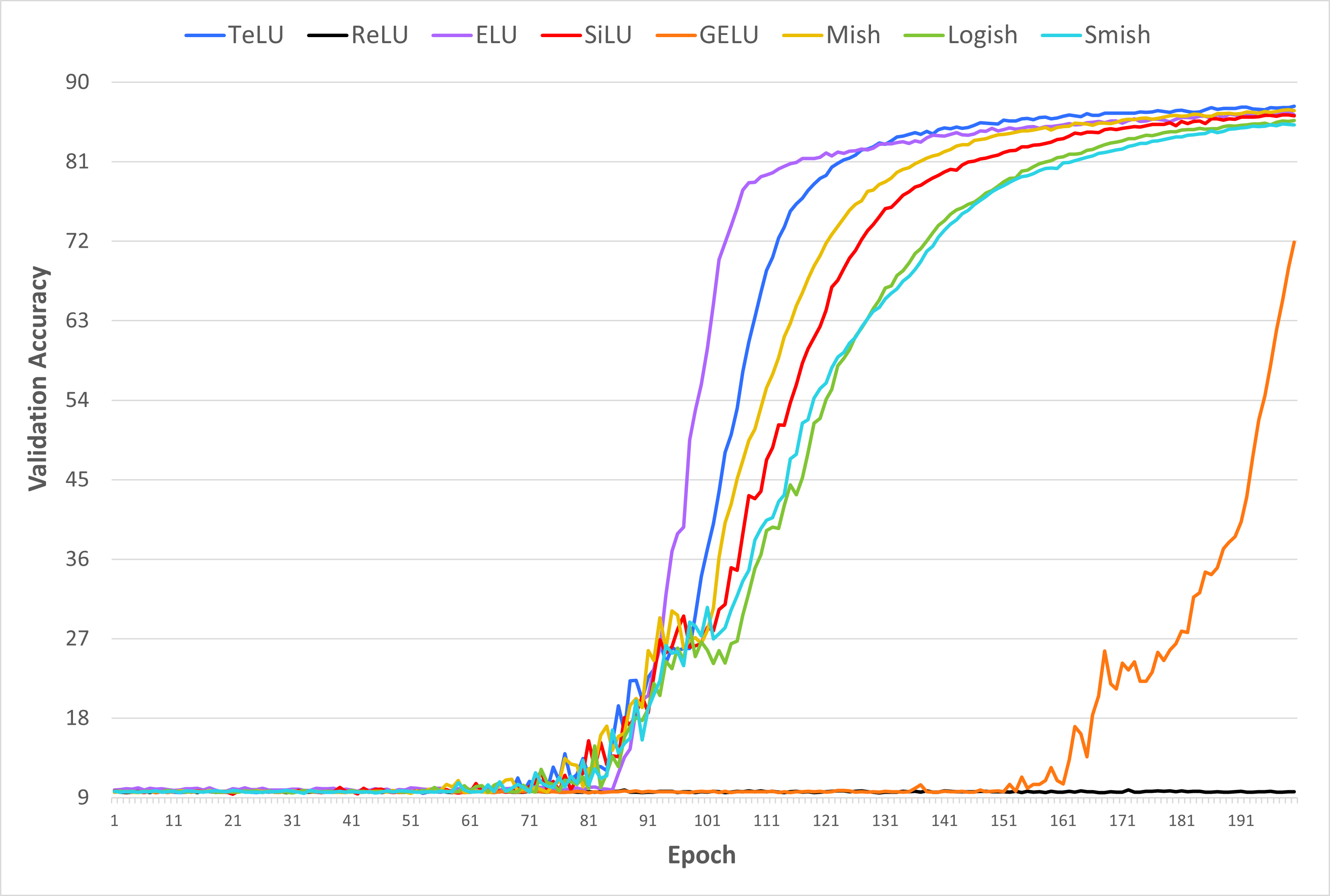}
    \caption[MLP FashionMNIST Light-Bias 200-Epoch Validation Accuracy.]
    {\textbf{MLP FashionMNIST Light-Bias 200-Epoch Validation Accuracy.} Validation accuracy progression of MLP models employing different hidden activation functions. Each architecture is initialized with bias values of -10 and trained over 300 epochs. Each data point is averaged over 10 trials.}
    \label{fig:neg10bias200}
    \vspace{1.0\baselineskip}
\end{figure}

We demonstrate the consistency of this experiment by initializing the bias learned parameters to -20, pushing neurons further into their saturation regions. We plot the validation accuracies for each epoch, averaged over 9 trials in Figure \ref{fig:neg20bias200}. We notice that TeLU continues to be the first to recover from vanishing gradients, giving further empirical evidence that it is less prone to stalling learning when TeLU neurons are pushed into their inactive regions. We also note that the order of each architecture begins their recovery in the same order as before, adding the relevance of the asymptotic decay classifications. Mish, SiLU, Logish, and Smish; belonging to the same asymptotic decay class as TeLU but having weaker gradients near the origin; being their recovery before ELU, GELU, and ReLU. We summarize our test accuracies averaged over 10 trials per data point for the 200 epoch experiments in Table \ref{tab:negbiastestsummary}. We recognize that within the 200 epochs TeLU architectures generalize to a superior testing accuracy in both cases.

For biases intialized to -20, we notice that 200 epochs may not be enough time for GELU to begin its recovery. Therefore we increase the number of epochs to 300 to gain further insight into the remainder of each architecture's recovery. Figure \ref{fig:neg20bias300} plots the 300 epoch journey of each architecture, averaged over 5 trials. Again, we see nonlinearities begin their recovery in the order of TeLU, Mish, SiLU, Logish, Smish, ELU, and GELU. It is worth noting that although the order of recovery follows our asymptotic classifications, the rate of recovery once begun depends on other factors we have not yet accounted for. We hypothesize that TeLU's near-linearity and ELU's linearity throughout positive inputs helps both shoot up to their concluding validation accuracies faster than other nonlinearities. Throughout all trials, ReLU exhibits no ability to recover, as is expected from its zero-gradient off state, which characterizes its dying neuron problem. 


\begin{figure}
    \centering
    \includegraphics[width=0.75\linewidth]{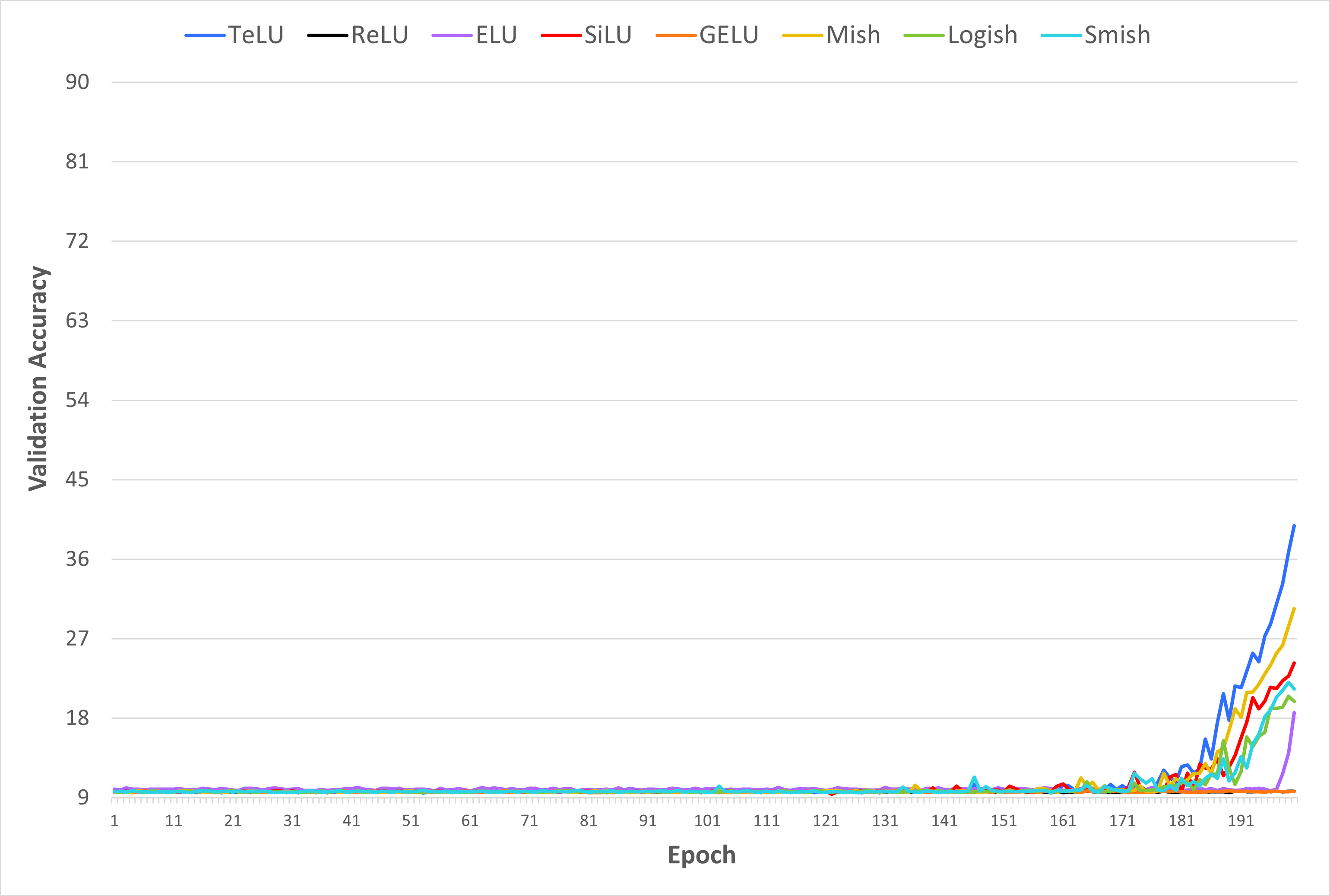}
    \caption[MLP FashionMNIST Heavy-Bias 200-Epoch Validation Accuracy.]
    {\textbf{MLP FashionMNIST Heavy-Bias 200-Epoch Validation Accuracy.} Validation accuracy progression of MLP models employing different hidden activation functions. Each architecture is initialized with bias values of -20 and trained over 200 epochs. Each data point is averaged over 10 trials.}
    \label{fig:neg20bias200}
    \vspace{1.0\baselineskip}
\end{figure}


\begin{table}
\centering
 \caption[MLP FashionMNIST Negative-Bias Test Accuracy Averaged over 10 Trials.]
 {\textbf{MLP FashionMNIST Negative-Bias Test Accuracy Averaged over 10 Trials.} Testing accuracies of MLP models employing various hidden activation functions after being initialized to negative biases. This experiment tests the ability of each architecture to recover from vanishing gradients over the course of 200 epochs, as an indication of being able to continue learning effectively despite vanishing gradients.}
    \begin{tabular}{||c c c||} 
 \hline
 $f(x)$ & bias=-10 & bias=-20 \\
 \hline\hline
 $TeLU(x)$ & \textbf{86.41}$\pm$0.37 & \textbf{39.98}$\pm$9.27 \\ 
 \hline
 $ReLU(x)$ & 10.00$\pm$0.00 & 10.0$\pm$0.00 \\
 \hline
 $ELU(x)$ &  85.58$\pm$0.36 & 18.53$\pm$15.6 \\
 \hline
 $SiLU(x)$ & 85.31$\pm$0.24 & 24.99$\pm$3.18 \\
 \hline
 $GELU(x)$ & 71.49$\pm$5.25 & 10.0$\pm$0.00 \\
 \hline
 $Mish(x)$ & 86.09$\pm$0.25 & 31.01$\pm$2.65 \\
 \hline
 $Logish(x)$ & 85.00$\pm$0.69 & 20.04$\pm$2.21 \\
 \hline
 $Smish(x)$ & 84.606$\pm$0.73 & 21.55$\pm$1.68 \\
 \hline
\end{tabular}
\label{tab:negbiastestsummary}
\vspace{1.0\baselineskip}
\end{table}

\begin{figure}
    \centering
    \includegraphics[width=0.75\linewidth]{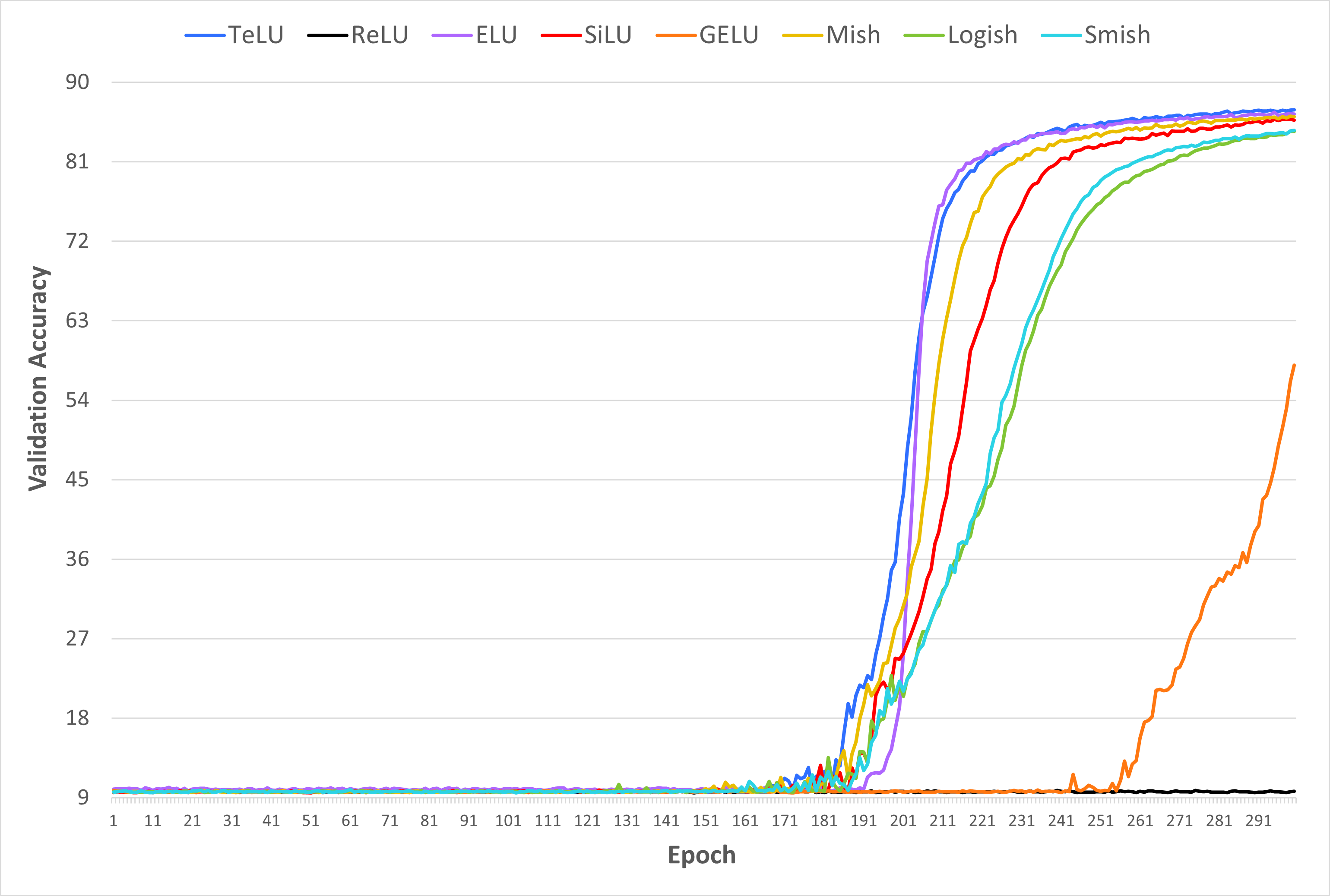}
    \caption[MLP FashionMNIST Heavy-Bias 300-Epoch Validation Accuracy.]
    {\textbf{MLP FashionMNIST Heavy-Bias 300-Epoch Validation Accuracy.} Validation accuracy progression of MLP models employing different hidden activation functions. Each architecture is initialized with bias values of -20 and trained over 300 epochs. Each data point is averaged over 5 trials.}
    \label{fig:neg20bias300}
\end{figure}


 In practice, the Hyperbolic Tangent is generally preferred over the Logistic Sigmoid for preventing vanishing gradient problems when properly regularized. Similarly, we believe that TeLU and SiLU offer more persistent gradients with proper regularization, which may explain why Mish, Logish, and Smish are often considered to have a stronger self-regularizing effect. We capitalize the practical applications of slow gradient decay by initializing neuron biases in a multi-layer perceptron (MLP) to varying negative values. We then perform training on the MNIST datasets to test the ability of neural networks employing different hidden nonlinearities to reactivate their inactive neurons. To simplify results, we take nonlinearities that have shown a difference in their decay rates, numerical underflow regions, or inactive gradient calculations. Therefore, we use ReLU, ELU, GELU, Mish, and TeLU nonlinearities.

 \subsubsection{CIFAR-10 Dataset with DenseNet CNN}

We hypothesize that the pronounced vanishing gradient effects associated with GELU and ReLU activation functions may be particularly harmful in CNN architectures lacking bypass connections. CNNs process local information from input images through convolution, where small square filters are applied around each pixel. This local information is passed through multiple convolution layers, often combined with pooling operations to reduce the signal's dimensionality. In architectures like DenseNet \cite{DenseNet}, fully connected neural blocks are interleaved between convolutional steps. Unlike Residual Networks (ResNet) \cite{ResNet}, DenseNets do not have bypass connections to facilitate proper gradient flow to all layers. Consequently, vanishing gradients can significantly impede the learning process of such architectures. In extreme cases, this stalling can have lasting effects, ultimately reducing the model's ability to fit the data. Therefore, we anticipate that activation functions prone to "dying neurons," such as ReLU and GELU, may severely limit the performance of DenseNets.


To evaluate our hypothesis, we use a CNN without residual connections to compare the testing accuracies of architectures that employ TeLU, GELU, and ReLU as their hidden layer activations. Specifically, we implement TeLU, GELU, and ReLU versions of the DenseNet121 architecture \cite{DenseNet} and train them on the CIFAR10 dataset \cite{cifar} for 200 epochs each. The dataset is partitioned into 40,000 training images, 10,000 validation images, and 10,000 testing images. We use an SGD optimizer with a base learning rate of 0.1, along with a learning rate scheduler that reduces the learning rate by a factor of 0.2 every 60 epochs. Additionally, we apply L2 regularization with a weight decay coefficient of 0.0005 and perform mini-batch gradient descent \cite{minibatchgradientdescent} using a batch size of 128. Further experimental configurations are detailed in Table \ref{tab:DenseNetHyps}.

\begin{table}[]
    \centering
    \caption[DenseNet CIFAR-10 Configuration.]
    {\textbf{DenseNet CIFAR-10 Configuration.} Static Configuration of experiment ran of the DenseNet architecture on the CIFAR-10 dataset. }
    \label{tab:DenseNetHyps}
    \begin{tabular}{||c c||} 
 \hline
 Hyperparameter & Value \\
 \hline\hline
 Train/Val/Test Split & 40,000 / 10,000 / 10,000\\ 
 \hline
 Normalization & Standard Score, detailed in Table \ref{tab:NormalizationDetails} \\ 
 \hline
 Architecture & DenseNet\\ 
 \hline
 Initialization & Xavier Uniform\\ 
 \hline
 Optimizer & Mini-batch SGD with momentum\\ 
 \hline
 Batch Size & 128\\
 \hline
 Weight Decay & 0.0007\\
 \hline
 Base Learning Rate & 0.1\\
 \hline
 Learning Scheduler Period & 60\\
 \hline
 Learning Scheduler Decay & 0.2\\
 \hline
 Data Augmentations & Random 4-padded cropping and horizontal flipping\\
 \hline
 Number of Epochs & 200\\
 \hline
 Number of Trials & 10\\
 \hline
\end{tabular}
\vspace{1.0\baselineskip}
\end{table}

We observe the test accuracies of each architecture, averaged over 10 trials, in Table \ref{tab:DenseNetTestAccuracies}. We also view the progression of validation accuracy of the TeLU, GELU, and ReLU architectures, averaged across 10 trials, in Figure \ref{fig:DenseNetValidations}. We notice that TeLU, an activation function that features stronger mitigation of vanishing gradients in the deactivation region, performs significantly better than both the GELU and ReLU architectures. These results are consistent with those of Dubey et al \cite{DUBEY202292}, which highlight that residual connections are integral to the success of architectures employing GELU or ReLU activation functions.

Table \ref{tab:DenseNetTestAccuracies} presents the test accuracies of each architecture, averaged over 10 trials, while Figure \ref{fig:DenseNetValidations} shows the progression of validation accuracy for the TeLU, GELU, and ReLU architectures, also averaged over 10 trials. Notably, TeLU outperforms both GELU and ReLU, highlighting its advantage due to its persistent gradients in the deactivation region, which more effectively mitigate the vanishing gradient problem. This persistent gradient flow allows for more robust learning, especially in architectures lacking residual connections. These results are consistent with the findings of Dubey et al. \cite{DUBEY202292}, which highlight the necessity of residual connections for the success of architectures using GELU or ReLU activation functions. In contrast, TeLU effectively addresses the vanishing gradient and numerical underflow issues related to the dying neuron problem, eliminating the need for residual connections to maintain high performance. As a result, TeLU offers greater architectural flexibility by reducing dependencies on specific model configurations.

\begin{figure}
    \centering
    \includegraphics[width=0.75\linewidth]{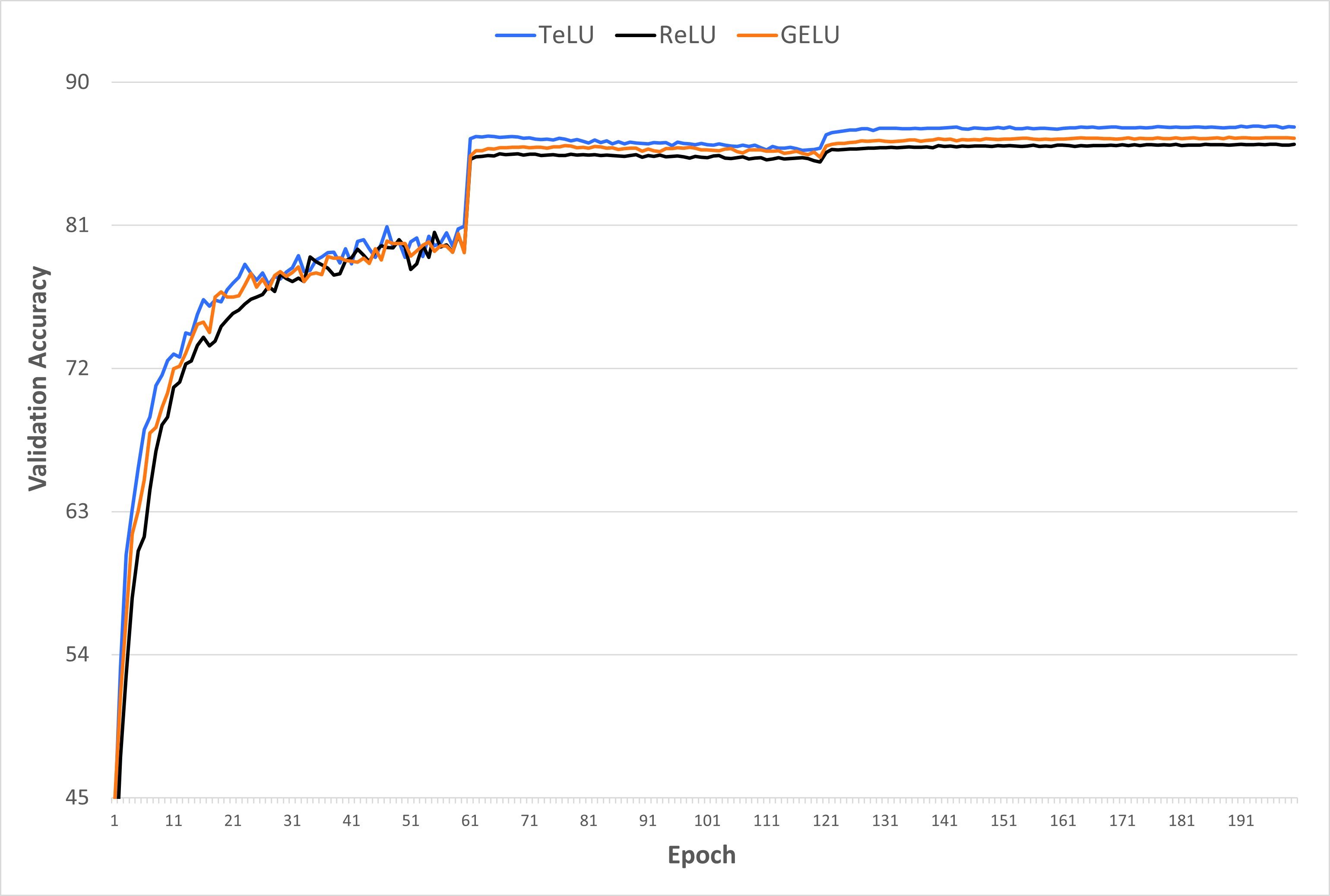}
    \caption[DenseNet CIFAR-10 Validation Accuracy Progression.]
    {\textbf{DenseNet CIFAR-10 Validation Accuracy Progression.} Progression of validation accuracy of DenseNet121 architecture on CIFAR-10 dataset over 200 epochs, averaged over 10 trials per data point.}
    \label{fig:DenseNetValidations}
\end{figure}

\begin{table}[]
\centering
 \caption[DenseNet CIFAR-10 Accuracy Summary.]
 {\textbf{DenseNet CIFAR-10 Accuracy Summary.} Training, validation, and testing accuracies of DenseNet121 architecture on CIFAR-10, averaged over 10 trials.}
    \begin{tabular}{||c c c c||} 
 \hline
 Function & Train & Validation & Testing \\
 \hline\hline
 TeLU & 99.897 & \textbf{87.620} $\pm$ 0.197 & \textbf{86.806} $\pm$ 0.282 \\ 
 \hline
 ReLU & 99.910 & 86.669 $\pm$ 0.176 & 85.825 $\pm$ 0.279  \\
 \hline
 GELU & 99.946 & 86.243 $\pm$ 0.179 & 85.386 $\pm$ 0.200 \\
 \hline
\end{tabular}
\label{tab:DenseNetTestAccuracies}
\vspace{1.0\baselineskip}
\end{table}

\subsection{Near-Linearity of Active Region}
\label{subsect:nearlinexp}

In Subsection \ref{subsect:nearlinsubsection}, we investigated how effectively various linear units approximate a linear function within their active regions. This was done by calculating the integral difference between each activation function and the line representing its slope as $x \to \infty$. Among the smooth linear units, TeLU was found to minimize this integral difference. While both ReLU and ELU are defined as the identity for positive inputs, we argue that TeLU offers superior learning efficiency due to its persistent gradients, which help mitigate the vanishing gradient problem and prevent learning slowdowns.

\subsubsection{ImageNet Dataset with ResNet18 Architecture}

We first test the practical applicability of the theoretical improvement in convergence rate with the ImageNet dataset \cite{ImageNet}. We directly implement Pytorch's example ImageNet training on the ResNet18 architecture \cite{ResNet}. Residual Neural Networks (ResNets) are a type of Convolutional Neural Networks (CNNs) \cite{CNNs} that define bypass connections between hidden convolutional layers to treat effective model depth as a learnable parameter and provide strong gradients to earlier neural layers during backpropagation. We use the predefined base learning rate of 0.1, weight decay coefficient of 0.0001, and minibatch size of 256. We optimize according to Stochastic Gradient Descent (SGD) with a momentum of 0.9 over 90 epochs, dividing our learning rate by 10 at periods of 30 epochs. Due to the computational requirements and the size of the ImageNet dataset, we only ran experiments with the ReLU and TeLU nonlinearities. This configuration is detailed in Table \ref{tab:ImageNetResNet18Hyps}. Our first trial was run with a seed of 1111, but the following two trials were unseeded due to performance drop. We plot the averaged validation accuracy metrics over the trials in Figure \ref{fig:ImageNet90}. Additionally, we show the test accuracy in Table \ref{tab:ImageNetTests}. We observe that TeLU outperforms ReLU with baseline hyperparameters when training for 90 epochs.

\begin{table}[]
    \centering
    \caption[ResNet18 ImageNet Experiment Configuration.]
    {\textbf{ResNet18 ImageNet Experiment Configuration.} Configuration of experiments ran on a ResNet18 architecture training on the ImageNet dataset over 90 epochs. }
    \label{tab:ImageNetResNet18Hyps}
    \begin{tabular}{||c c||} 
 \hline
 Hyperparameter & Value \\
 \hline\hline
 Train/Val Split & 1,281,167/50,000\\ 
 \hline
 Normalization & Standard Score, detailed in Table \ref{tab:NormalizationDetails} \\ 
 \hline
 Architecture & ResNet18 \cite{ResNet}\\ 
 \hline
 Initialization & Xavier Uniform\\ 
 \hline
 Optimizer & Mini-batch SGD\\
 \hline
 Momentum & 0.9\\
 \hline
 Batch Size & 256\\
 \hline
 Learning Scheduler Period & 30\\
 \hline
 Learning Scheduler Decay & 0.1\\
 \hline
 Learning Rate & 0.01\\
 \hline
 Weight Decay & 0.0001\\
 \hline
 Number of Epochs & 90, 50, or 20\\
 \hline
 Number of Trials & 1, 3, or 3 in respect to number of epochs\\
 \hline
\end{tabular}
\vspace{1.0\baselineskip}
\end{table}

\begin{figure}
    \centering
    \includegraphics[width=0.75\linewidth]{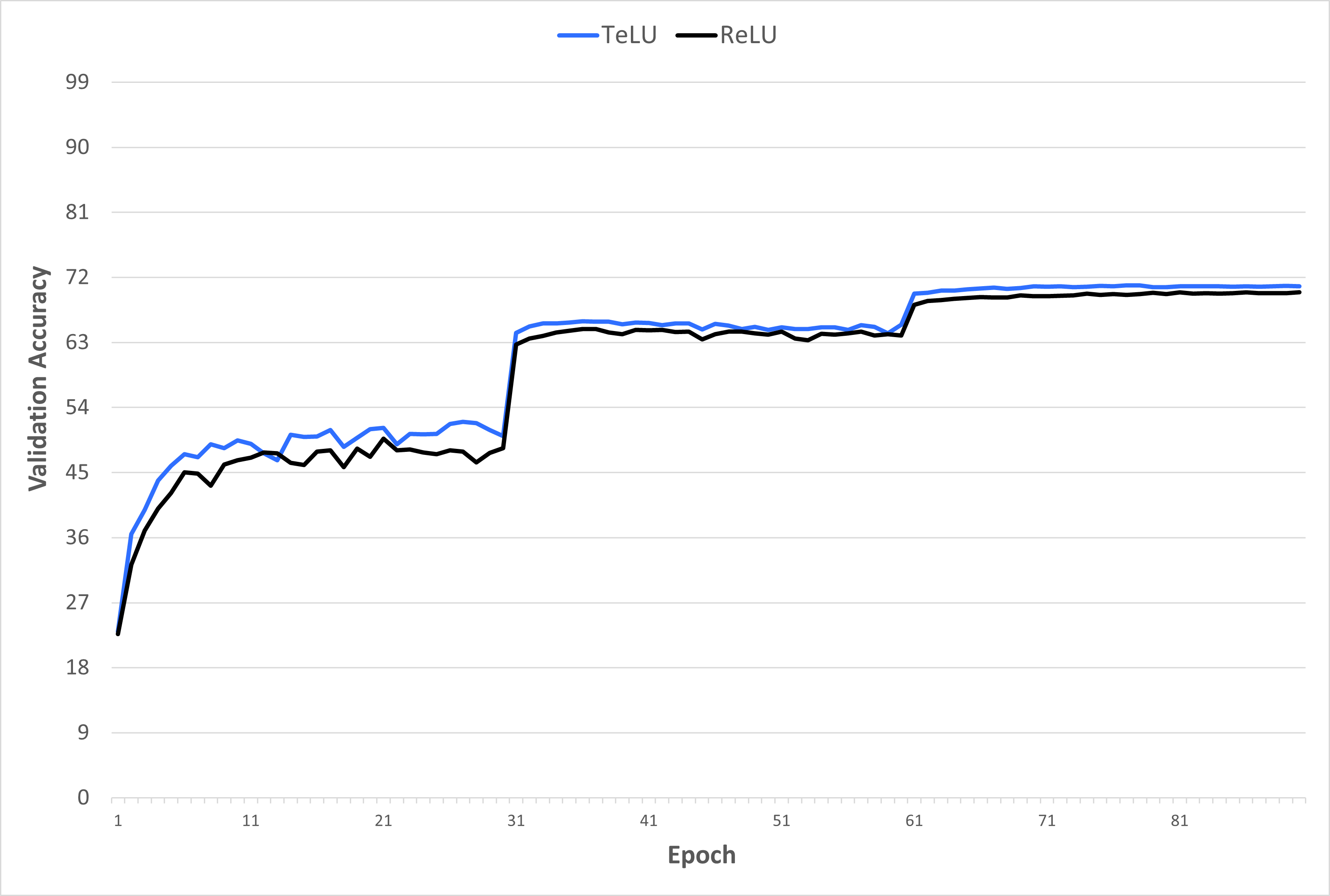}
    \caption[ResNet18 ImageNet Validation Accuracy over 90 Epochs.]
    {\textbf{ResNet18 ImageNet Validation Accuracy over 90 Epochs.} Progression of validation accuracy of ResNet34 architectures employing either TeLU or ReLU nonlinearities over the course of 90 training epochs.}
    \label{fig:ImageNet90}
    \vspace{1.0\baselineskip}
\end{figure}

Given that the unseeded trials still required 14 days to train on 4 GTX 1080 Ti Graphics Processing Units (GPU), we wish to take advantage of TeLU's theoretical improvements in convergence speed. Therefore, we lower the number of training epochs from 90 to 50 and expect to see comparable results. To keep the learning rate scheduler roughly proportional, we lowered the learning rate decay period from 30 to 20. We show the average validation epoch accuracy metrics over the 50 epochs in Figure \ref{fig:ImageNet50Validations}. We also note a comparable performance of both activation function architectures, but a more-unchanged TeLU accuracy in Table \ref{tab:ImageNetTests}. We attribute this to TeLU's distant saturation region helping the network stay relevant across forward passes as well as TeLU's strong gradients for positive inputs.


\begin{figure}
    \centering
    \includegraphics[width=0.75\linewidth]{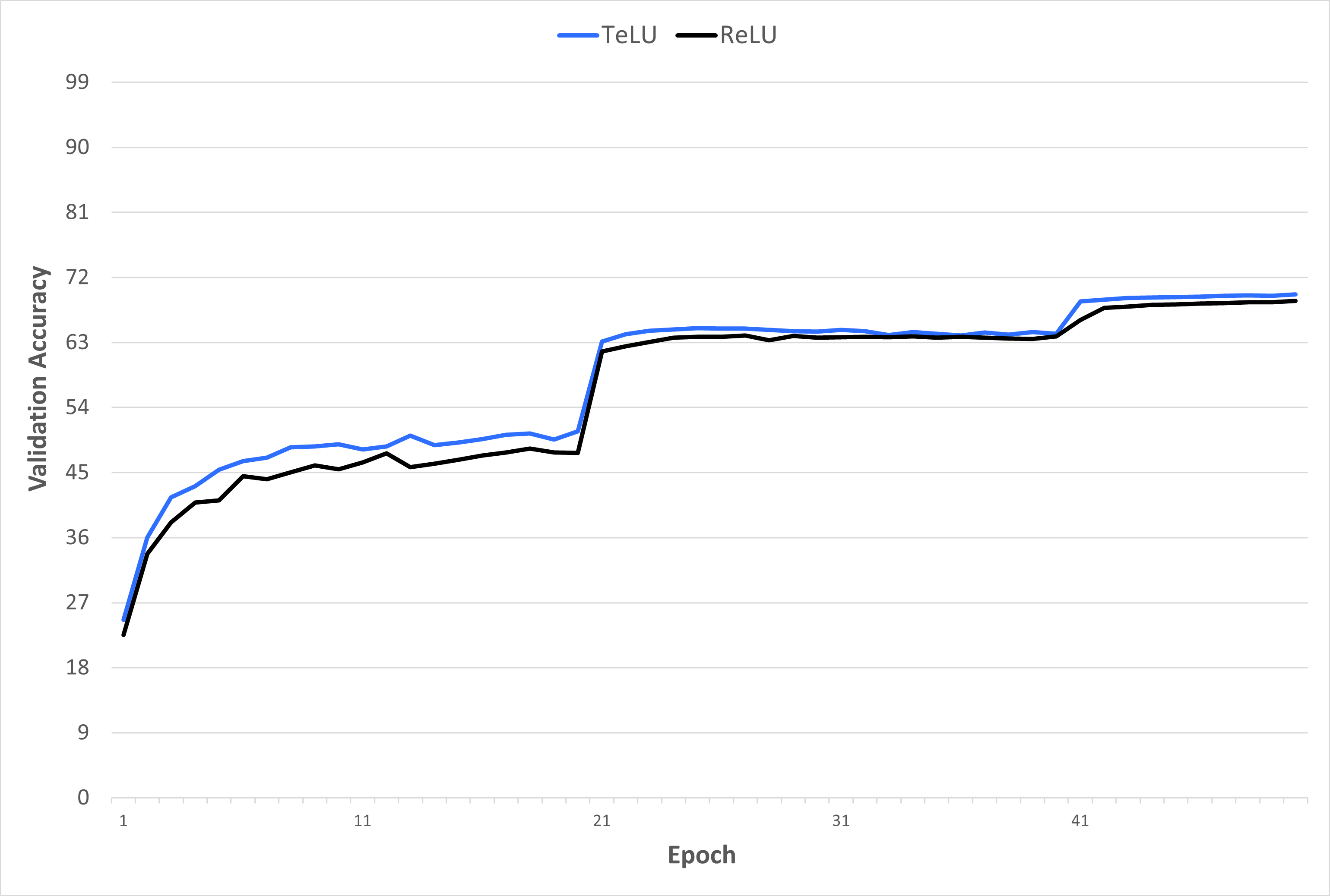}
    \caption[ResNet18 ImageNet Validation Accuracy over 50 Epochs.]
    {\textbf{ResNet18 ImageNet Validation Accuracy over 50 Epochs.} Progression of validation accuracy of ResNet34 architectures employing either TeLU or ReLU nonlinearities over the course of 50 training epochs. Each data point is averaged over 3 trials.}
    \label{fig:ImageNet50Validations}
\end{figure}

\begin{figure}
    \centering
    \includegraphics[width=0.75\linewidth]{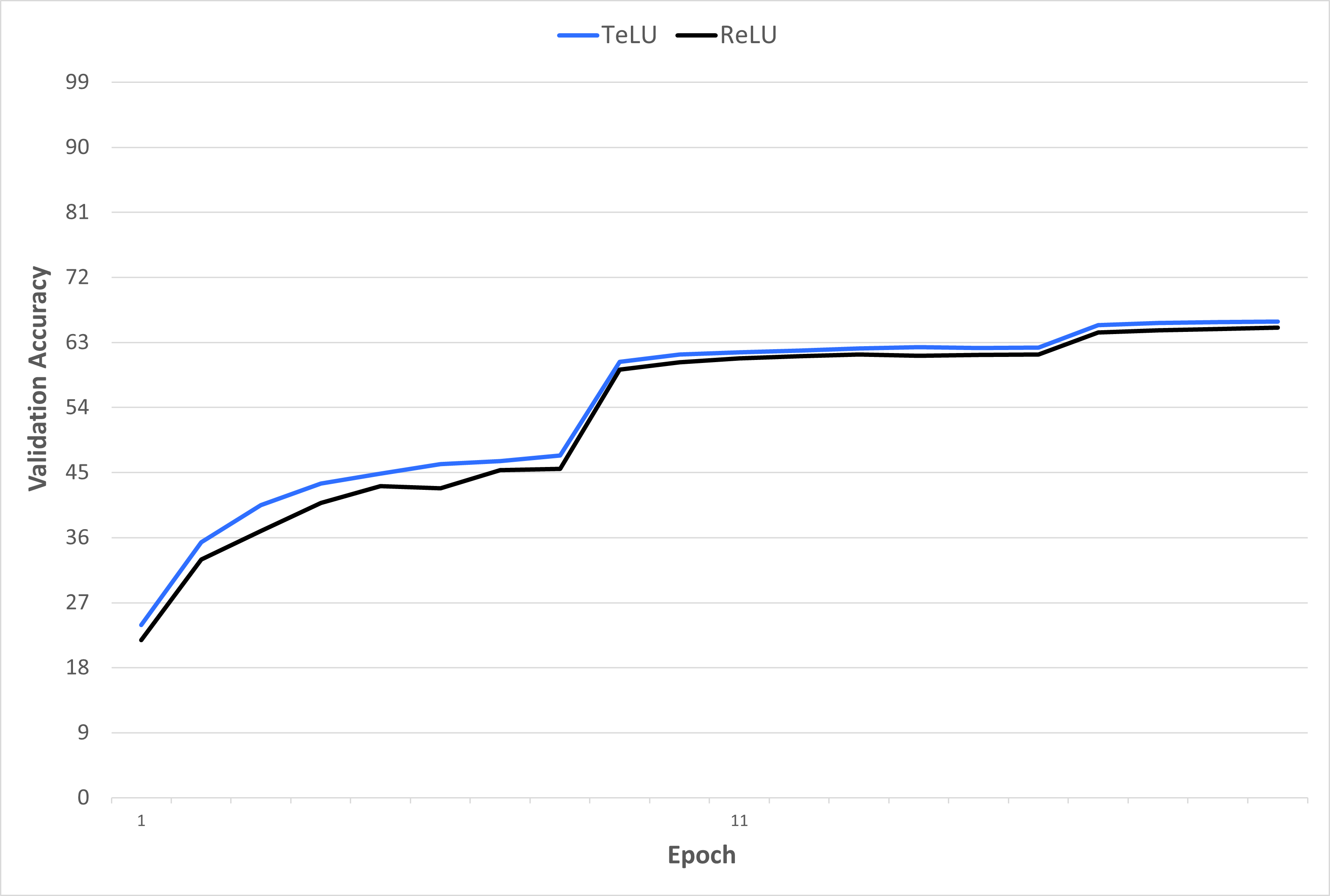}
    \caption[ResNet18 ImageNet Validation Accuracy over 20 Epochs.]
    {\textbf{ResNet18 ImageNet Validation Accuracy over 20 Epochs.} Progression of validation accuracy of ResNet34 architectures employing either TeLU or ReLU nonlinearities over the course of 20 training epochs. Each data point is averaged over 3 trials.}
    \label{fig:ImageNet20Validations}
\end{figure}

\begin{table}[]
    \centering
        \caption[ResNet18 ImageNet Test Accuracy Summary.]
        {\textbf{ResNet18 ImageNet Test Accuracy Summary.} Top-1 and top-5 test accuracies for ResNet34 architectures employing either TeLU or ReLU nonlinearities. Experiment was repeated for 90, 50, and 20 epochs to demonstrate the convergance speed advantages of TeLU. Each testing accuracy is averaged over 3 trials for the 50 epoch and 20 epoch experiments.}
    \label{tab:ImageNetTests}
    \begin{tabular}{||c c c||} 
 \hline
 Name & Top-1 Accuracy \% & Top-5 Accuracy \%\\
 \hline\hline
 TeLU 90 Epoch & \textbf{70.86} & \textbf{89.70} \\ 
 \hline
 ReLU 90 Epoch & 69.96 & 89.27 \\ 
 \hline
 TeLU 50 Epoch & \textbf{69.76}$\pm$0.095 & \textbf{89.16}$\pm$0.063 \\ 
 \hline
 ReLU 50 Epoch & 68.736$\pm$0.092 & 88.47$\pm$0.006 \\ 
 \hline
 TeLU 20 Epoch & \textbf{65.87}$\pm$0.044 & \textbf{86.67}$\pm$0.107 \\ 
 \hline
 ReLU 20 Epoch & 65.08$\pm$0.041 & 86.24$\pm$0.082 \\ 
 \hline
\end{tabular}
\vspace{1.0\baselineskip}
\end{table}

\subsubsection{Text8 Dataset with Dynamic Pooling Transformer}

We then proceeded to test TeLU's baseline convergence rate on a transformer \cite{Transformer} architecture. For performance improvements, we utilized a dynamic pooling transformer architecture \cite{DynamicPoolingTransformers} with 8 heads of attention. We set the fully connected and attention neurons to drop out at a rate of 0.12 during training. We utilized a base learning rate of 0.0002 that updated according to a 200,000 step cosine scheduler with 4,000 initial warm-up steps. We used the ADAM optimizer \cite{Adam}, given that we did not employ weight decay. We detail our configuration in Table \ref{tab:Transformerhyperparams}. With this configuration, we ran 3 trials and recorded the averaged progression of validation error metrics over the 20 evaluation steps in Figure \ref{fig:transformers_full}.

\begin{table}[]
    \centering
    \caption[ Dynamic-Pooling Transformer Text8 Experiment Configuration.]
    {\textbf{Dynamic-Pooling Transformer Text8 Experiment Configuration.} Configuration of experiments ran on a dynamic-pooling transfromer architecture on the Text8 dataset. }
    \label{tab:Transformerhyperparams}
    \begin{tabular}{||c c||} 
 \hline
 Hyperparameter & Value \\
 \hline\hline
 Train/Val/Test Split & $90 \cdot 10^6$/$5 \cdot 10^6$/$5 \cdot 10^6$\\ 
 \hline
 Architecture & Dynamic-Pooling Transformer \cite{DynamicPoolingTransformers}\\ 
 \hline
 Model Dimension & 512\\
 \hline
 Number of Heads & 8\\
 \hline
 Head Dimension & 64\\
 \hline
 dropout & 0.06\\
 \hline
 computational precision & Floating Point 16\\
 \hline
 gradient clipping & 0.25\\
 \hline
 Initialization & Gaussian with 0.02 average and 0.01 stdv\\ 
 \hline
 Optimizer & Mini-batch Adam\\
 \hline
 Batch Size & 16\\
 \hline
 Base Learning Rate & 0.0002\\
 \hline
 Learning Scheduler & Cosine Scheduler \cite{warmcosineScheduler}\\
 \hline
 Warm-Up Steps & 4000\\
 \hline
 Learning Scheduler Period & 150,000 or 200,000 \\
 \hline
 Mini-Batch Steps & 100,000 or 200,000\\ 
 \hline
 Number of Trials & 3 or 9, in respect to mini-batch steps \\
 \hline
 Random Seed & [0000, 1111, ..., 9999]\\
 \hline
\end{tabular}
\end{table}

\begin{figure}
    \centering
    \includegraphics[width=0.75\linewidth]{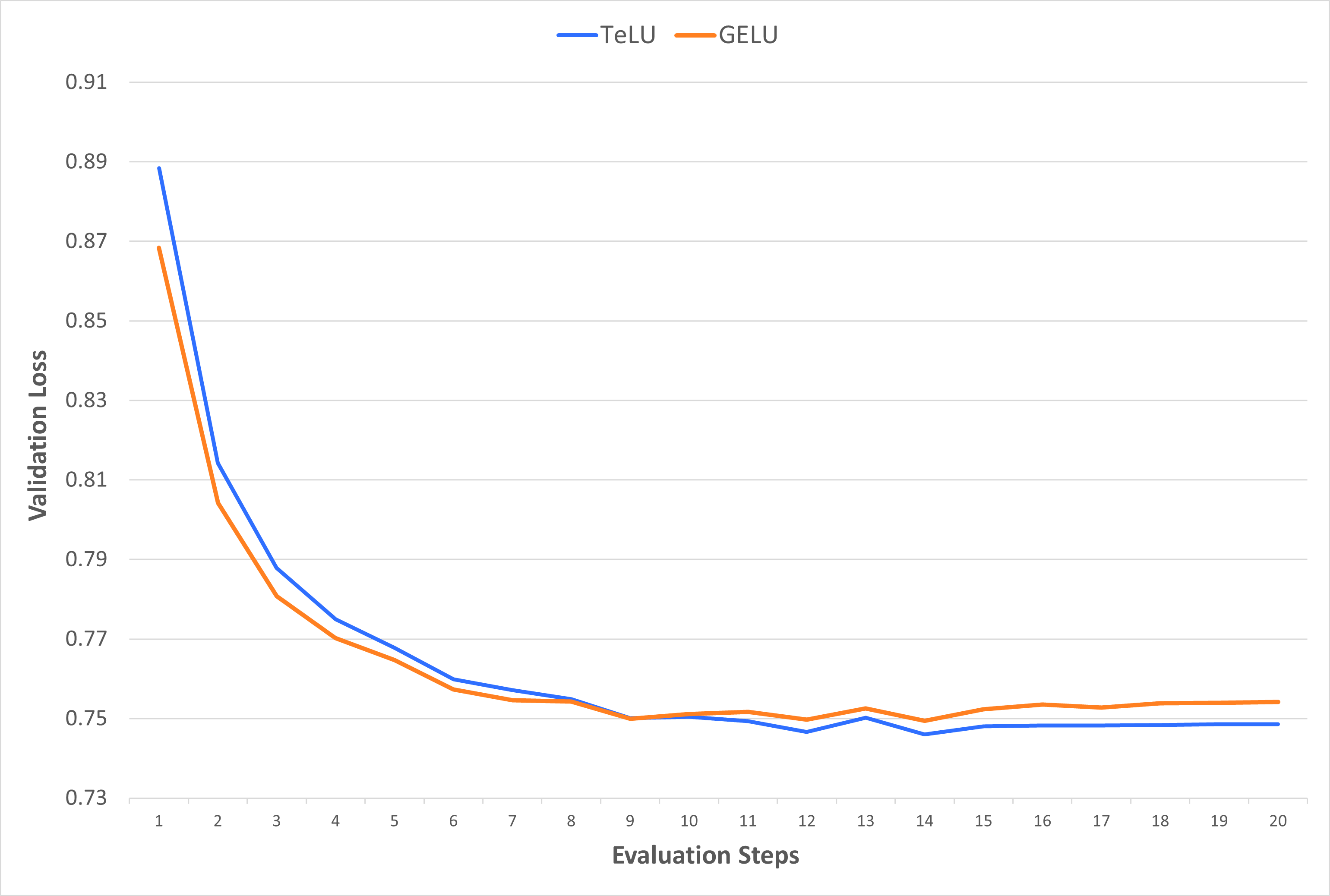}
    \caption[Dynamic-Pooling Transformer Text8 Validation Loss.]
    {\textbf{Dynamic-Pooling Transformer Text8 Validation Loss.} Shows the progression of validation Loss throughout 200,000 mini-batches of the Dynamic-Pooling Transformer architecture being trained on the Text8 dataset with the Mini-batch Adam optimizer. Each data point has been averaged over 3 trials.}
    \label{fig:transformers_full}
    \vspace{1.0\baselineskip}
\end{figure}

As shown in Figure \ref{fig:transformers_full}, GELU initially achieves lower validation loss during the early evaluation steps. However, TeLU’s faster convergence allows its loss to drop below GELU’s over time. Notably, TeLU's validation loss stabilizes by evaluation step 15, whereas GELU begins to experience increasing validation error beyond step 14. These validation improvements with TeLU translate into superior test loss, as seen in Table \ref{tab:transformer}. Furthermore, TeLU's stability at favorable validation accuracy results in lower standard deviation in test loss compared to GELU.

To demonstrate TeLU's improved convergence rate further, we reduced the training steps from 200,000 to 100,000 and updated the cosine scheduler to a period of 150,000. This was empirically chosen to allow the architectures to continue applying meaningful updates during the last few training steps. With a shorter experiment, our budget now allows for a ReLU architecture to be included. We view the progression of the validation error in Figure \ref{fig:transformers_half} and the test error summary in Table \ref{tab:transformer}, with data averaged over 9 trials.

\begin{table}[]
    \centering
       \caption[Dynamic-Pooling Transformer on Text8 Test Loss.]
        {\textbf{Dynamic-Pooling Transformer on Text8 Test Loss.} Shows the progression of validation Loss throughout 200,000 mini-batches of the Dynamic-Pooling Transformer architecture being trained on the Text8 dataset with the Mini-batch Adam optimizer. Each data point has been averaged over 9 trials.}
    \label{tab:transformer}
    \begin{tabular}{||c c c||} 
 \hline
 Name & 100 Epoch & 200 Epoch\\
 \hline\hline
 TeLU & \textbf{0.7985}$\pm$0.00089 & \textbf{0.7970}$\pm$0.00078 \\ 
 \hline
 ReLU & 0.7998$\pm$0.00055 & - \\
 \hline
 GELU & 0.7999$\pm$0.0027 & 0.8034$\pm$0023 \\ 
 \hline
\end{tabular}
\vspace{1.0\baselineskip}
\end{table}

\begin{figure}
    \centering
    \includegraphics[width=0.75\linewidth]{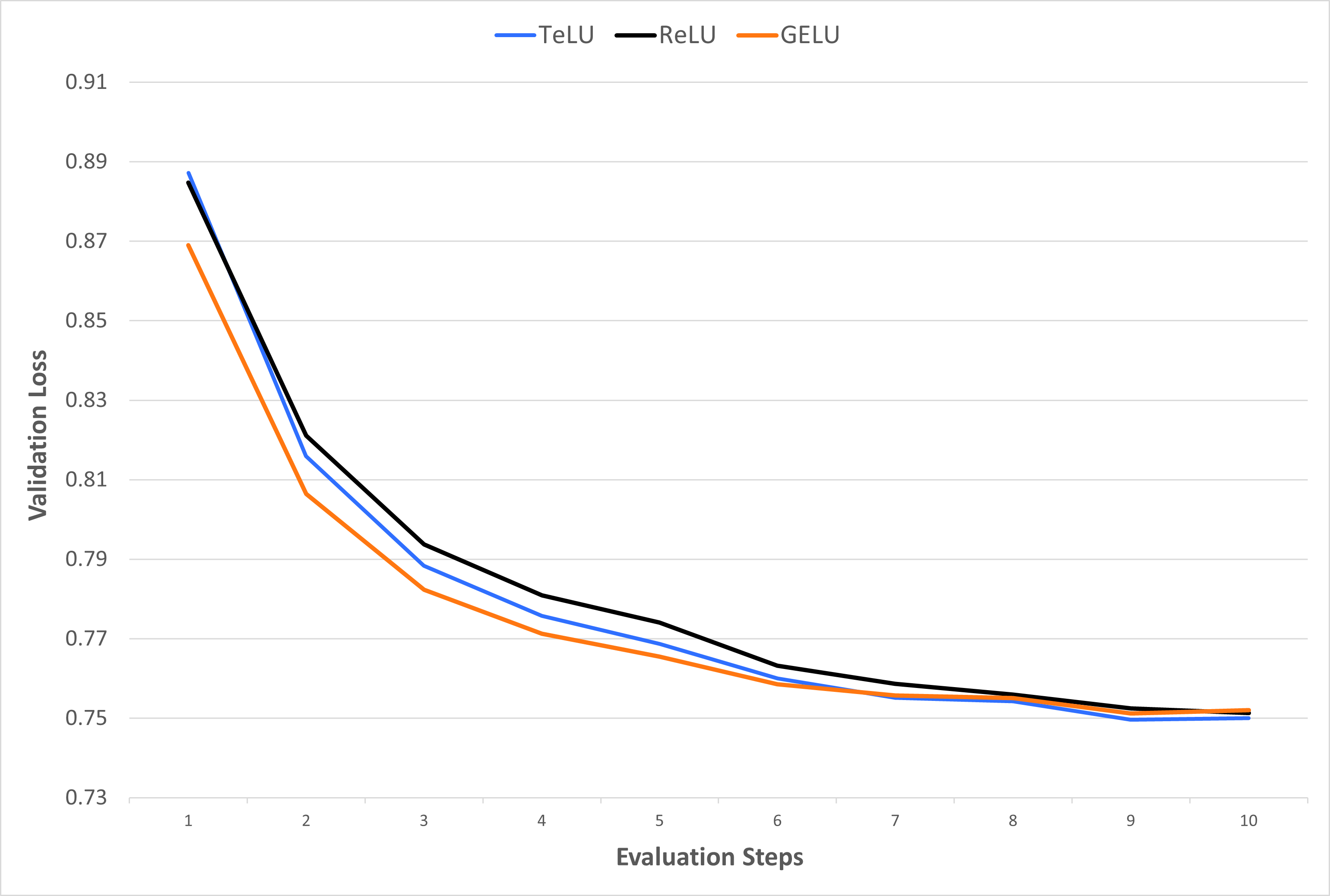}
    \caption[Dynamic-Pooling Transformer on Text8 Validation Loss Short Experiment.]
    {\textbf{Dynamic-Pooling Transformer on Text8 Validation Loss Short Experiment.} Shows the progression of validation Loss throughout 100,000 mini-batches of the Dynamic-Pooling Transformer architecture being trained on the Text8 dataset with the Mini-batch Adam optimizer. Each data point has been averaged over 9 trials.}
    \label{fig:transformers_half}
    \vspace{1.0\baselineskip}
\end{figure}

The shortening of the cosine scheduler period \cite{warmcosineScheduler} enables TeLU to outperform GELU at earlier evaluation steps, as anticipated. ReLU, like GELU, displays strong gradients in the active region but converges more slowly. This slower convergence can likely be attributed to the dying neuron problem, which inhibits learning in ReLU architectures. As a result, both ReLU’s dying neuron issue and GELU’s sub-linearity in the active region contribute to the slower and less consistent convergence observed.

\subsection{Runtime Efficiency}
\label{subsect:runtimeefficiencyexp}

Using the function definitions provided in Table \ref{table:baselineFormulations}, we measured the computational efficiency of forward and backward passes for a neuron employing the baseline implementation of each nonlinearity. To ensure accurate delay measurements, we first minimized operating system scheduler disturbances \cite{OS1} \cite{OS2} by running only essential Windows operating system programs on the system in addition to the experiment itself. Our host system had 12 Intel i7 processors, 32 GB of RAM, and an NVIDIA RTX-2070 Graphics Processing Unit (GPU) \cite{GPU}. We defined an input vector of $10^6$ randomly initialized 32-bit floating point values and performed a sequence of forward and backward passes over $10^6$ iterations. This approach helped evenly distribute any remaining operating system interrupts across each nonlinearity's measurement. We repeated the experiment after its completion to ensure that the sequential execution of each function's forward and backward passes did not provide any bias to our experiment. This repeat experiment confirmed that system performance factor differences between the times of each nonlinearity's calculations were negligible. The delay of each activation function is recorded in Table \ref{tab:RunTimesLaptop}, where we observed that TeLU's performance was second only to ReLU's.

\begin{table}[]
\centering
 \caption[Timing of $10^6$ Iterations of Single Neuron with Input Size of $10^6$ on RTX2070.]
 {\textbf{Timing of $10^6$ Iterations of Single Neuron with Input Size of $10^6$ on RTX2070.} Delay in seconds of 1 million forward and backward passes of a single-neuron architectures with each of the focused non-linearities. Experiment executed in system utilizing NVIDIA Geforce 2070. Input Size totaled over 1 million iterations.}
    \begin{tabular}{||c c c||} 
 \hline
 Function & Forward & Backward  \\
 \hline\hline
 TeLU & 139.56s & 322.35s  \\ 
 \hline
 ReLU & 107.75s & 299.77s   \\
 \hline
 ELU & 167.64s & 326.48s \\
 \hline
 SiLU & 187.80s & 535.00s \\
 \hline
 GELU & 210.27s & 457.67s  \\
 \hline
 Mish & 205.23s & 380.03s \\
 \hline
 Logish & 244.66s & 574.30s \\
 \hline
 Smish & 285.02s & 621.04s \\
 \hline
\end{tabular}
\label{tab:RunTimesLaptop}
\vspace{1.0\baselineskip}
\end{table}

We repeated the baseline definition experiment on two Linux systems, each equipped with 64 GB of RAM. One system featured an A100 GPU, while the other used a TI-1080 GPU. This time, we used input vector of length $10^{6}$ and $10^{7}$ to ensure that each GPU would require multiple operations to process the larger 32-bit floating point input vector. The timing comparisons between these experiments was proportional, as seen in Tables \ref{tab:RunTimesGPU} and \ref{tab:RunTimesGPU10Minput} indicating minimal variance in nonlinearity delays between the two systems. The concise formulation of TeLU allows for significant computational efficiency improvements over other smooth functions, making TeLU-based architectures more scalable and leading to reduced training times. This speedup in training suggests not only an improvement in TeLU's convergence speed but also enhanced energy efficiency during operation.

\begin{table}[]
\centering
\caption[Timing of $10^6$ Iterations of Single Neuron with Input Size of $10^6$ on Server GPUs.]
{\textbf{Timing of $10^6$ Iterations of Single Neuron with Input Size of $10^6$ on Server GPUs.} Delay in seconds of 1 million forward and backward passes of a single-neuron architectures with each of the focused non-linearities. Experiment repeated for systems utilizing A100 and 1080Ti GPUs. Input Size totaled over 1 million iterations.}
    \begin{tabular}{||c | c c | c c ||} 
 \hline
          & \multicolumn{2}{|c|}{A100 GPU} & \multicolumn{2}{|c|}{1080Ti GPU} \\
\hline
 Function & Forward & Backward & Forward & Backward \\
 \hline\hline
 TeLU & 77.337s & 150.98s & 119.35s & 279.83s \\ 
 \hline
 ReLU & 65.356s & 148.46s & 96.525s & 248.87s \\
 \hline
 ELU & 103.00s & 171.44s & 156.60s & 281.62s\\
 \hline
 SiLU & 97.036s & 195.08s & 162.31s & 444.18s\\
 \hline
 GELU & 114.06s & 209.38s & 174.27s & 387.38s \\
 \hline
 Mish & 108.41s & 170.07s & 173.37s & 315.67s\\
 \hline
 Logish & 120.66s & 211.11s & 208.95s & 479.72s\\
 \hline
 Smish & 137.34s & 222.90s & 233.01s & 515.59s\\
 \hline
\end{tabular}
\label{tab:RunTimesGPU}
\end{table}

\begin{table}[]
\centering
 \caption[Timing of $10^6$ Iterations of Single Neuron with Input Size of $10^7$.]
 {\textbf{Timing of $10^6$ Iterations of Single Neuron with Input Size of $10^7$.} Delay in seconds of 10 million forward and backward passes of a single-neuron architectures with each of the focused non-linearities. Experiment repeated for systems utilizing A100 and 1080Ti GPUs. Input Size totaled over 1 million iterations.}
    \begin{tabular}{||c | c c | c c ||} 
 \hline
          & \multicolumn{2}{|c|}{A100 GPU} & \multicolumn{2}{|c|}{1080Ti GPU} \\
\hline
 Function & Forward & Backward & Forward & Backward \\
 \hline\hline
 TeLU & 206.54s & 517.10s & 771.53s & 2032.20s \\ 
 \hline
 ReLU & 127.60s & 453.21s & 442.32s & 1719.73s \\
 \hline
 ELU & 251.17s & 505.88s & 926.03s & 1943.89s\\
 \hline
 SiLU & 282.00s & 872.09s & 1074.99s & 3464.10s\\
 \hline
 GELU & 313.81s & 757.55s & 1188.41s & 2953.31s \\
 \hline
 Mish & 311.79s & 593.34s & 1181.49s & 2324.54s\\
 \hline
 Logish & 386.72s & 947.87s & 1491.52s & 3775.24s\\
 \hline
 Smish & 439.14s &  1023.43s & 1700.14s & 4086.84s\\
 \hline
\end{tabular}
\label{tab:RunTimesGPU10Minput}
\vspace{1.0\baselineskip}
\end{table}

\begin{table}[]
\centering
\caption[Timing of $10^7$ Iterations of Single Neuron with Input Size of $10^6$.]
{\textbf{Timing of $10^7$ Iterations of Single Neuron with Input Size of $10^6$.} Delay in seconds of 1 million forward and backward passes of a single-neuron architectures with each of the focused non-linearities. Experiment repeated for systems utilizing A100 and 1080Ti GPUs. Input Size totaled over 10 million iterations.}
    \begin{tabular}{||c | c c | c c ||} 
 \hline
          & \multicolumn{2}{|c|}{A100 GPU} & \multicolumn{2}{|c|}{1080Ti GPU} \\
\hline
 Function & Forward & Backward & Forward & Backward \\
 \hline\hline
 TeLU & 772.77s & 1511.7s & 1176.7s & 2786.8s \\ 
 \hline
 ReLU & 661.14s & 1489.03s & 966.6s & 2509.2s \\
 \hline
 ELU & 1041.3s & 1706.5s & 1556.9s & 2820.8s\\
 \hline
 SiLU & 973.83s & 1952.3s & 1707.4s & 4547.6s\\
 \hline
 GELU & 1124.1s & 2102.2s & 1707.5s & 3864.6s \\
 \hline
 Mish & 1084.0s & 1694.9s & 1865.3s & 3287.5s\\
 \hline
 Logish & 1230.7s & 2099.5s & 2080.2s & 4802.1s\\
 \hline
 Smish & 1346.7s & 2242.4s & 2311.0s & 5148.4s\\
 \hline
\end{tabular}
\label{tab:RunTimesGPU10Miter}
\vspace{1.0\baselineskip}
\end{table}

This experiment was repeated for $10^7$ iterations of input vectors of size $10^6$, with results shown in Table \ref{tab:RunTimesGPU10Minput}. Across all cases, TeLU demonstrates substantial performance improvements over all nonlinearities except ReLU. Although ReLU is faster to compute, TeLU compensates with faster convergence, achieving high performance with fewer training iterations. These findings are particularly meaningful, as they highlight TeLU's superior computational efficiency compared to other smooth activation functions. This efficiency plays a critical role in deep neural networks because activation functions heavily influence the overall training speed. Faster training times not only enhance productivity but also enable the use of larger models, increased data throughput, and more frequent experimentation. Moreover, computationally efficient activation functions can reduce hardware demands, lower energy consumption, and facilitate real-time applications, making them vital for scalable and sustainable AI solutions.

\subsection{ReLU Compatibility}
\label{subsect:relucompexp}

Optimizing hyperparameter selection for training deep learning models is a time-intensive process. The time required can span weeks or even months, depending on the number of hyperparameters involved, the complexity of the architecture, and the duration of the training. The challenge arises from the fact that hyperparameter values can impact model accuracy, overfitting, convergence speed, and stability in nonlinear and interdependent ways. Furthermore, when designing experiments from scratch, each new setup demands its own lengthy tuning process, as we encountered in our initial experiments.

To enhance our productivity and streamline the process of finding competitive configurations, we began utilizing hyperparameter settings sourced from GitHub repositories or published papers. These configurations were primarily designed for ReLU-based architectures. Fortunately, we discovered that TeLU remained competitive without requiring any modifications to these ReLU configurations. In practice, adapting ReLU architectures to TeLU architectures proved to be straightforward—simply by defining the TeLU nonlinearity and substituting it for the ReLU nonlinearity in the hidden layers. This observed compatibility suggests that ReLU and TeLU perform well within similar hyperparameter subspaces, which could facilitate the adoption of TeLU. Given that ReLU is arguably the most popular activation function for hidden neurons, it is likely that data scientists will initially test a novel nonlinearity within an existing ReLU architecture using ReLU configurations. Therefore, TeLU's similarity to ReLU across different architectures could help it make a positive first impression.

\subsubsection{FashionMNIST Dataset with MLP with Varied Weight Decay}

To experimentally demonstrate this compatibility, we designed MLP architectures utilizing various activation functions, including ReLU, ELU, GELU, SiLU, Mish, Logish, Smish, and TeLU. Each architecture was defined with 8 hidden layers, each with a width of 128 neurons. The batch size was set to 128, and the learning rate was fixed at 0.005. We used the Fashion MNIST dataset, dividing it into training, validation, and testing partitions with 50,000; 10,000; and 10,000 samples, respectively. Each configuration was tested over 6 trials across 6 different weight decay values: 0.0001, 0.0003, 0.0005, 0.001, 0.003, and 0.005. We chose to focus on weight decay because these activation functions seemed to differ most in their degree of self-regularization. The complete experimental configuration can be viewed in Table \ref{tab:ReLUSimilarityMLPHyperparams}. The resulting accuracies for each configuration are presented in Figure \ref{fig:ReLULikenessOverall}.

\begin{table}[]
    \centering
    \caption[MLP FashionMNIST Across Weight Decays Configuration.]
    {\textbf{MLP FashionMNIST Across Weight Decays Configuration.} Configuration of experiments ran on an MLP model on the FashionMNIST dataset across various weight decay coefficient values. }
    \label{tab:ReLUSimilarityMLPHyperparams}
    \begin{tabular}{||c c||} 
 \hline
 Hyperparameter & Value \\
 \hline\hline
 Train/Val/Test Split & 50,000 / 10,000 / 10,000\\ 
 \hline
 Hidden Layer Count & 8\\ 
 \hline
 Hidden Layer Width & 128\\ 
 \hline
 Initialization & Xavier Uniform\\ 
 \hline
 Optimizer & Mini-batch SGD\\
 \hline
 Momentum & 0.9\\
 \hline
 Batch Size & 128\\
 \hline
 Learning Rate & 0.005\\
 \hline
 Weight Decay & [0.0001, 0.0003, 0.0005, 0.001, 0.003, 0.005]\\
 \hline
 Number of Epochs & 100\\ 
 \hline
 Number of Trials & 6 trials per weight decay\\ 
 \hline
\end{tabular}
\end{table}

\begin{figure}
    \centering
    \includegraphics[width=0.75\linewidth]{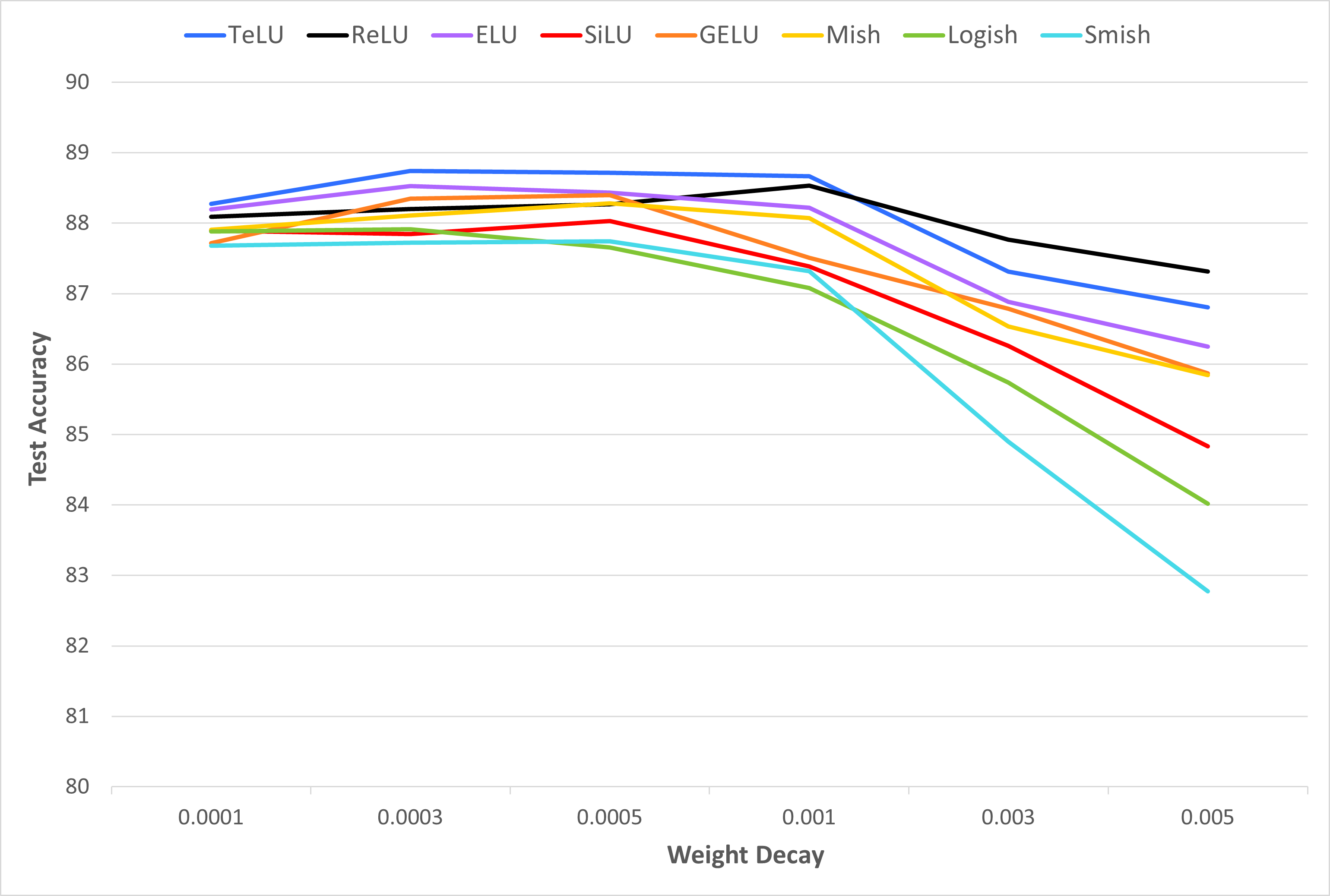}
    \caption[MLP FashionMNIST Test Accuracy Across Wide Range of Weight Decays.]
    {\textbf{MLP FashionMNIST Test Accuracy Across Wide Range of Weight Decays.} Visualizing the testing accuracy of different MLP architectures when the weight decay coefficient of their loss function is treated as an independent variable. Wider weight decay set of {0.0001, 0.0003, 0.0005, 0.001, 0.003, 0.005} shows broad spectrume of the testing accuracy of different MLP architectures.}
    \label{fig:ReLULikenessOverall}
\end{figure}

It is observed that TeLU exhibits superior accuracy when using a weight decay coefficienct that is optimal for ReLU. ReLU and TeLU nonlinearities generally benefit from higher levels of L2 regularization, while Logish and Smish require minimal or no weight decay to maintain stability. ReLU likely benefits more from increased weight decay due to its tendency to cause neurons to become permanently inactive during training. A larger weight decay helps mitigate the dying ReLU problem by keeping preactivations close to zero, which allows neurons to stay near the threshold needed to activate or deactivate based on different input signals. In contrast, nonlinearities with complex nested subfunctions, such as Smish, experience instability at weight decay levels that are beneficial for ReLU. When external regularization pushes the activation of preceding Smish neurons towards zero, subsequent layers receive inputs close to zero. This can lead to numerical underflow as the nested nonlinearities are calculated from the inside out, resulting in dying neurons that cause subsequent layers to produce zero activation, further propagating numerical instability. Overall, we observe a strong correlation between the mathematical complexity of a nonlinearity and its ideal level of L2 regularization, with the exception of SiLU and ELU.

We repeat this experiment for weight decay coefficient values 0.0004, 0.0006, 0.0008, 0.001, 0.0012, 0.0014, and 0.0016. As shown in Figure \ref{fig:ReLULikenessFocus}, the same key observation persists: where ReLU performs optimally, TeLU outperforms it. This is significant because ReLU is widely adopted in deep neural networks, meaning many existing projects rely on its configuration. With TeLU acting as a straightforward substitute for ReLU, researchers and developers can seamlessly switch to TeLU and immediately benefit from faster convergence and enhanced stability without requiring additional tuning, leading to greater productivity and minimal adjustment effort.

\begin{figure}
    \centering
    \includegraphics[width=0.75\linewidth]{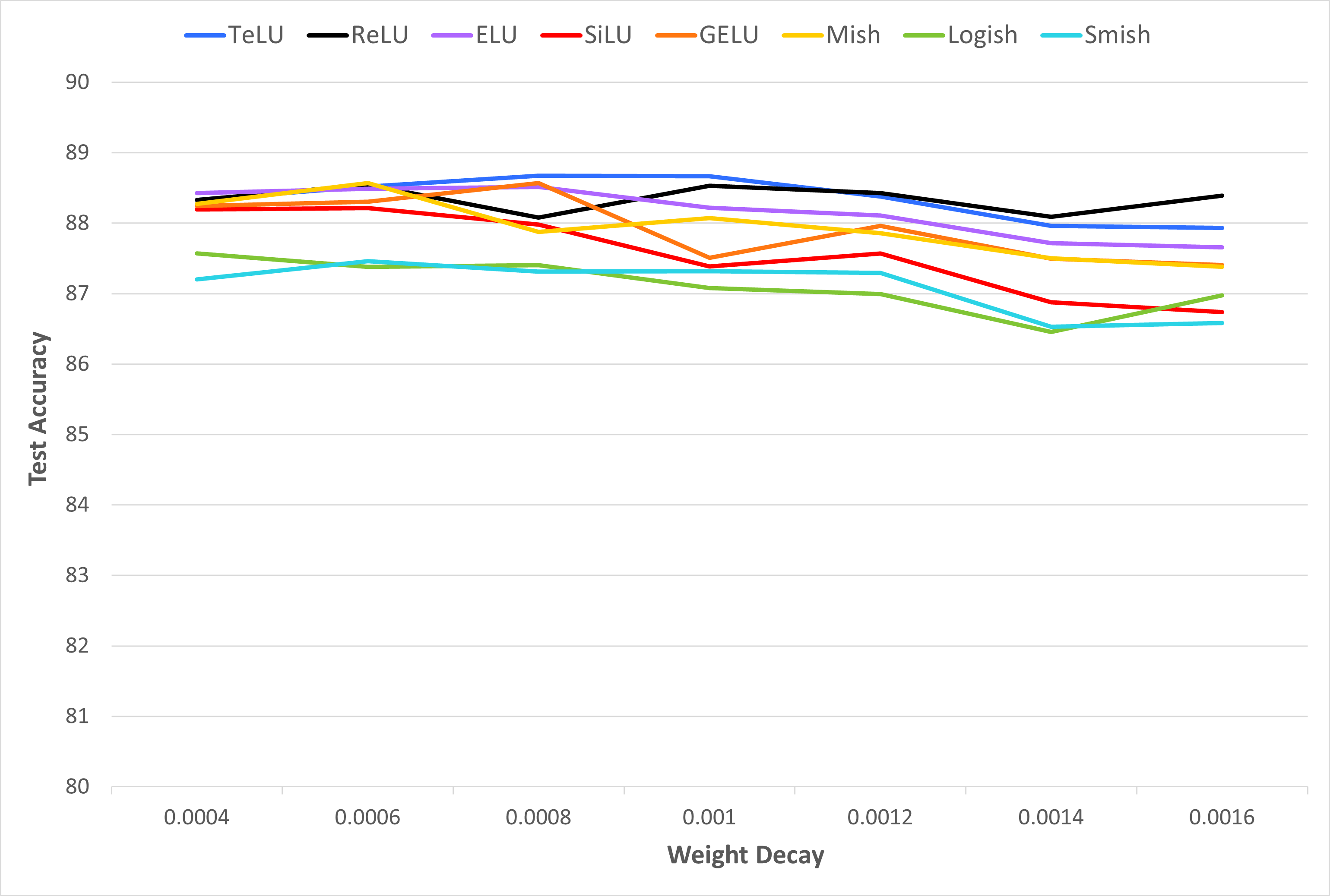}
    \caption[MLP FashionMNIST Test Accuracy Across Focused Range of Weight Decays.]
    {\textbf{MLP FashionMNIST Test Accuracy Across Focused Range of Weight Decays.} Visualizing the testing accuracy of different MLP architectures when the weight decay coefficient of their loss function is treated as an independent variable. Narrower weight decay set of {0.0004, 0.0006, 0.0008, 0.001, 0.0012, 0.0014, 0.0016} meant to focus on configurations that optimize ReLU and TeLU hidden non-linearities. Each data point shown is averaged over 6 trials.}
    \label{fig:ReLULikenessFocus}
    \vspace{1.0\baselineskip}
\end{figure}

\subsubsection{CIFAR-10 Dataset with SqueezeNext with Varied Optimizers}

To further validate the compatibility between ReLU and TeLU configurations, we evaluated the performance of various nonlinearities when trained with hyperparameters optimized for ReLU. We utilized a SqueezeNext CNN architecture with ReLU as the activation function in its hidden layers. We then fine-tuned the learning rate, weight decay, and learning rate scheduler scaling hyperparameters for each of four optimization methods. We utilized Minibatch Stochastic Gradient Descent (SGD) \cite{SGD1, SGD2}, SGD with momentum \cite{Momentum}, AdamW \cite{AdamW}, and RMSprop \cite{RMSprop}. The resulting learning rate and weight decay hyperparameters for each optimizer are detailed in \ref{tab:CIFAR_sup_sq_hps}. The portion of the hyperparameters that stay constant throughout optimizer choice is detailed in Table \ref{tab:SqueezeNextReLULikenessHyps}. The learning rate scheduler for each optimizer scaled the learning rate according to the learning decay value once every 60 epochs. For the experiment, we set aside 10,000 of the 60,000 CIFAR-10 images to define a validation partition, trained on the remaining 50,000 images, and tested on the standard 10,000-image testing split. We then substituted ReLU with the other nonlinearities and summarized the testing accuracies for each architecture over 10 trials in Table \ref{tab:CIFARSUM}.

\begin{table}[]
    \centering
    \caption[SqueezeNext CIFAR-10 ReLU-Focused Configuration.]
    {\textbf{SqueezeNext CIFAR-10 ReLU-Focused Configuration.} Static Configuration of experiments ran on the SqueezeNext architecture on the CIFAR-10 dataset across various optimization algorithms. }
    \label{tab:SqueezeNextReLULikenessHyps}
    \begin{tabular}{||c c||} 
 \hline
 Hyperparameter & Value \\
 \hline\hline
 Train/Val/Test Split & 45,000 / 5,000 / 10,000\\ 
 \hline
 Normalization & Standard Score, detailed in Table \ref{tab:NormalizationDetails} \\ 
 \hline
 Architecture & SqueezeNext \cite{SqueezeNext}\\ 
 \hline
 Initialization & Xavier Uniform\\ 
 \hline
 Batch Size & 128\\
 \hline
 Learning Scheduler Period & 60\\
 \hline
 Data Augmentations & Random 4-padded cropping and horizontal flipping\\
 \hline
 Number of Epochs & 200\\ 
 \hline
 Number of Trials & 5 per optimizer\\ 
 \hline
\end{tabular}
\vspace{1.0\baselineskip}
\end{table}

\begin{table}[]
    \centering
    \caption[SqueezeNext CIFAR-10 Optimizer-Dependent Hyperparameters.]
    {\textbf{SqueezeNext CIFAR-10 Optimizer-Dependent Hyperparameters.} Learning rate, weight decay coefficient, and learning rate scheduler decay hyperparameter values which vary across optimizer choice. Each hyperparameter value has been tuned to optimize the performance of SqueezeNext architectures employing the ReLU hidden non-linearity.}
    \label{tab:CIFAR_sup_sq_hps}
    \begin{tabular}{||c c c c||} 
 \hline
 Optimizer & learning rate & weight decay & learning decay\\
 \hline\hline
 SGD & 0.1 & 0.003 & 0.2\\ 
 \hline
 Momentum & 0.1 & 0.0007 & 0.2\\ 
 \hline
 AdamW & 0.005 & 0.005 & 0.2\\ 
 \hline
 RMSprop & 0.0002 & 0.005 & 0.4\\
 \hline
\end{tabular}
\vspace{1.0\baselineskip}
\end{table}

\begin{table}
\centering
\caption[SqueezeNext CIFAR-10 Test Accuracy with Optimizers.]
{\textbf{SqueezeNext CIFAR-10 Test Accuracy with Optimizers.} average test accuracy of the SqueezeNext architecture training over 200 epochs on the CIFAR-10 dataset across various optimization algorithms. Each data point has been averaged over a minimum of 5 trials.}
{
\begin{tabular}{||c c c c c||} 
 \hline
 Name & SGD & Momentum & AdamW & RMSprop\\
 \hline\hline
 TeLU & 91.40$\pm$0.11 & \textbf{90.96}$\pm$0.29 & \textbf{91.22}$\pm$0.07 & \textbf{89.86}$\pm$0.28\\ 
 \hline
 ReLU & \textbf{91.84}$\pm$0.33 & 90.77$\pm$0.16 & 90.31$\pm$0.44 & 88.59$\pm$0.14\\ 
 \hline
 ELU & 90.96$\pm$0.24 & 90.05$\pm$0.42 & 90.94$\pm$0.15 & 89.45$ \pm $0.39 \\ 
 \hline
 SiLU & 78.61$\pm$6.3 & 84.10$\pm$1.1 & 90.24$\pm$0.38 & 66.00$\pm$1.3\\
 \hline
 GELU & 88.42$\pm$0.28 & 89.33$\pm$0.24 & 90.37$\pm$0.11 & 80.68$\pm$1.2 \\
 \hline
 Mish & 89.87$\pm$0.21 & 90.04$\pm$0.25 & 90.85$\pm$0.30 & 87.39$\pm$0.17\\
 \hline
 Logish & 61.44$\pm$29 & 66.10$\pm$3.6 & 86.20$\pm$1.1 & 43.20$\pm$19\\ 
 \hline
 Smish & 77.28$\pm$3.0 & 68.60$\pm$2.2 & 82.54$\pm$1.0 & 66.91$\pm$2.3\\ 
 \hline
\end{tabular}}
\label{tab:CIFARSUM}
\vspace{1.0\baselineskip}
\end{table}

We observe that in the momentum, AdamW, and RMSprop optimizer-oriented ReLU configurations, TeLU performs presents a significant improvement over ReLU. This reinforces our ongoing observation that TeLU tends to be superior over ReLU in configurations that optimize ReLU performance. The SGD-oriented ReLU configuration presented the only case where ReLU performed best out of all activation functions. We note that the SGD optimizer was the most challenging to tune due to their smaller viable hyperparameter subspaces, yet it yielded the highest accuracies for the architecture being optimized once an adequate tuning was found. This behavior regarding SGD is common \cite{WhySGD}, and may help explain why TeLU did not outperform ReLU in this context. Since the success of SGD is particularly inconsistent between similar configurations, it is expected that the substitution of activation functions within an architecture adds some additional tuning costs. TeLU, however, retains comparable performance with ReLU in the SGD case and shows significant improvement over all other activation functions.

To visualize the empirical similarity between TeLU and ReLU further, we plot the progression of validation accuracies in each architecture while optimizing according to the SGD and momentum-accelerated SGD optimizers in Figures \ref{fig:SGDSqz10} and \ref{fig:MomentumSqz10}. We find that architectures employing each nonlinearity learn in consistent patterns in both cases. ReLU, having little to no empirical self regularization according to our observations so far, benefits from large weight decays. Therefore, we see that more self-regularized activation functions cause their architecture's validation accuracy to saturate to a lower value in this high regularization context. We also notice that TeLU and ELU architectures initially converge at a faster rate than their competition, similar to the persistent gradient experiments. We hypothesize that this occurs for the same reasons of having strong gradients in the unbounded region of activation without suffering from dying neuron problems. Eventually; ReLU, TeLU, and ELU architectures saturate to a similar validation accuracy and TeLU generalizes to a superior test accuracy. 

\begin{figure*}
    \centering
    \includegraphics[width=0.75\linewidth]{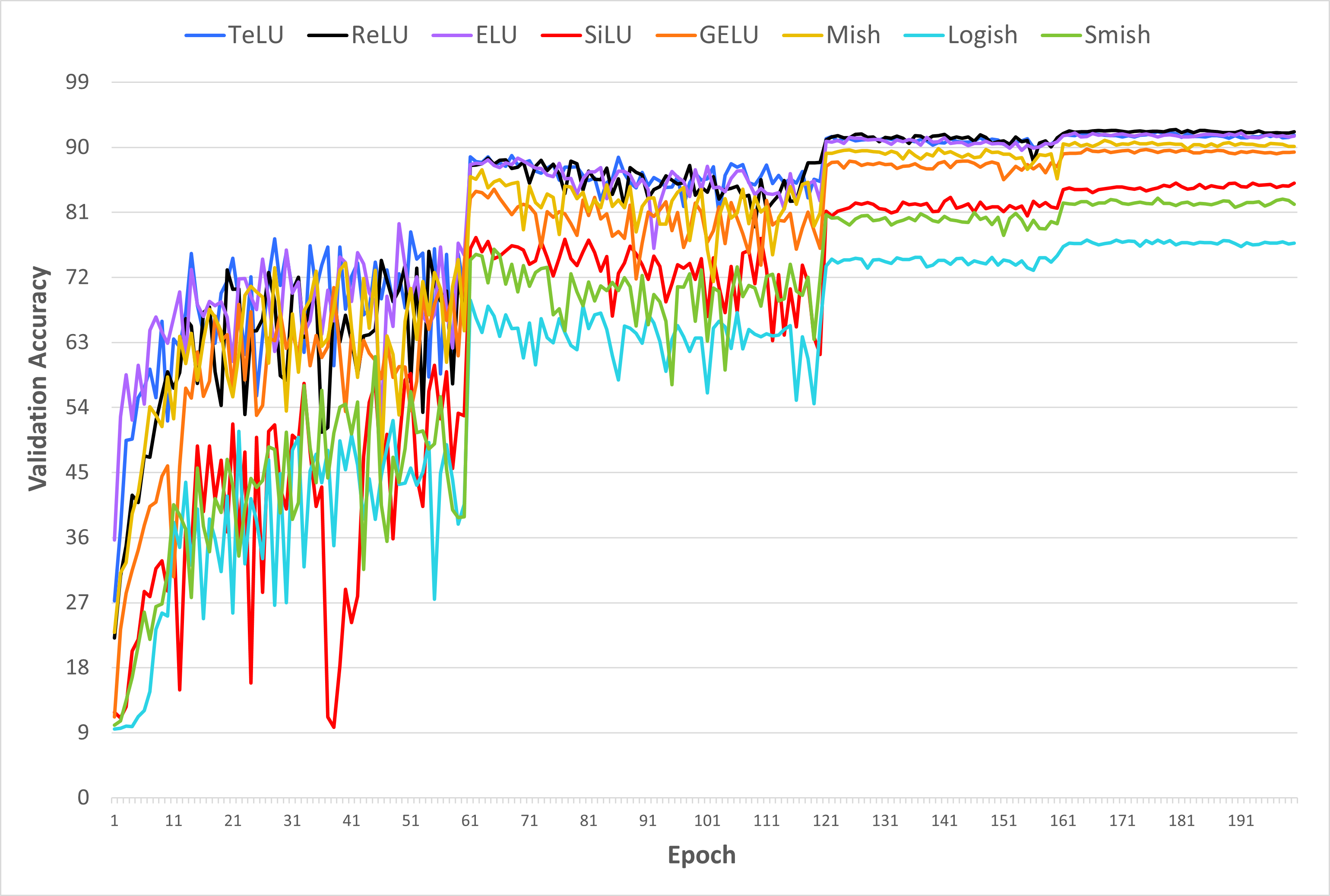}
    \caption[SqueezeNext on CIFAR-10 with SGD Validation Accuracy.]
    {\textbf{SqueezeNext on CIFAR-10 with SGD Validation Accuracy.} Shows the progression of validation accuracy throughout 200 epochs of the SqueezeNext architecture being trained on the CIFAR-10 dataset with the Mini-batch SGD optimizer. Each data point has been averaged over 5 trials.}
    \label{fig:SGDSqz10}
\end{figure*}

\begin{figure*}
    \centering
    \includegraphics[width=0.75\linewidth]{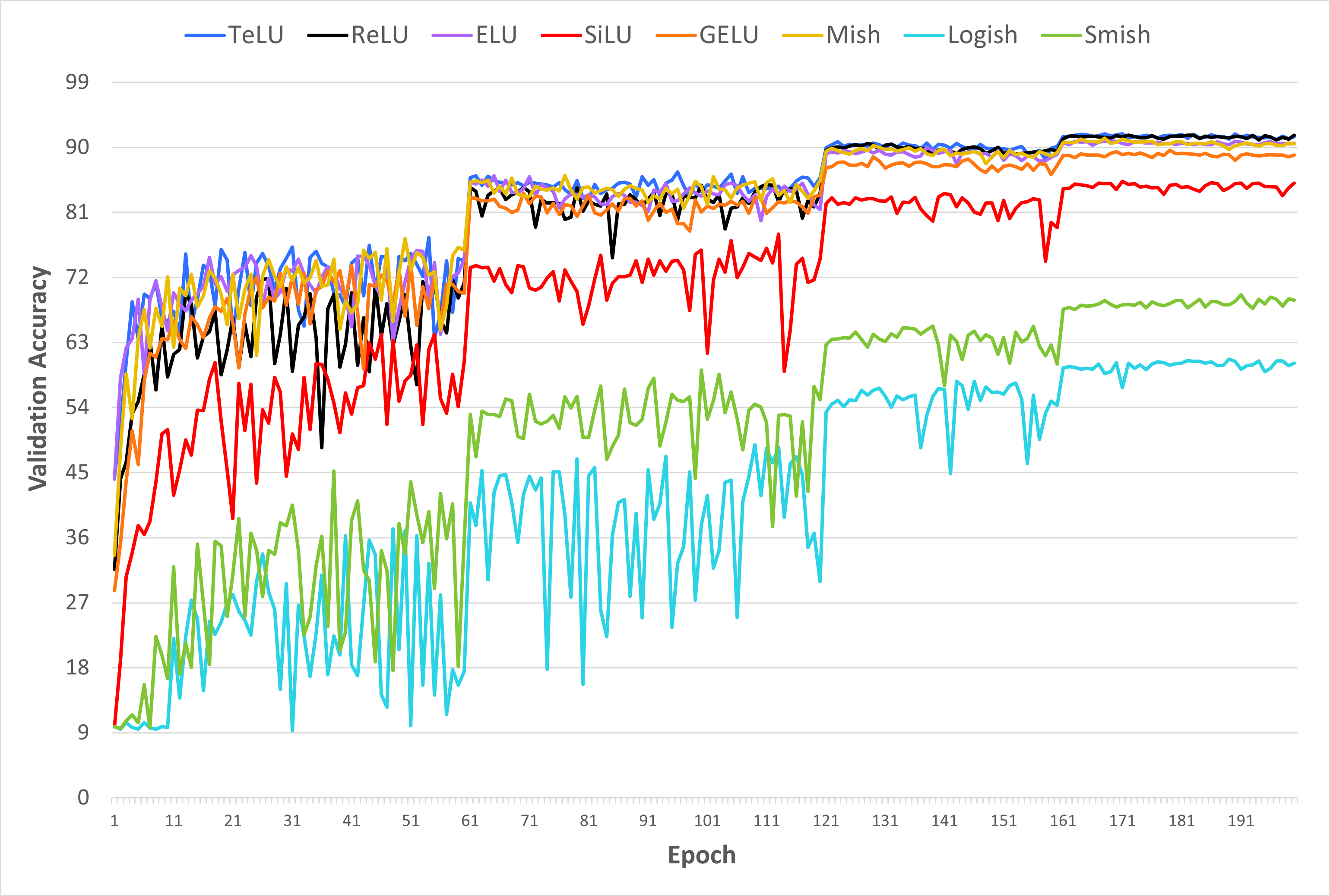}
    \caption[SqueezeNext on CIFAR-10 with Momentum-Accelerated SGD Validation Accuracy.]
    {\textbf{SqueezeNext on CIFAR-10 with Momentum-Accelerated SGD Validation Accuracy.} Shows the progression of validation accuracy throughout 200 epochs of the SqueezeNext architecture being trained on the CIFAR-10 dataset with the momentum-accelerated Mini-batch SGD optimizer. Each data point has been averaged over 10 trials.}
    \label{fig:MomentumSqz10}
\end{figure*}


We emphasize for these ReLU-similarity experiments that all experiments are done with configurations that are meant to favor the learning of ReLU architectures. This way, an activation function's empirical similarity to ReLU in terms of viable configurations can be highlighted. Based on the resulting observations, we conclude that TeLU and ReLU share similar ideal hyperparameter subspaces, allowing them to perform competitively after training with the same configurations. This suggests that architectures and tuning strategies optimized for ReLU can also be effectively applied to TeLU, minimizing the cost of substituting the hidden nonlinearities. In practice, we found that TeLU performs well with ReLU hyperparameters but may achieve even better results with a $20\%$ reduction in weight decay and learning rate. Conversely, ReLU appears to perform optimally with a $25\%$ increase in weight decay and learning rate when applied to a TeLU configuration.

\subsection{Analytic Universal Approximation}
\label{subsect:AUAexp}

\subsubsection{MNIST Dataset with Variational AutoEncoder}

To highlight the practical relevance of TeLU being an analytic universal approximation, we compare its training and testing loss to that of ReLU on the data compression and expansion of MNIST samples. A Variational AutoEncoder (VAE) \cite{VAE} architecture is used with encoder layers of size 784, and 512, and 32. Unlike the standard autoencoder \cite{StandardAE}, the encoder of a VAE concludes with a probability distribution of the latent space rather than the latent space itself. From this, the forward pass of a VAE's decoder samples from the distribution to generate a latent space representation before continuing to reconstruct the image. We apply the standard KL Divergence Loss \cite{KLDLoss} to the sampling at the start of decoding as well the mean squared error (MSE) \cite{MSELoss} reconstruction loss. We use a batch size of 100 and train ReLU and TeLU networks over 50 epochs.Table \ref{tab:VAEHyperparams} summarizes these hyperparameters. In Table \ref{tab:VAE}, we observe the training and testing errors of each model across varying learning rates to show that the results are consistent. ReLU, a piecewise-linear universal approximator is not analytic and therefore does not form smooth approximations to the target representation of the task. Therefore, we observe that TeLU reconstructions outperform ReLU reconstructions. 

\begin{table}[]
    \centering
    \caption[Variational AutoEncoder on MNIST Configuration.]
    {\textbf{Variational AutoEncoder on MNIST Configuration.} Configuration of experiments ran on a Variational AutoEncoder model on the MNIST dataset. }
    \label{tab:VAEHyperparams}
    \begin{tabular}{||c c||} 
 \hline
 Hyperparameter & Value \\
 \hline\hline
 Train/Val & 60,000 / 10,000\\ 
 \hline
 Architecture & VAE\\ 
 \hline
 Layer Widths & 784, 512, 256, 32, 256, 512, 784\\ 
 \hline
 Initialization & Xavier Uniform\\ 
 \hline
 Loss Functions & BCE and KLD\\
 \hline
 Optimizer & Adam\\
 \hline
 Batch Size & 100\\
 \hline
 Learning Rate & 0.001, 0.005, or 0.0001\\
 \hline
 Number of Epochs & 50\\ 
 \hline
 Number of Trials & 10 per learning rate\\ 
 \hline
\end{tabular}
\end{table}


\begin{table}[]
    \centering
    \caption[VAE MNIST Reconstruction Error.]
    {\textbf{VAE MNIST Reconstruction Error.} Training and testing errors of our Variational AutoEncoder compressing and decompressing samples, averaged over 5 trials for each combination.}
    \label{tab:VAE}
    \begin{tabular}{||c c c||} 
 \hline
 Name & Train Error & Test Error\\
 \hline\hline
 TeLU lr=0.001 & \textbf{100.77}$\pm$0.240 & \textbf{100.56}$\pm$0.245 \\ 
 \hline
 ReLU lr=0.001 & 102.00$\pm$0.028 & 101.78$\pm$0.052 \\
 \hline
 TeLU lr=0.0005 & \textbf{95.102}$\pm$0.090 & \textbf{96.367}$\pm$0.097 \\ 
 \hline
 ReLU lr=0.0005 & 96.537$\pm$0.248 & 97.884$\pm$0.145 \\
 \hline
 TeLU lr=0.0001 & \textbf{95.234}$\pm$0.126 & \textbf{97.242}$\pm$0.174 \\ 
 \hline
 ReLU lr=0.0001 & 96.785$\pm$0.164 & 98.844$\pm$0.397 \\
 \hline
\end{tabular}
\end{table}

\subsubsection{Penn TreeBank Dataset with RNN Architectures}

Additionally, we show that TeLU performs competitively when used as the hidden nonlinearity in Recurrent Neural Networks (RNN) \cite{RNNBPTT} architectures on the Penn TreeBank (PTB) \cite{PTBDataSet} dataset. First, we define an Elman architecture \cite{ELMAN1990179}, the simplest form of a recurrent neural network. In an Elman recurrent neural network, hidden neurons perform their forward pass by taking their previous activation as inputs, allowing for the introduction of neuron states. We set an input embedding size of 450, a hidden layer width of 650, and 2 hidden layers. We train over 40 epochs with a batch size of 20, a learning rate of 0.2, and a hidden neuron dropout rate of 0.2. We truncate the backpropagation through time (BPTT) \cite{RNNBPTT} steps to 25 and perform gradient clipping to limit the absolute value of each computer gradient during the backward pass to 0.1. We compare the Perplexity (PPL) \cite{PPL} of TeLU, ReLU, and Tanh architectures in Table \ref{tab:RNN} over 10 seeded trials and observe that TeLU networks offer significant improvements over networks employing ReLU and Tanh nonlinearities. These 10 trials were seeded with the first 10 whole numbers, 0 through 10. The hyperparameter values used for this Elman network experiment are detailed in Table \ref{tab:ElmanHyperparams}.

We follow this success to demonstrate similar improvements in Long Short Term Memory (LSTM) RNNs. We utilize the same batch size, number of hidden layers, and number of epochs as in the Elman experiments. Since LSTM networks benefit from additional memory and gating over Elman networks, we increase the number of backpropagation steps to 35. We also empirically determine that the embeddeding size of 550 and the hidden layer size of 750 are preferred by LSTMs. To balance this modification, we increase our hidden neuron dropout rate to 0.65 and reduce our gradient clipping to 0.25 to reduce overfitting of the model. We maintain the gating activation function as logistic sigmoids and hyperbolic tangents like in standard LSTMs so that they may continue functioning as gates. Since these sigmoid gates suffer from vanishing gradients, we increase our learning rate to 20. Lastly, we modify only the LSTM cells' output activation function to employ either TeLU, ReLU, or Tanh. We observe significant improvements of TeLU over ReLU, and slight improvements over Tanh in Table \ref{tab:RNN}. For reproducibility, we detail this seeded LSTM experiment in Table \ref{tab:LSTMHyperparams}.

\begin{table}[]
    \centering
        \caption[RNN Penn Treebank Test PPL Summary.]
        {\textbf{RNN Penn Treebank Test PPL Summary.} Elman and LSTM RNN architectures' testing PPL, averaged over 10 trials per non-linearity.}
    \label{tab:RNN}
    \begin{tabular}{||c c c||} 
 \hline
 Name & Elman & LSTM\\
 \hline\hline
 TeLU & \textbf{132.65}$\pm$2.48 & \textbf{81.842}$\pm$0.37 \\ 
 \hline
 ReLU & 140.63$\pm$4.25 & 86.855$\pm$0.62 \\
 \hline
 Tanh & 138.65$\pm$5.39 & 82.123$\pm$1.25 \\ 
 \hline
\end{tabular}
\end{table}

\begin{table}[]
    \centering
    \caption[Elman RNN on Penn Treebank Experiment Configuration.]
    {\textbf{Elman RNN on Penn Treebank Experiment Configuration.} Configuration of experiments ran on an Elman RNN model on the Penn Treebank dataset.}
    \label{tab:ElmanHyperparams}
    \begin{tabular}{||c c||} 
 \hline
 Hyperparameter & Value \\
 \hline\hline
 Train/Val/Test Split & 930,000/74,000/82,000 words\\ 
 \hline
 Architecture & Elman Network \\ 
 \hline
 Embedding Size & 450\\ 
 \hline
 Hidden Layer Count & 2\\ 
 \hline
 Hidden Layer Size & 650\\ 
 \hline
 Dropout & 0.2\\ 
 \hline
 Initialization & Uniform(-0.1,0.1)\\ 
 \hline
 Optimizer & Mini-batch SGD\\
 \hline
 Batch Size & 20\\
 \hline
 BPTT Steps & 25\\
 \hline
 Gradient Clipping & 1.0\\
 \hline
 Seeds & 0, 1, ..., 9\\
 \hline
 Learning Rate & 0.2\\
 \hline
 Number of Epochs & 40\\ 
 \hline
 Number of Trials & 9\\ 
 \hline
\end{tabular}
\vspace{1.0\baselineskip}
\end{table}

\begin{table}[]
    \centering
    \caption[LSTM RNN on Penn Treebank Experiment Configuration.]
    {\textbf{LSTM RNN on Penn Treebank Experiment Configuration.} Configuration of experiments ran on an LSTM RNN model on the Penn Treebank dataset.}
    \label{tab:LSTMHyperparams}
    \begin{tabular}{||c c||} 
 \hline
 Hyperparameter & Value \\
 \hline\hline
 Train/Val/Test Split & 930,000/74,000/82,000 words\\ 
 \hline
 Architecture & LSTM Network \cite{LSTM} \\ 
 \hline
 Embedding Size & 550\\ 
 \hline
 Hidden Layer Count & 2\\ 
 \hline
 Hidden Layer Size & 750\\ 
 \hline
 Dropout & 0.65\\ 
 \hline
 Initialization & Uniform(-0.1,0.1)\\ 
 \hline
 Optimizer & Mini-batch SGD\\
 \hline
 Batch Size & 20\\
 \hline
 BPTT Steps & 35\\
 \hline
 Gradient Clipping & 0.25\\
 \hline
 Seeds & 0000, 1111, ..., 9999\\
 \hline
 Learning Rate & 0.2\\
 \hline
 Number of Epochs & 40\\ 
 \hline
 Number of Trials & 9\\ 
 \hline
\end{tabular}
\vspace{1.0\baselineskip}
\end{table}


\subsubsection{FashionMNIST Dataset with MLP Robustness Experiment}

To observe the robustness benefits of a smooth universal approximator, we designed an experiment with an MLP on the Fashion-MNIST dataset \cite{MNIST}. We performed ten trials, seeded from integers 0 to 9, where an architecture employing the TeLU, ReLU, and ELU nonlinearities as its hidden activation function trains over 100 epochs. We partition the FashionMNIST according to 50,000 training samples; 10,000 validation samples; and 10,000 testing samples. In order to reach robust solutions, we use an SGD optimizer \cite{SGD1} with a large learning rate of 0.05 and a large weight decay coefficient of 0.001. We set our momentum parameter update acceleration coefficient to 0, upon noticing that momentum tends to produce less robust models in this simple experiment. We define a batch size of 256 with the intention to perform more regularized weight parameter update steps, resulting in mini-batch gradient descent \cite{minibatchgradientdescent}. After each epoch of training, we evaluate our model on the validation partition and save a checkpoint if validation accuracy has improved. After 100 epochs, we evaluate the best validation accuracy checkpoint on 11 instances of our test dataset, each with an increasing degree of Gaussian noise standard deviation summed to the inputs. We detail our experimental setup in Table \ref{tab:RobustnessHyperparams}.

\begin{table}[]
    \centering
    \caption[MLP FashionMNIST Robustness Configuration.]
    {\textbf{MLP FashionMNIST Robustness Configuration.} Configuration of experiments ran on an MLP model on the FashionMNIST dataset and tested on various standard deviations of zero-centered Gaussian noise. }
    \label{tab:RobustnessHyperparams}
    \begin{tabular}{||c c||} 
 \hline
 Hyperparameter & Value \\
 \hline\hline
 Train/Val/Test Split & 50,000 / 10,000 / 10,000\\ 
 \hline
 Hidden Layer Count & 8\\ 
 \hline
 Hidden Layer Width & 128\\ 
 \hline
 Initialization & Xavier Uniform\\ 
 \hline
 Optimizer & Mini-batch SGD\\
 \hline
 Momentum & 0.0\\
 \hline
 Batch Size & 256\\
 \hline
 Learning Rate & 0.05\\
 \hline
 Weight Decay & 0.001\\
 \hline
 Number of Epochs & 100\\ 
 \hline
 Number of Trials & 10\\ 
 \hline
\end{tabular}
\end{table}

\begin{figure}
    \centering
    \includegraphics[width=0.75\linewidth]{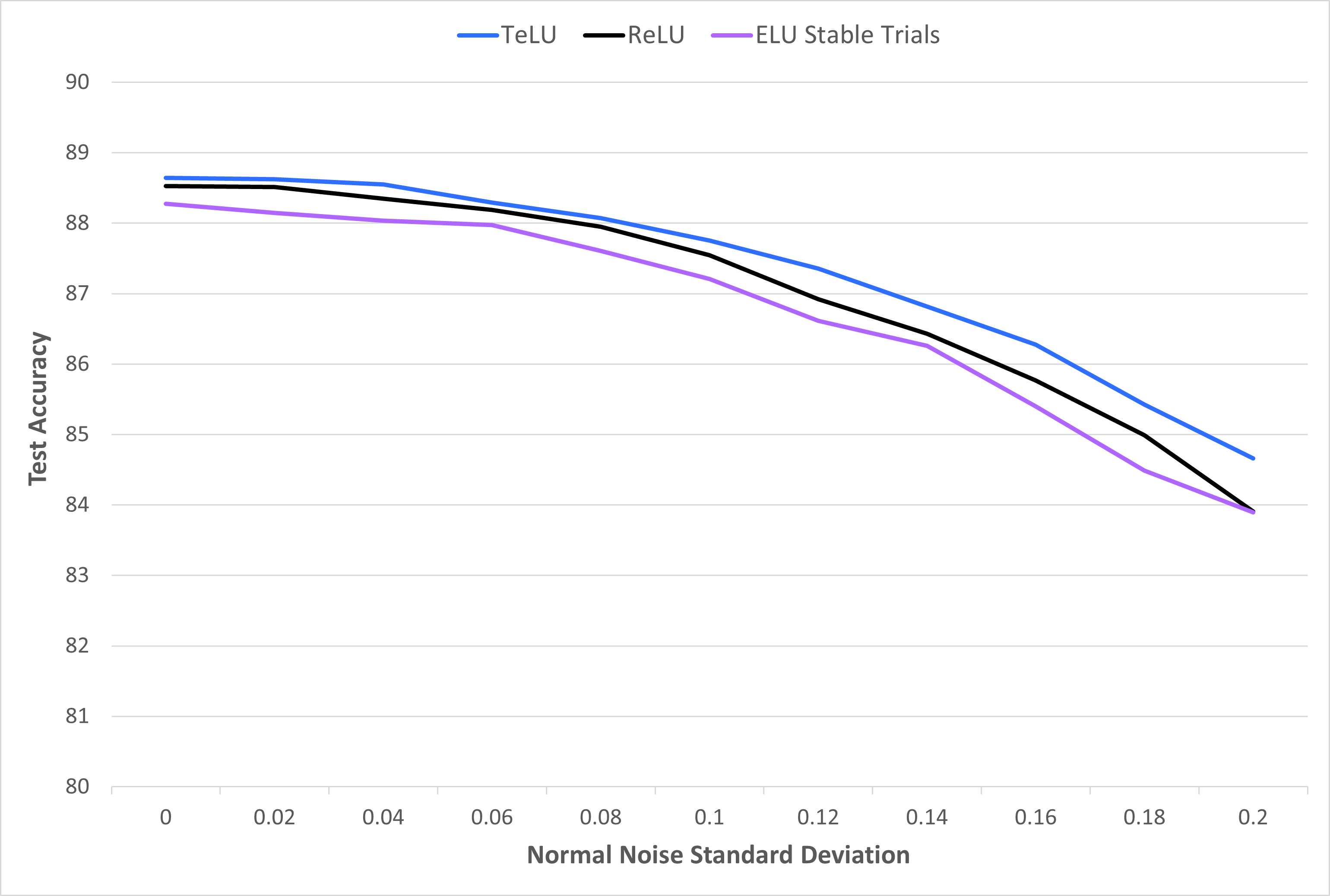}
    \caption[MLP FashionMNIST Robustness Test Accuracy.]
    {\textbf{MLP FashionMNIST Robustness Test Accuracy.} Shows TeLU, ReLU, and stable ELU robustness accuracies across various standard deviations of Gaussian noise.}
    \label{fig:RobustnessFigure}
    \vspace{1.0\baselineskip}
\end{figure}

\begin{table}[]
\centering
 \caption[MLP FashionMNIST Robustness Test Accuracy.]
 {\textbf{MLP FashionMNIST Robustness Test Accuracy.} Average test accuracy of our MLP model training over 100 epochs on the FashionMNIST dataset with increasing levels of Gaussian noise applied to the testing samples. Each data point has been averaged over 10 trials.}
    \begin{tabular}{||c c c c c||} 
 \hline
 Stdv & TeLU & ReLU & ELU & ELU Stable \\
 \hline\hline
 0.00 & \textbf{88.64} & 88.53 & 25.65 & 88.27 \\ 
 \hline
 0.02 & \textbf{88.62} & 88.51 & 25.62 & 88.14 \\
 \hline
 0.04 & \textbf{88.54} & 88.35 & 25.60 & 88.03\\
 \hline
 0.06 & \textbf{88.29} & 88.18 & 25.59 & 87.97\\
 \hline
 0.08 & \textbf{88.07} & 87.94 & 25.52 & 87.60 \\
 \hline
 0.10 & \textbf{87.76} & 87.54 & 25.44 & 87.21\\
 \hline
 0.12 & \textbf{87.35} & 86.91 & 25.32 & 86.61\\
 \hline
 0.14 & \textbf{86.81} & 86.43 & 25.25 & 86.26\\
 \hline
 0.16 & \textbf{86.28} & 85.77 & 25.08 & 85.40\\
 \hline
 0.18 & \textbf{85.42} & 84.98 & 24.89 & 84.49\\
 \hline
 0.20 & \textbf{84.66} & 83.90 & 24.77 & 83.89\\
 \hline
\end{tabular}
\label{tab:RobustnessTable}
\end{table}

We record our testing accuracies over the 11 noise standard deviations in Table \ref{tab:RobustnessTable}. We note that while tuning the hyperparameters of the experiment to optimize robustness without reducing accuracy, ELU experienced numerical instability at favorable configurations. To focus on the robustness of the activation functions, rather than the stability, we separate ELU's testing accuracies into two columns. The first column depicts the average ELU test accuracy for each noise level, while the second "ELU Stable" column calculates the average accuracies for stable ELU trials. We also visualize the stable behavior of all three architectures in Figure \ref{fig:RobustnessFigure} across the gradually increasing noise distributions. We did not include the overall testing average of ELU architectures in this Figure, as this would significantly reduce the resolution needed to valuably compare the robustness exhbited each activation function. We note that TeLU, an analytic universal approximator, significantly outperforms the piecewise ReLU and ELU activation functions when employed as the hidden layer activation of an MLP.

\subsection{Stability}
\label{subsect:stabilityexp}
In Subsection \ref{subsect:stabilitytheory}, we explored how an activation function's properties; such as simple formulation, mitigation of vanishing and exploding gradients, smoothness, and reduced output bias; contribute to the overall stability of a neural network. With these foundational concepts established, we now shift from theory to practice by validating these benefits through experimental analysis. In this subsection, we will demonstrate how these theoretical advantages translate into real-world improvements in training stability, convergence speed, and model performance. By conducting empirical tests across various architectures and datasets, we aim to showcase the tangible impact of these activation functions on deep learning tasks, reinforcing the connection between theory and practical application.

\subsubsection{FashionMNIST Dataset with MLP with Varied Network Depths}

First we demonstrate that TeLU architectures offer stable learning across a wide range of model depths. We define MLP architectures that employ hidden activations of TeLU, ReLU, and Mish with a starting number of 16 hidden layers of width 128. We then train models with a momentum-accelerated Mini-batch Gradient Descent \cite{minibatchgradientdescent} \cite{Momentum} on the MNIST dataset \cite{MNIST} with a weight decay coefficient of 0.0005; setting aside 10,000 validation images and 10,000 testing images. We repeat this process for incrementally increasing hidden layers until reaching a hidden depth of 44. We detail our precise experimental configuration in Table \ref{tab:StabilityDepthHyperparams} and show the average testing accuracies of each architecture in Figure \ref{fig:stabilitydepth05}. We observe that ReLU, being a piecewise linear function with greater output bias than TeLU and Mish, succumbs to instability as depth increases. TeLU and Mish, both being smooth non-monotonic functions with less biased outputs, maintain their stability and perform predictably across all depths. Our detailed configuration is given by Table \ref{tab:StabilityDepthHyperparams}.

\begin{table}[]
    \centering
    \caption[MLP FashionMNIST Cross-Depth Stability Experiment Configuration.]
    {\textbf{MLP FashionMNIST Cross-Depth Stability Experiment Configuration.} Configuration of experiments ran on MLP architectures with different hidden layer counts on the FashionMNIST dataset. }
    \label{tab:StabilityDepthHyperparams}
    \begin{tabular}{||c c||} 
 \hline
 Hyperparameter & Value \\
 \hline\hline
 Train/Val/Test Split & 50,000 / 10,000 / 10,000\\ 
 \hline
 Hidden Layer Count & [16, 44]\\ 
 \hline
 Hidden Layer Width & 128\\ 
 \hline
 Initialization & Xavier Uniform\\ 
 \hline
 Optimizer & Mini-batch SGD\\
 \hline
 Momentum & 0.9\\
 \hline
 Batch Size & 128\\
 \hline
 Learning Rate & 0.005\\
 \hline
 Weight Decay & Repeated with 0.0005 and 0.001\\
 \hline
 Number of Epochs & 100\\ 
 \hline
 Number of Trials & 10 trials per depth\\ 
 \hline
\end{tabular}
\end{table}

\begin{figure}
    \centering
    \includegraphics[width=0.75\linewidth]{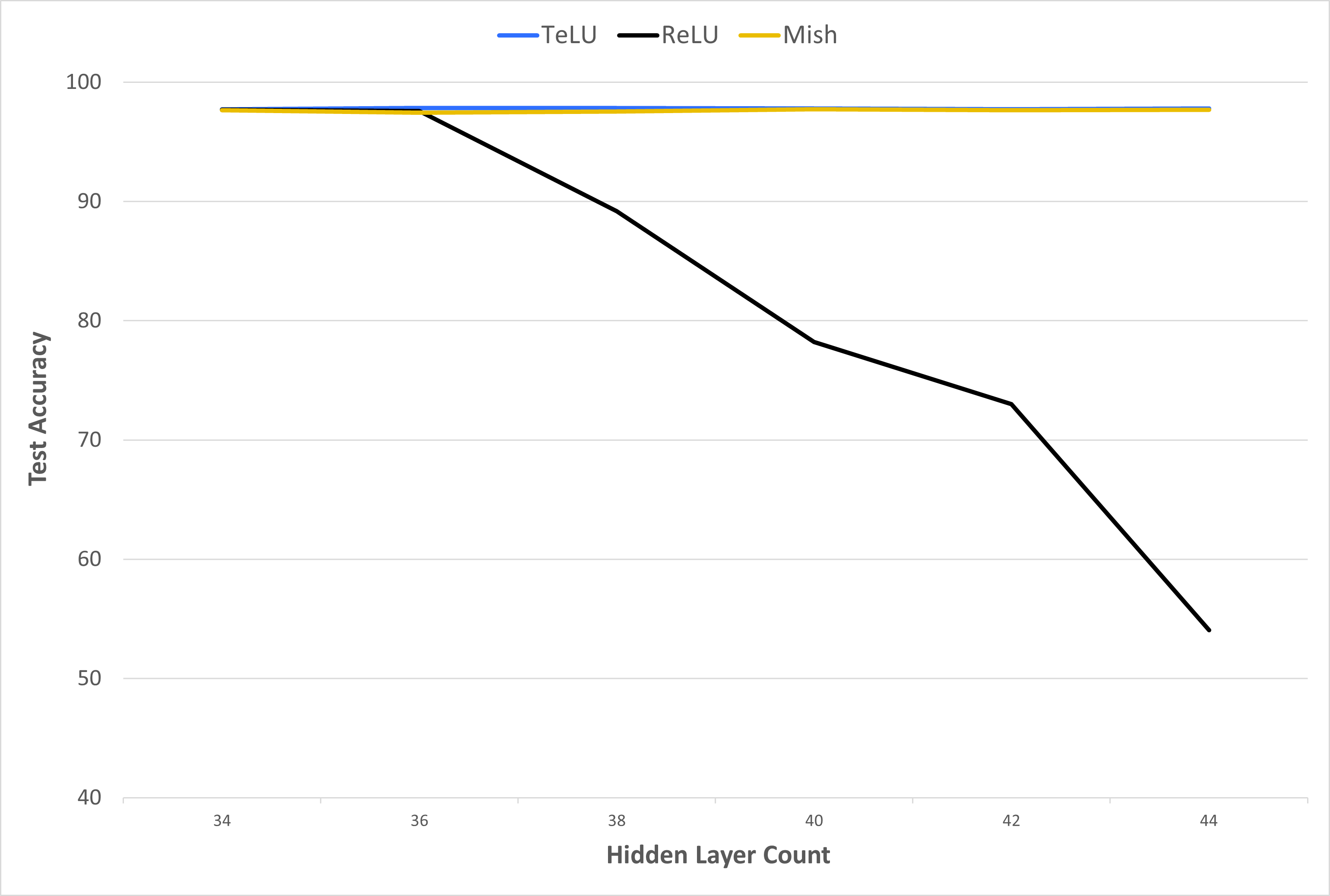}
    \caption[MLP FashionMNIST Test Accuracy Across Depths with Weight Decay 0.0005.]
    {\textbf{MLP FashionMNIST Test Accuracy Across Depths with Weight Decay 0.0005.} Architecture stability as number of hidden layers increases while using weight decay coefficient of 0.0005.}
    \label{fig:stabilitydepth05}
    \vspace{1.0\baselineskip}
\end{figure}

We then repeat the experiment for hidden layer counts between 34 and 44 with an increased weight decay coefficient of 0.001. We observe in Figure \ref{fig:enter-stabilitydepth1} how ReLU, being less self-regularized, stabilizes with greater L2 regularization. Mish, however, exhibits drastic instability with this minor increase in weight decay. We understand that Mish is a self-regularized non-monotonic function \cite{Mish} with a more complex formulation than TeLU and ReLU. In this context, $Mish(x) = x \cdot tanh(ln(1+e^x))$ may be interpreted as a version of $TeLU(x) = x \cdot tanh(e^x)$ with the addition a nested natural logarithm nonlinearity. This extra intermediate computation exposes each Mish activation to numerical underflow. Once underflow occurs, it may propagate along the remaining layers until dictating a particular choice of inference. At the worst case of a depth of 44, this underflow is more likely to occur somewhere in the network and the average accuracy of the model be comparable to that of a random guess.

\begin{figure}
    \centering
    \includegraphics[width=0.75\linewidth]{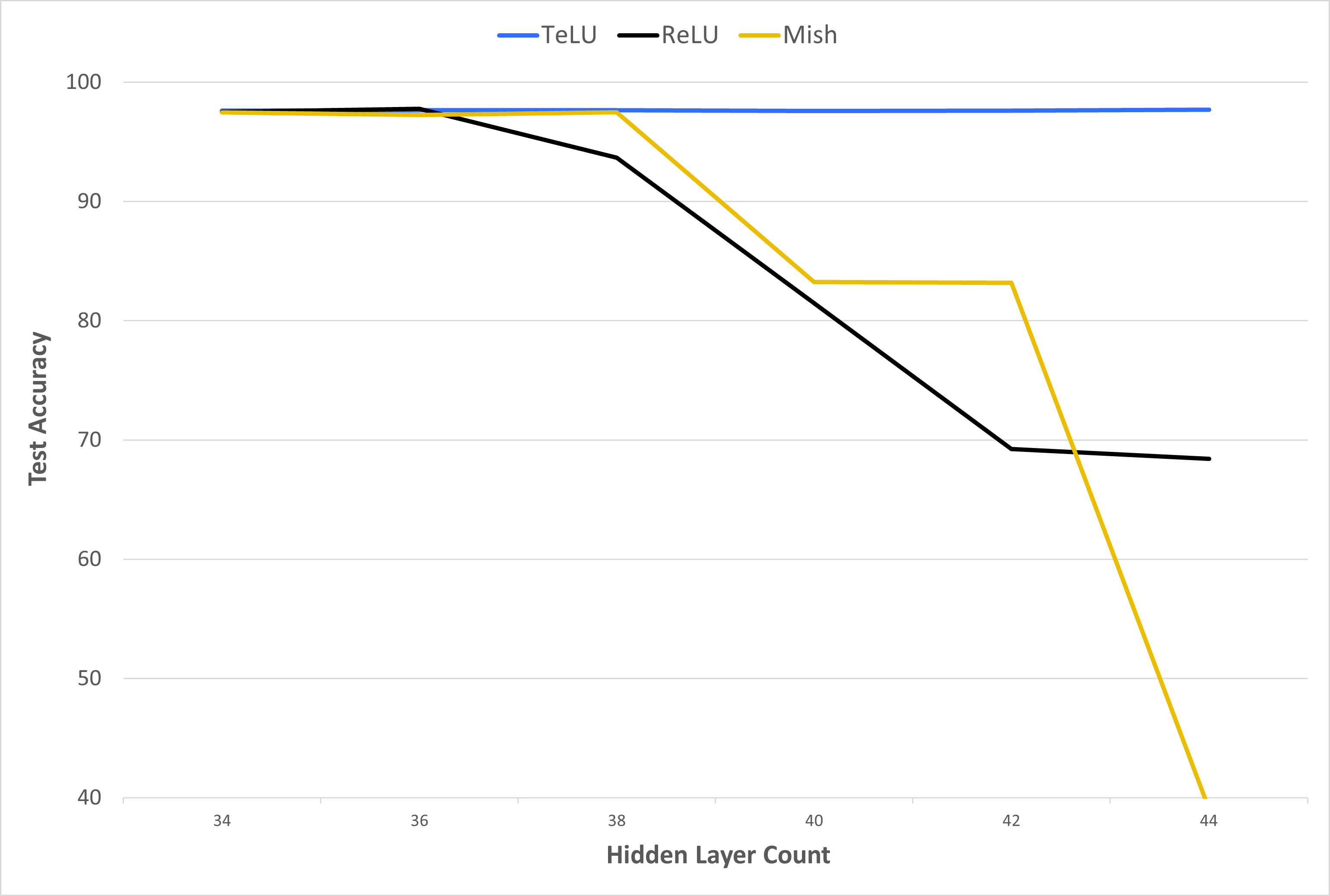}
    \caption[MLP FashionMNIST Test Accuracy Across Depths with Weight Decay of 0.001.]
    {\textbf{MLP FashionMNIST Test Accuracy Across Depths with Weight Decay of 0.001.} Architecture stability as number of hidden layers increases while using weight decay coefficient of 0.001.}
    \label{fig:enter-stabilitydepth1}
    \vspace{1.0\baselineskip}
\end{figure}

\subsubsection{FashionMNIST Dataset with MLP with Varied Initialization Methods}

We perform a similar experiment with the MLP architecture on the FashionMNIST dataset \cite{MNIST}. This time we treat our weight initialization method as an independent variable instead of our model's depth. We define TeLU, ReLU, ELU, SiLU, GELU, Mish, Logish, and Smish MLP architectures by changing only the activation function of hidden layers. We initialize each architecture according to Xavier Uniform, Xavier Normal, He Uniform, and He Normal for different experiments. To be able to observe the impact of the different initializations, we limit the number of epochs to prevent eventual convergance on similar accuracies across initializers. We configure our experiment hyperparameters according to Table \ref{tab:StabilityInitHyperparams}. We train each combination of initializer and architecture over 10 trials. We view the averaged testing accuracies for each architecture on each initializer in Figure \ref{fig:stabilityinit}. We observe how TeLU maintains superior performance over other architectures across all initializers. We also note that the ReLU architectures also exhibit considerable accuracy, but all other activation architectures seem to vary in success across initializers.

\begin{table}[]
    \centering
    \caption[MLP FashionMNIST Cross-Initializer Stability Configuration.]
    {\textbf{MLP FashionMNIST Cross-Initializer Stability Configuration.} Configuration of experiments ran on an MLP model on the FashionMNIST dataset across various learnable parameter initialization methods. }
    \label{tab:StabilityInitHyperparams}
    \begin{tabular}{||c c||} 
 \hline
 Hyperparameter & Value \\
 \hline\hline
 Train/Val/Test Split & 50,000 / 10,000 / 10,000\\ 
 \hline
 Hidden Layer Count & 8\\ 
 \hline
 Hidden Layer Width & 128\\ 
 \hline
 Initialization & Xavier Uniform, Xavier Normal, He Uniform, He Normal\\ 
 \hline
 Optimizer & Mini-batch SGD\\
 \hline
 Momentum & 0.9\\
 \hline
 Batch Size & 128\\
 \hline
 Learning Rate & 0.1\\
 \hline
 Weight Decay & 0.001\\
 \hline
 Number of Epochs & 50\\ 
 \hline
 Number of Trials & 6 trials per initializer\\
 \hline
\end{tabular}
\end{table}

\begin{figure}
    \centering
    \includegraphics[width=0.75\linewidth]{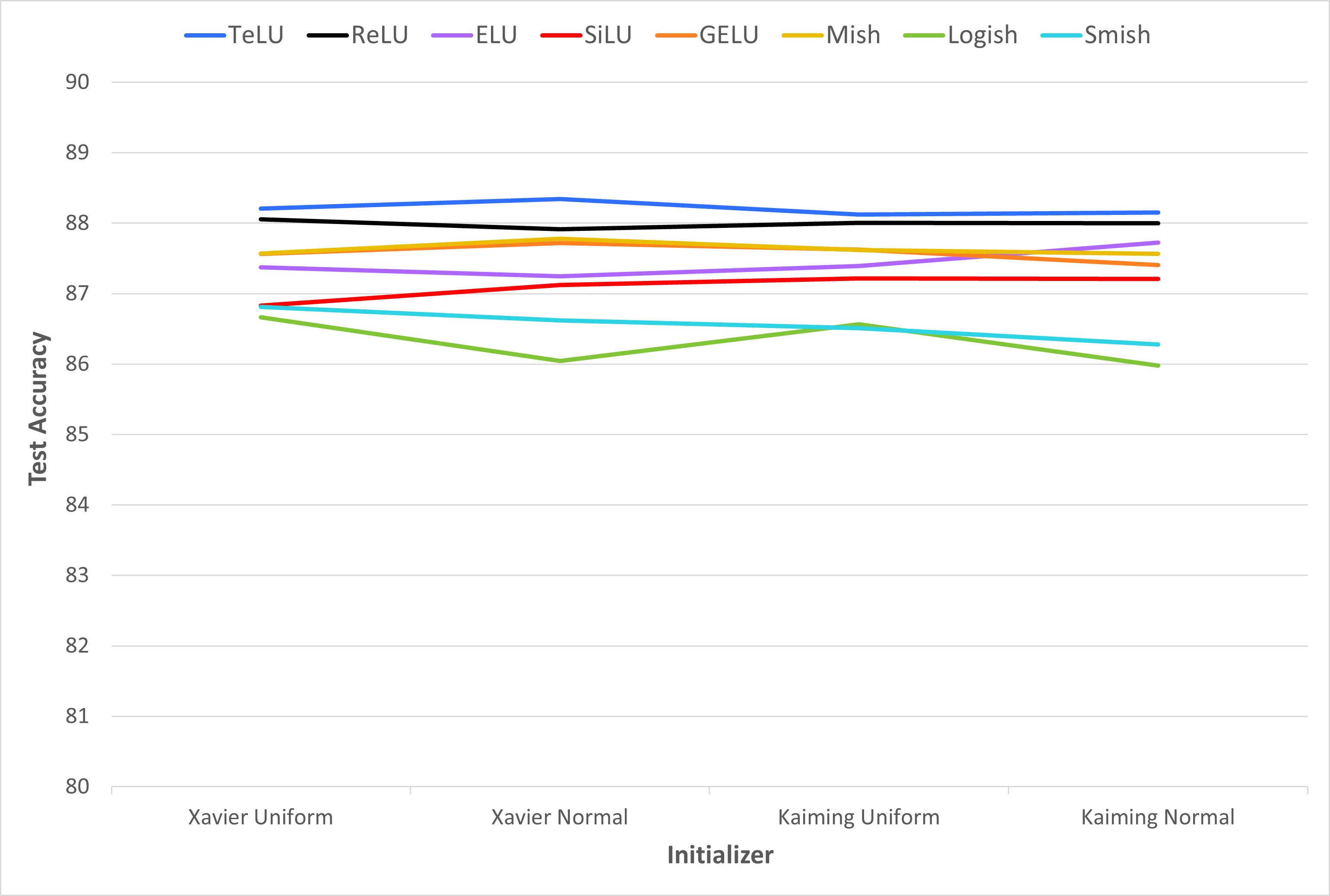}
    \caption[MLP FashionMNIST Test Accuracy Across Initializers.]
    {\textbf{MLP FashionMNIST Test Accuracy Across Initializers.} average test accuracy of our MLP model training over 50 epochs on the FashionMNIST dataset across various weight initialization methods. Each data point has been averaged over 10 trials.}
    \label{fig:stabilityinit}
\end{figure}

\subsubsection{CIFAR-100 Dataset with SqueezeNext with Varied Optimizers}

We go on to showcase the comparative stability and performance of the TeLU, ReLU, and GELU activation functions when utilized as the hidden activation function in a SqueezeNext CNN \cite{SqueezeNext}. We train the SqueezeNext architecture on the CIFAR-100 dataset \cite{cifar} for 200 epochs with the Mini-batch SGD \cite{minibatchgradientdescent} \cite{SGD1}, momentum-accelerated SGD \cite{Momentum}, AdamW \cite{AdamW}, and RMSprop \cite{RMSprop}. We utilize a learning rate scheduler that updates the base learning rate every 60 epochs. We detail the training hyperparameter values that stay constant across optimizer settings in Table \ref{tab:SqueezeNextCIFAR100}, as well as the hyperparameters that change in Table \ref{tab:cifar100_sup_sq_hps}. After each training epoch, we check to see if our validation has improved. If it has, we overwrite our checkpoint with the current weight and bias values of our model. After 5 trials of each optimizer and architecture combination's training is complete, we evaluate the best validation accuracy checkpoint on the testing partition of 10,000 CIFAR-100 images. We summarize the testing accuracies and standard deviations in Table \ref{tab:CIFAR100SUM}. Additionally, we compare the progression of the validation accuracies of each architecture for the Minibatch SGD and momentum-accelerated SGD optimizers in Figures \ref{fig:SGDC100SqzNxt} and \ref{fig:MomentumC100SqzNxt}.

\begin{table}[]
    \centering
    \caption[SqueezeNext CIFAR-100 Optimizer-Independent Hyperparameters.]
    {\textbf{SqueezeNext CIFAR-100 Optimizer-Independent Hyperparameters.} Static configuration of stability experiments ran on the SqueezeNext architecture on the CIFAR-100 dataset across different optimizers. }
    \label{tab:SqueezeNextCIFAR100}
    \begin{tabular}{||c c||} 
 \hline
 Hyperparameter & Value \\
 \hline\hline
 Train/Val/Test Split & 45,000 / 15,000 / 10,000\\ 
 \hline
 Normalization & Standard Score, detailed in Table \ref{tab:NormalizationDetails} \\ 
 \hline
 Architecture & SqueezeNext \cite{SqueezeNext}\\ 
 \hline
 Initialization & Xavier Uniform\\ 
 \hline
 Batch Size & 128\\
 \hline
 Learning Scheduler Period & 60\\
 \hline
 Data Augmentations & Random 4-padded cropping and horizontal flipping\\
 \hline
  Number of Epochs & 200\\ 
 \hline
  Number of Trials & 10\\ 
 \hline
\end{tabular}
\end{table}

\begin{table}[]
    \centering
    \caption[SqueezeNext CIFAR-100 Optimizer-Dependent Hyperparameters.]
    {\textbf{SqueezeNext CIFAR-100 Optimizer-Dependent Hyperparameters.} Learning rate, weight decay coefficient, and learning rate scheduler decay hyperparameter values which vary across optimizer choice.}
    \label{tab:cifar100_sup_sq_hps}
    \begin{tabular}{||c c c c||} 
 \hline
 Optimizer & learning rate & weight decay & gamma\\
 \hline\hline
 SGD & 0.04 & 0.003 & 0.2\\ 
 \hline
 Momentum & 0.003 & 0.003 & 0.2\\ 
 \hline
 AdamW & 0.005 & 0.005 & 0.4\\ 
 \hline
 RMSprop & 0.0002 & 0.005 & 0.4\\  
 \hline
\end{tabular}
\vspace{1.0\baselineskip}
\end{table}

\begin{table}
\centering
\caption[CIFAR-100 SqueezeNext Test Accuracy Summary.]
{\textbf{CIFAR-100 SqueezeNext Test Accuracy Summary.} A summary of the test accuracies achieved across various optimizers on the CIFAR-100 dataset by the SqueezeNext Architecture checkpoint that achieved optimal validation accuracy throughout 200 epochs of training.}
    \begin{tabular}{||c c c c c||}  
 \hline
 Name & SGD & Momentum & AdamW & RMSprop\\
 \hline\hline
 TeLU & \textbf{71.47}$\pm$0.08 & \textbf{70.53}$\pm$0.25 & \textbf{69.64}$\pm$0.07 & \textbf{68.83}$\pm$0.33\\ 
 \hline
 ReLU & 69.52$\pm$0.43 & 65.05$\pm$0.51 & 66.31$\pm$0.48 & 67.99$\pm$0.21\\ 
 \hline
 GELU & 67.09$\pm$0.36 & 66.26$\pm$29 & 66.50$\pm$0.44 & 65.19$\pm$0.25 \\ [1ex]
 \hline
\end{tabular}
\label{tab:CIFAR100SUM}
\vspace{1.0\baselineskip}
\end{table}

\begin{figure}
    \centering
    \includegraphics[width=0.75\linewidth]{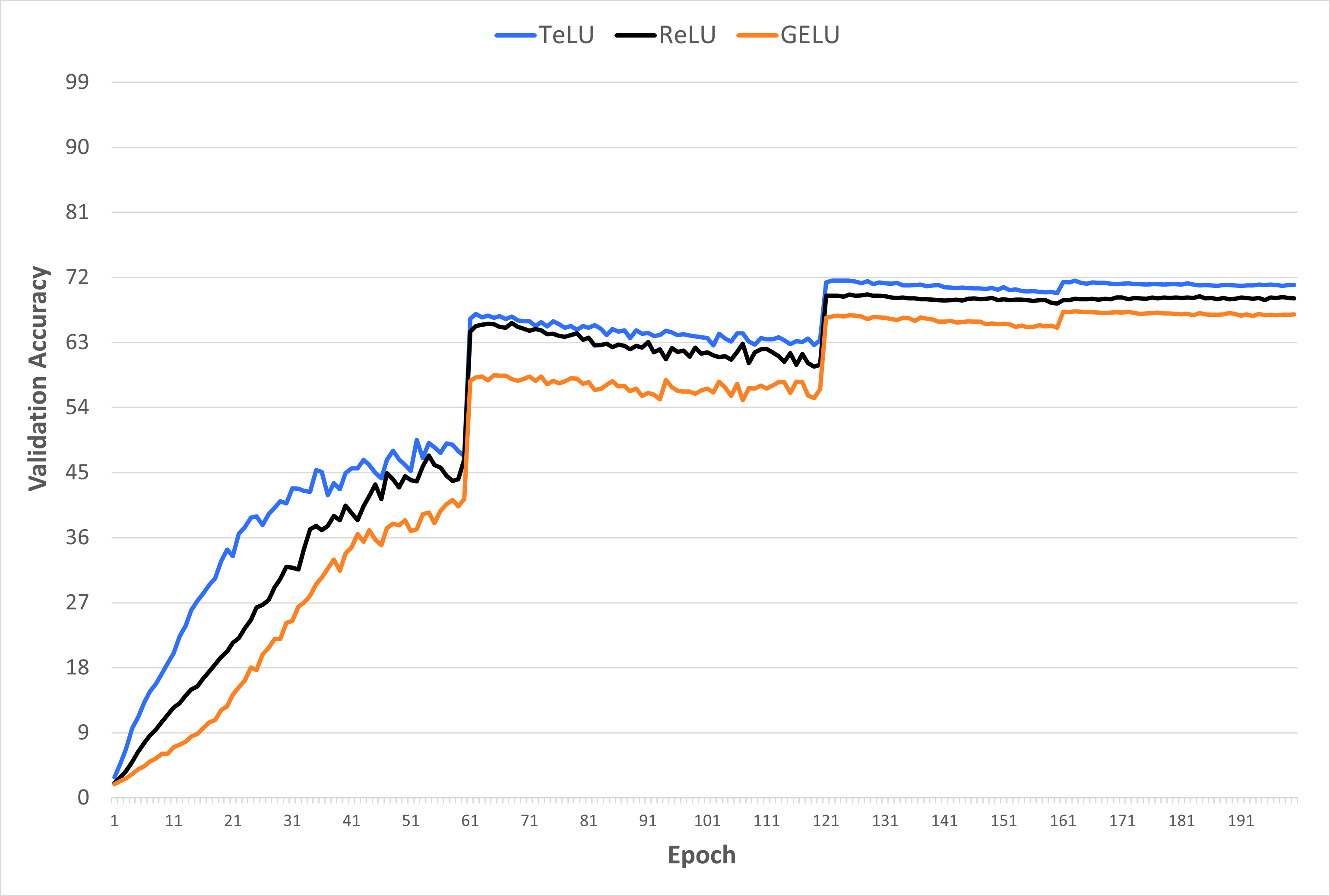}
    \caption[SqueezeNext on CIFAR-100 with SGD Validation Accuracy.]
    {\textbf{SqueezeNext on CIFAR-100 with SGD Validation Accuracy.} Shows the progression of validation accuracy throughout 200 epochs of the SqueezeNext architecture being trained on the CIFAR-100 dataset with the Mini-batch SGD optimizer. Each data point has been averaged over 5 trials.}
    \label{fig:SGDC100SqzNxt}
    \vspace{1.0\baselineskip}
\end{figure}

\begin{figure}
    \centering
    \includegraphics[width=0.75\linewidth]{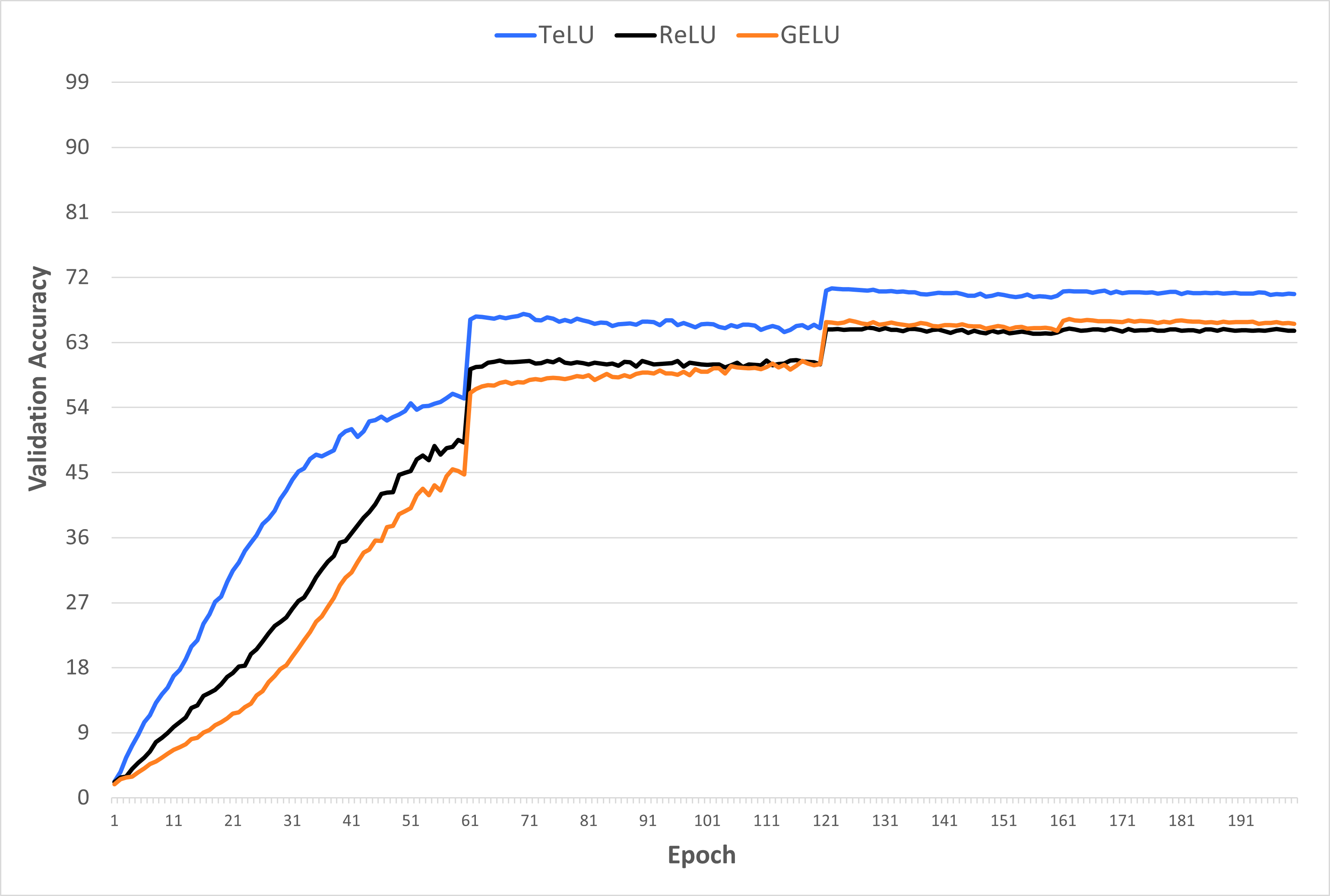}
    \caption[SqueezeNext on CIFAR-100 with Momentum-Accelerated SGD Validation Accuracy.]
    {\textbf{SqueezeNext on CIFAR-100 with Momentum-Accelerated SGD Validation Accuracy.} Shows the progression of validation accuracy throughout 200 epochs of the SqueezeNext architecture being trained on the CIFAR-100 dataset with the momentum-accelerated Mini-batch SGD optimizer. Each data point has been averaged over 5 trials.}
    \label{fig:MomentumC100SqzNxt}
    \vspace{1.0\baselineskip}
\end{figure}

We observe from Table \ref{tab:CIFAR100SUM} that TeLU exhibits greater accuracy than both ReLU and GELU accuracies in all cases. Additionally, we notice that TeLU exhibits greater learning stability in most cases, as determined by its minimal standard deviation of test accuracy across the Mini-batch SGD, momentum-accelerated SGD, and AdamW optimizers. This means that the network is less sensitive to fluctuations during training, leading to more predictable and stable outcomes, which are crucial in real-world applications where reliability is paramount.

\subsubsection{Tiny ImageNet Dataset with ResNet34 with Varied Optimizers}

We extend our experimentation to include a ResNet34 architecture on TinyImageNet to ensure that this trend persists for larger architectures and datasets. We utilize the configuration detailed in Table \ref{tab:ResNet34TinyHyps} on both Mini-batch SGD and momentum-accelerated SGD optimizers. Across optimizers, we vary only the base learning rate, weight decay coefficient, and learning rate scheduler decay hyperparameters as detailed in Table \ref{tab:tiny_sup_r34_hps}. We perform each trial by training the respective ResNet34 architecture on the TinyImageNet training partition, consisting of 100,000 images, over the course of 200 epochs. After each epoch, we see if the validation accuracy that the updated model exhibits is an improvement over the current validation accuracy. If the previous epoch of training has improved validation accuracy, we save a new checkpoint of our model. After training concludes, we evaluate our best checkpoint on the unseen testing partition to determine our top-1 and top-5 testing accuracies.

We record our test results across 5 trials for the Mini-batch SGD optimizer in Table \ref{tab:tiny_sup_r34_sgd} and observe that TeLU offers greater average test accuracy and less variance in top-1 test scores. For another perspective on these accuracy improvements, we plot the progression of validation accuracy across the 200 epochs of training in Figure \ref{fig:TinySGDVal}. Here, we observe that the TeLU architecture exhibits faster convergence speeds after the learning rate updates at epochs 60 and 100. After this, both the ReLU and TeLU architectures experience a saturation in validation accuracies, with TeLU architectures ending up at a greater validation accuracy.

\begin{table}[]
    \centering
    \caption[ResNet34 TinyImageNet Optimizer-Independent Hyperparameters.]
    {\textbf{ResNet34 TinyImageNet Optimizer-Independent Hyperparameters.} Static configuration of stability experiments ran on the ResNet34 architecture on the TinyImageNet dataset across different optimizers. }
    \label{tab:ResNet34TinyHyps}
    \begin{tabular}{||c c||} 
 \hline
 Hyperparameter & Value \\
 \hline\hline
 Train/Val/Test Split & 100,000 / 10,000 / 10,000\\ 
 \hline
 Normalization & Standard Score, detailed in Table \ref{tab:NormalizationDetails} \\ 
 \hline
 Architecture & ResNet34\\ 
 \hline
 Initialization & Xavier Uniform\\ 
 \hline
 Batch Size & 128\\
 \hline
 Learning Scheduler Updates & 60, 100, 140, 170\\
 \hline
 Data Augmentations & Random 4-padded cropping and horizontal flipping\\
 \hline
 Number of Epochs & 200\\ 
 \hline
 Number of Trials & 10\\ 
 \hline
\end{tabular}
\end{table}

\begin{table}[]
    \centering
    \caption[ResNet34 TinyImageNet Optimizer-Dependent Hyperparameters.]
    {\textbf{ResNet34 TinyImageNet Optimizer-Dependent Hyperparameters.} Learning rate, weight decay coefficient, and learning rate scheduler decay hyperparameter values which vary across optimizer choice.}
    \begin{tabular}{||c c c c||} 
 \hline
 Optimizer & learning rate & weight decay & gamma\\
 \hline\hline
 SGD & 0.05 & 0.001 & 0.3\\ 
 \hline
 Momentum & 0.04 & 0.0004 & 0.4\\ 
 \hline
 AdamW & 0.0005 & 0.004 & 0.5\\ 
 \hline
 RMSprop & 0.0001 & 0.0002 & 0.6\\  
 \hline
\end{tabular}
\label{tab:tiny_sup_r34_hps}
\end{table}

\begin{table}[]
\vspace{1.0\baselineskip}
    \centering
    \caption[ResNet34 TinyImageNet SGD Test Accuracy.]
    {\textbf{ResNet34 TinyImageNet SGD Test Accuracy.} Test accuracies of TeLU and ReLU ResNet34 architectures averaged across 10 trials at the state at which they achieved optimal validation accuracy when training on TinyImageNet dataset with a Mini-batch SGD optimizer.}
    \label{tab:tiny_sup_r34_sgd}
    \begin{tabular}{||c c c||} 
 \hline
 Name & Top-1 Test & Top-5 Test\\
 \hline\hline
 TeLU & \textbf{62.34}$\pm$0.173 & \textbf{81.86}$\pm$0.337\\ 
 \hline
 ReLU & 61.16$\pm$0.314 & 80.51$\pm$0.263\\ 
 \hline
\end{tabular}
\end{table}

\begin{table}[]
    \centering
    \caption[ResNet34 TinyImageNet Momentum-Accelerated SGD Test Accuracy Summary.]
    {\textbf{ResNet34 TinyImageNet Momentum-Accelerated SGD Test Accuracy Summary.} Test accuracies of TeLU and ReLU ResNet34 architectures averaged across 10 trials at the state at which they achieved optimal validation accuracy when training on TinyImageNet dataset with a momentum-accelerated Mini-batch SGD optimizer.}
    \label{tab:tiny_sup_r34_mom}
    \begin{tabular}{||c c c||} 
 \hline
 Name & Top-1 Test & Top-5 Test\\
 \hline\hline
 TeLU & \textbf{62.09}$\pm$0.222 & \textbf{82.28}$\pm$0.453\\ 
 \hline
 ReLU & 38.37$\pm$34.6 & 50.32$\pm$43.8\\ 
 \hline
\end{tabular}
\vspace{1.0\baselineskip}
\end{table}

 \begin{figure}
     \centering
     \includegraphics[width=0.75\linewidth]{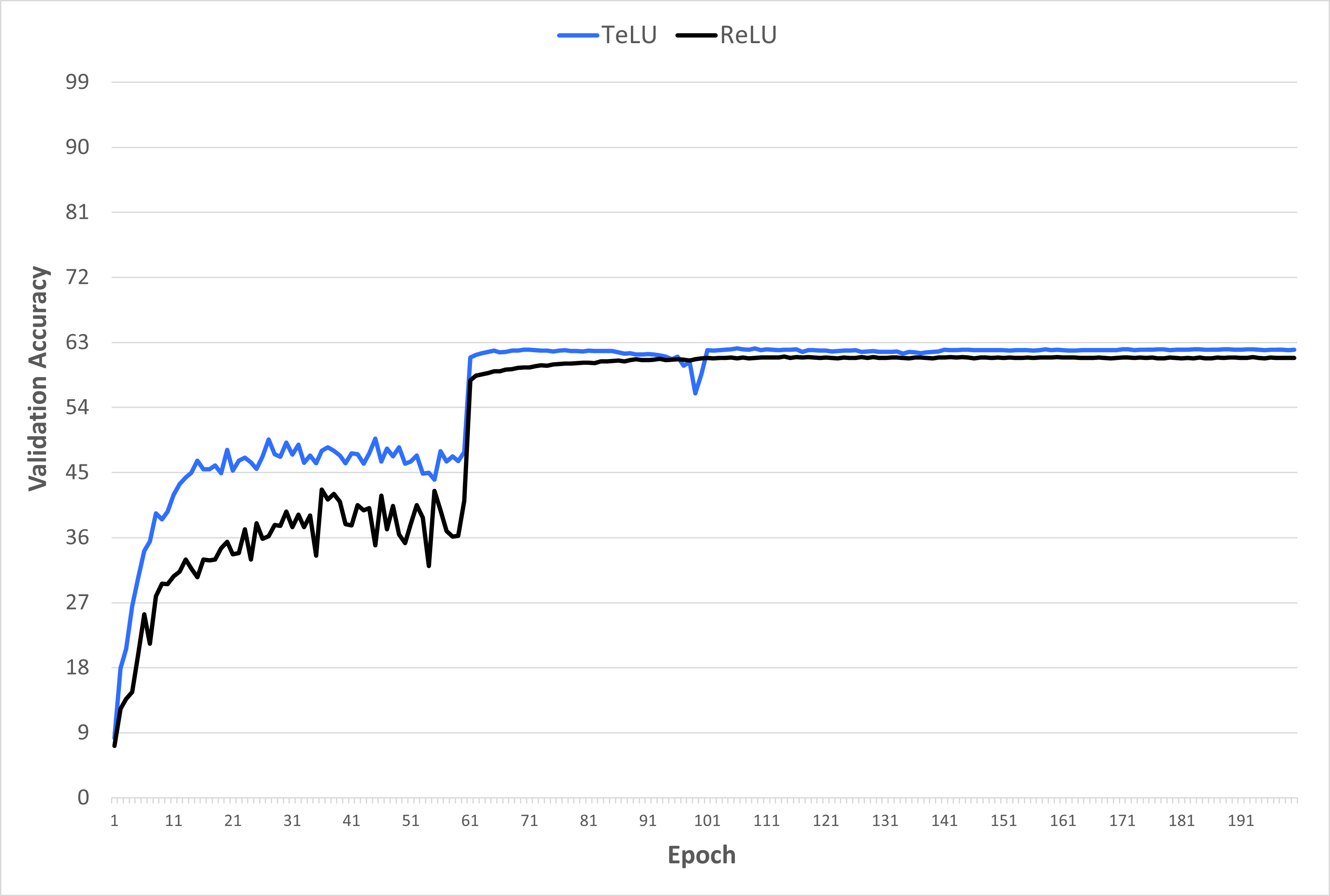}
     \caption[ResNet34 TinyImageNet SGD Validation Accuracy.]
     {\textbf{ResNet34 TinyImageNet SGD Validation Accuracy.} Progression of validation accuracies of TeLU and ReLU ResNet34 architectures averaged across 10 trials when training on TinyImageNet dataset with a Mini-batch SGD optimizer.}
     \label{fig:TinySGDVal}
     \vspace{1.0\baselineskip}
 \end{figure}

 \begin{figure}
     \centering
     \includegraphics[width=0.75\linewidth]{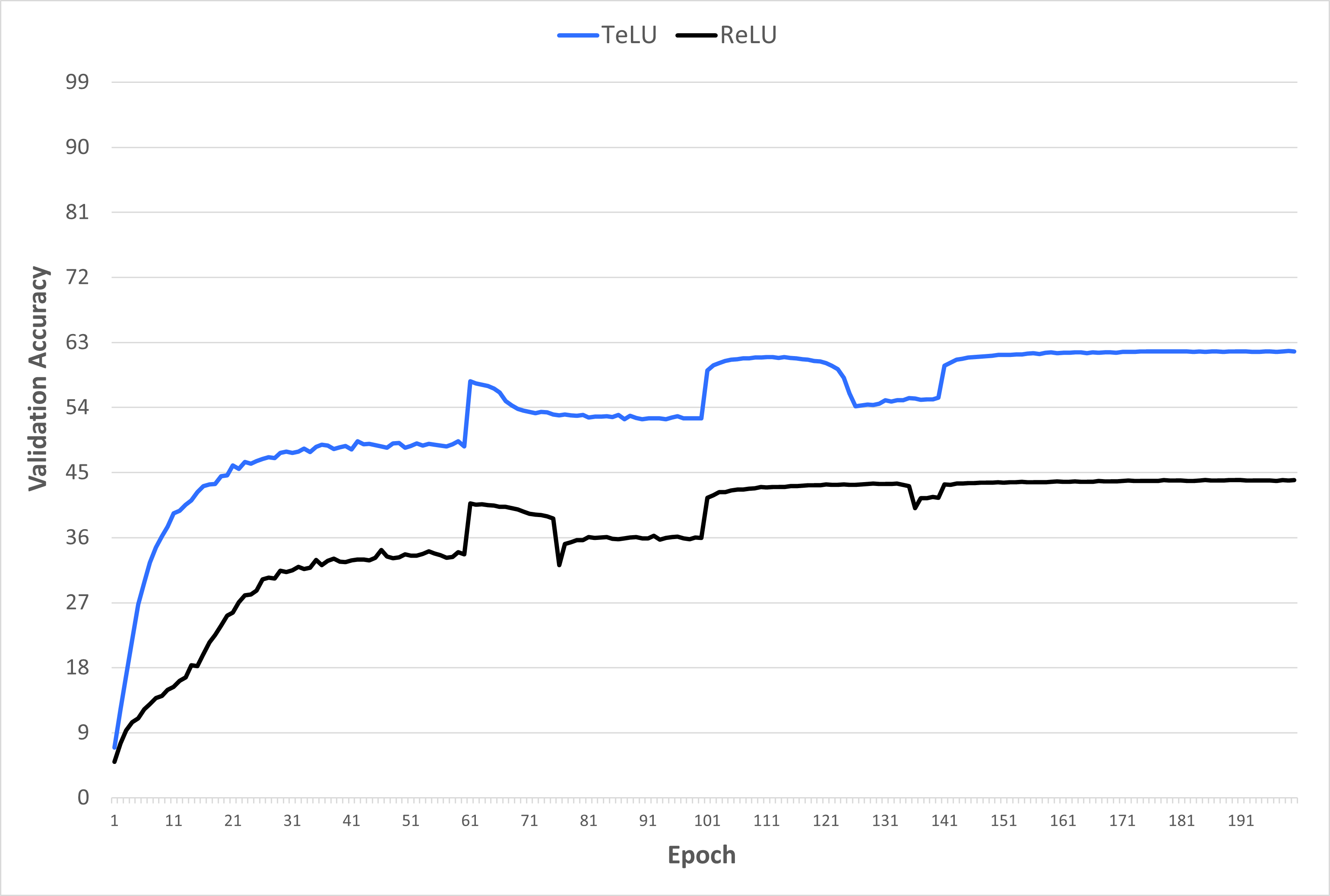}
     \caption[ResNet34 TinyImageNet Momentum-Accelerated SGD Validation Accuracy.]
     {\textbf{ResNet34 TinyImageNet Momentum-Accelerated SGD Validation Accuracy.}Progression of validation accuracies of TeLU and ReLU ResNet34 architectures averaged over 10 trials when training on TinyImageNet dataset with a momentum-accelerated Mini-batch SGD optimizer}
     \label{fig:TinyMomVal}
 \end{figure}

 We also summarize the test results across 10 trials for the momentum-accelerated Mini-batch SGD optimizer in Table \ref{tab:tiny_sup_r34_mom}. It is evident that the ReLU-based ResNet34 architecture experiences significant instability in both top-1 and top-5 accuracy metrics due to numerical issues. Notably, 3 out of the 10 trials resulted in 'not a number' (NaN) training accuracies across all 200 epochs, indicating that numerical instability arises as early as the first epoch. When training accuracy is compromised by such instability, the neurons transmitting it dictate the output, leading to random validation accuracies, where the network effectively guesses output values. This behavior is further illustrated in Figure \ref{fig:TinyMomVal}, where we observe that the ReLU ResNet34 architecture underperforms by approximately 20\% in average validation accuracy compared to the TeLU model. These findings confirm that TeLU-based architectures maintain stability across different optimizers.





\subsubsection{Summary of Stability Experimental Validation}

Throughout previous subsections of this section, we have observed similar cases of instability with other activation functions. Figure \ref{fig:ReLULikenessOverall} shows that SiLU, GELU, Mish, Logish, and Smish experience significantly more instability than other nonlinearities when employed within MLP architectures when the weight decay coefficient is increased beyond 0.001. Figure \ref{fig:ReLULikenessFocus} adds further evidence to this observation, with SiLU, GELU, Mish, Logish, and Smish architectures resulting in inferior testing accuracies after training concludes. Furthermore, we observe in Figure \ref{fig:SGDSqz10} that SiLU, Logish, and Smish SqueezeNext architectures result in reduced test accuracies when trained on the CIFAR-10 dataset with a Mini-batch SGD optimizer. This drop in accuracy is amplified for Logish and Smish when momentum is introduced in \ref{fig:MomentumSqz10}. Across all experiments, TeLU experiences the most stable and consistent results.

The reasons for TeLU's unique degree of learning stability lie in its ability to mitigate both vanishing and exploding gradients while maintaining smoothness and minimal output bias. This combination ensures that gradients propagate effectively through deep networks, allowing for steady and controlled learning, even in complex architectures. TeLU's ability to deliver high accuracy while ensuring stable learning makes it particularly well-suited for tasks requiring long-term reliability, such as in autonomous systems, financial forecasting, or critical healthcare diagnostics, where instability in learning could result in severe consequences. By delivering consistently better accuracy and reduced variability, TeLU provides a foundation for more resilient, generalizable models across a wide range of applications.

\newpage
\section{Discussion}


We have proposed the Hyperbolic Tangent Exponential Unit (TeLU) an activation function for deep neural networks defined as $TeLU(x)= x \cdot tanh(e^x)$. TeLU is an analytic, non-monotonic non-linearity that belongs to the class of linearly-scaled cumulative distribution functions, which consists of linear units such as $ReLU(x)=max(0,x)$ and $GELU(x)=\frac{x}{2} \cdot (1 + erf(\frac{x}{\sqrt{2}}))$. Like other linear units, TeLU may be interpreted as the identity function $x$ with an implicit activation dropout as input becomes negative. This results in TeLU comprising a dense deactivation region followed by an active region exhibiting asymptotic identity growth. Throughout this paper, we have demonstrate TeLU's persistent gradients of its deactivation region, near-linearity of its active region, computational efficiency, compatibility as a substitute for ReLU, analytic universal approximation, and inherent stability properties. Together, these properties allow TeLU to uniquely overcome vanishing gradient, convergence delay, computational delay, configuration tuning delay, community adoptability, and learning instability challenges.


\subsection{Persistent Gradients of Deactivation Region of TeLU}

In Subsection \ref{subsect:persistentgradtheory}, we noticed that our lower-bounded activation functions had derivatives that saturate towards deactivation at different rates. To better describe these varying vanishing rates, we extended the asymptotic growth classes to encompass asymptotic decay with classes $\Theta(\frac{x}{e^x})$, $\Theta(\frac{1}{e^x})$, $O(\frac{1}{x!})$, and $\Omega(\frac{1}{(x^2)!})$. We discovered that TeLU's gradient vanished according to $\Theta(\frac{x}{e^x})$, leading to minimal vanishing gradient concerns among competing functions. We showcased the effect of these varying vanishing gradients numerically by expressing the domain that avoids numerical underflow in each activation function. 

In Subsection \ref{subsect:persistentgradexp}, we showed the direct impact of these asymptotic decays, we designed an experiment where MLP architectures employing different hidden non-linearities have biases initialized to -10 and -20. We trained each architecture on the FashionMNIST dataset and observed each model's ability to recover from the vanishing gradient conditions found at their saturation regions. In each case, we observed that TeLU emerges first. Remaining architectures lag behind TeLU, in an order that is expected given their asymptotic decay class and in a manner that is reflective of their gradient strengths. 

We proceed to provide further experimental validations to the impact of persistent gradients by comparing TeLU, ReLU, and GELU DenseNet architectures on CIFAR-10. DenseNets do not feature any residual bypass connections, so vanishing or dying gradients in GELU and ReLU architectures are especially effective in halting learning prematurely, resulting in models that exhibit poor accuracy. Meanwhile, the TeLU architecture trains and generalizes to improved accuracy in all experiments, showing that TeLU mitigates the vanishing gradient problem.

\subsection{Near-Linearity of Active Region of TeLU}

In Subsection \ref{subsect:nearlinsubsection}, we focused on TeLU's near-linear behavior of the active region, featuring strong gradients that are essential for meaningful learnable parameter update steps. This undiminished learning allows for rapid, stable convergence towards superior training accuracy and generalizes to superior test accuracies, as ensured by TeLU's implicit dropout. In Subsection \ref{subsect:nearlinexp}, we provide experimental validation for the rapid, stable convergence of TeLU. We started by comparing TeLU and ReLU ResNet34 architectures on the ImageNet dataset. While TeLU smoothly approaches linear growth rapidly, ReLU exhbits linear growth immediately once its input become positive. With both TeLU and ReLU activation functions exhibiting strong active gradients, we compare the speed and stability of both architectures. 

First, we utilized Pytorch's ResNet34 default training configuration, which encompassed 90 epochs of training. We observed that the TeLU activation allowed for faster and more consistent convergence than ReLU. To accentuate on the benefits of this improved convergence, we limited our number of epochs to 50 and 20 epochs in two separate follow-up experiments. In either case, we observe that TeLU architectures provide improved convergence and superior generalization over ReLU architectures on ImageNet. 

We further demonstrate TeLU's convergence advantages over ReLU and GELU within a dynamic-pooling transformer architecture trained on the Text8 dataset. Initially, we highlighted TeLU's superior convergence speed and stability compared to GELU in terms of validation loss, with this improvement translating to lower test loss and reduced standard deviation. To examine the impact of TeLU's faster convergence, we conducted a shorter experiment with a reduced learning rate scheduler period. As expected, TeLU outpaced both ReLU and GELU in terms of convergence speed and stability under these conditions.

% ...


\subsection{Computational Efficiency of TeLU}

In Subsection \ref{subsect:runtimeefficiencytheory}, we quantified the computational complexities of baseline definitions of activation functions. We counted the number of piecewise segmentations, nonlinear functions, arithmetic operations, and constant terms present within each activation function. We discovered that TeLU exhibited minimal computational complexity that was second only to ReLU. Table \ref{table:ComputationComplexity} summarizes these counts for our focused group of non-linearities. 

In Subsection \ref{subsect:runtimeefficiencyexp}, we witness the direct computational benefits on TeLU over other smooth activation functions within custom Pytorch \cite{Pytorch} benchmarks run on various systems. Each benchmark consisted of a number of forward and backward passes on neurons that activated according to the TeLU, ReLU, ELU, SiLU, GELU, Mish, Logish, and Smish nonlinearities. The first of these benchmarks was run on a Windows 11 operating system with an NVIDIA RTX 2070 GPU, and consisted of $10^6$ activations on an input size of $10^6$, each followed by the corresponding gradient calculation. The total amount of time spent on forward and backward passes was calculated for each non-linearity. 

TeLU and ReLU were found to offer the most computational efficiency, as expected from the computational complexity heuristics. The next benchmarks were run on a LINUX batch server and utilized A100 \cite{A100} and 1080Ti GPUs, respectively. Tables \ref{tab:RunTimesGPU}, \ref{tab:RunTimesGPU10Miter}, and \ref{tab:RunTimesGPU10Minput} summarize the resulting runtime necessary across varying sizes of input and number of iterations. Across each operating system, device, and experimental configuration; we observe TeLU offering optimal run-times that are second only to ReLU.

\subsection{Configuration Compatability of ReLU and TeLU}

In Subsection \ref{subsect:relucomptheory}, we describe how non-monotonic non-linearities are the most suitable for approximating the ReLU activation function. We begin by calculating the area between different activation functions and ReLU, showing an initial heuristic for the closeness of approximation each function provides. This is performed for both negative and positive inputs, to indicate the quality of approximation in both the inactive and active regions. From our candidate pool, we demonstrated that non-monotonic functions are uniquely capable of approximating ReLU as input $x \to \pm \infty$ and at $x=0$. Additionally, we found that TeLU serves as the most suitable substitute for ReLU in deep learning applications. TeLU's active gradients closely approximate ReLU's strong identity growth, resulting in a comparable convergence speed and low implicit regularization. In contrast, other activation functions with weaker active gradients experience implicit regularization in the form of gradient damping, a behavior that distinguishes them from ReLU and TeLU.

In Subsection \ref{subsect:relucompexp}, we experimentally validated the operational similarities between TeLU and ReLU by testing MLP architectures with different non-linearities as their hidden activation functions. The results showed that the same weight decay coefficient values led to optimal performance in both TeLU and ReLU architectures, whereas other architectures with different activation functions did not achieve similar results. Furthermore, we show that TeLU provides accuracy improvements over ReLU on the same configuration that optimizes ReLU convergence and generalization. We offer additional experimental evidence of this similarity by demonstrating that TeLU improves accuracy in training configurations optimized to minimize the loss in ReLU CNN architectures. For each optimizer, including Mini-batch SGD \cite{minibatchgradientdescent, SGD1}, momentum accelerated SGD \cite{Momentum}, AdamW \cite{AdamW}, and RMSprop \cite{RMSprop}, we tune the training hyperparameters that influence the learning of a SqueezeNext architecture using the ReLU nonlinearity. The hyperparameters are optimized to achieve the highest possible validation accuracy for each optimizer. Again, we observe that TeLU architectures are optimized by configurations designed for ReLU. In addition, the results demonstrate that TeLU outperforms ReLU when using Mini-batch momentum-accelerated SGD, AdamW, and RMSprop optimizers.

\subsection{Analytic Universal Approximation of TeLU}

In Subsection \ref{subsect:AUAtheory}, we establish that TeLU is an analytic universal approximator. Consequently, architectures utilizing TeLU as their activation function exhibit stable and efficient convergence when optimized with gradient-based methods. The smooth nature of analytic approximators, like TeLU, also enhances generalization to unseen data by minimizing overfitting to noise or irrelevant artifacts \cite{importantSmooth}. This improved robustness leads to more reliable model performance. Moreover, the smooth approximations provided by analytic functions are easier to interpret and analyze, a quality highly valued in mathematical research \cite{analytic}. Additionally, as an analytic function, TeLU ensures compatibility with second-order optimization techniques that leverage the curvature of the loss landscape, resulting in more stable and efficient learning \cite{backprop, SecondOrder1, SecondORder2, SecondOrder3}.

In Subsection \ref{subsect:AUAexp}, we perform practical demonstrations of the benefits of utilizing analytic nonlinearities with VAE \cite{VAE}, RNN \cite{ELMAN1990179, LSTM}, and MLP robustness \cite{robustnessseminal} experiments. Results show that VAE architectures that utilize the TeLU activation function lead to MNIST sample reconstructions with minimal loss and improved consistency over that of ReLU VAEs. For our RNN experiments, we define Elman \cite{ELMAN1990179} and LSTM \cite{LSTM} architectures that employ either TeLU, ReLU, or Tanh activation functions. Again, we find that TeLU architectures lead to minimal perplexity \cite{perplexity} on the Penn TreeBank dataset \cite{PTBDataSet}. Lastly, we demonstrate the improved robustness of analytic universal approximations by comparing the robustness of TeLU MLPs with that of ReLU and ELU models. As expected, we observe that TeLU consistently outperforms its piecewise competitors \cite{xie2021smoothadversarialtraining, importantSmooth}. Across all experiments, we notice that TeLU offers a strong advantage over non-analytic activation functions.

\subsection{Learning and Numerical Stability of TeLU Networks}

In Subsection \ref{subsect:stabilitytheory}, we examined the properties of activation functions that contribute to the learning stability of deep neural network models. Table \ref{tab:StabilityTable} provides a summary of key stability factors: This comprises the depth of embedded nonlinear computations, smoothness, susceptibility to the vanishing gradient problem, and output bias for each activation function studied. We demonstrate that smooth linear units like TeLU effectively mitigate the exploding gradient issue through their sub-linear growth in the positive domain. Our findings show that TeLU achieves the best overall balance across these criteria, resulting in notable stability improvements compared to existing activation functions.

In \ref{subsect:stabilityexp}, we provide experimental validation of our heuristic comparison by demonstrating that TeLU MLP architectures maintain stable performance as model depth increases across multiple configurations. Furthermore, we show that the learning stability of TeLU persists across various weight initialization methods. Specifically, using Xavier Uniform, Xavier Normal \cite{glorot_init}, Kaiming Uniform, and Kaiming Normal \cite{HeInitializationDBLP:journals/corr/HeZR015} initialization techniques, TeLU consistently outperforms similar MLPs employing other activation functions. This consistent performance extends to our CNN experiments as well. TeLU-based SqueezeNext architectures achieve superior test accuracies across all tested optimizers when trained on the CIFAR-100 dataset. Across all experiments, we find that TeLU architectures consistently exhibit faster convergence and greater learning stability, highlighting the unique advantages of TeLU in deep learning applications.

\subsection{Independent Rediscovery}

After submitting our work on the Hyperbolic Tangent Exponential Linear Unit (TeLU) activation function, we discovered that a similar formulation had been introduced in the literature under the name TanhExp \cite{tanhexp}. Our independent discovery, stemming from a theoretical study of activation functions initiated in late 2022, underscores the natural emergence of TeLU as an innovative design. By systematically analyzing activation functions and building on insights from prior work, including Mish \cite{Mish}, we identified TeLU’s distinctive properties. This parallel rediscovery highlights the intuitive appeal of TeLU’s design and its potential to address key challenges in machine learning applications. While both studies demonstrate the efficacy of this activation function, they approach the problem from complementary perspectives. The TanhExp study emphasizes empirical results, particularly on small-scale vision benchmarks like CIFAR-10 and Fashion-MNIST, providing valuable insights into its practical utility. In contrast, our work integrates theoretical rigor with extensive empirical validation, offering a deeper understanding of TeLU’s behavior and properties.

Our study expands on the existing literature by deriving and validating theoretical bounds, which explain why TeLU works effectively and establish it as a reliable drop-in replacement for ReLU. Additionally, we significantly broaden the scope of experimentation, testing TeLU on large-scale benchmarks such as ImageNet and Text8 to evaluate its versatility across diverse tasks. This comprehensive evaluation demonstrates TeLU’s applicability beyond vision tasks and bridges the gap between theory and practice, showcasing its robustness in real-world scenarios. To ensure reliability and reproducibility, our methodology incorporates multiple trials, reports standard deviations, and employs separate testing and validation sets, adhering to best practices in machine learning experimentation. By providing detailed experimental settings and making our code publicly available, we aim to promote transparency and encourage further exploration of TeLU. These contributions complement prior work by addressing aspects such as statistical significance and scalability, reinforcing the utility of this activation function. Through systematic comparisons against a wide range of activation functions across diverse datasets and experimental setups, we offer a balanced and thorough analysis. This combination of theoretical insights and empirical breadth underscores TeLU’s potential as a lightweight, efficient, and effective alternative to ReLU, paving the way for broader adoption and future advancements in activation function design.

\newpage
\section{Conclusion}

In conclusion, the Hyperbolic Tangent Exponential Linear Unit (TeLU) stands out as a highly effective activation function that addresses several critical challenges in neural network training. One of its primary advantages is the presence of persistent gradients in its inactive region, which mitigates the vanishing gradient problem that can slow down learning in deep networks. By maintaining non-zero gradients even for negative input values, TeLU ensures that all neurons continue to learn, enhancing overall network performance.

Moreover, TeLU closely approximates the identity function for positive input values. This characteristic strikes an ideal balance, preventing both vanishing and exploding gradient issues. The linear approximation allows for consistent and efficient gradient propagation, leading to faster and more stable convergence during training. By minimizing unintended damping of gradients, TeLU ensures that learnable parameter updates are not inadvertently downscaled. This allows optimizers and learning rate schedulers to effectively manage the magnitude of learning steps, maximizing convergence rates and enabling a more modular and flexible training configuration.


Unlike more complex activation functions, TeLU's simple formulation reduces computational overhead, resulting in significant efficiency improvements. TeLU's straightforward design also offers seamless compatibility as a drop-in substitute for the widely used ReLU activation function. This ease of integration encourages adoption within the deep learning community, allowing practitioners to leverage TeLU's benefits without overhauling existing architectures. Unlike ReLU, TeLU is an analytic function, which enhances robustness and stability during training. Its analytic nature not only improves convergence but also makes it compatible with second-order optimization methods. By utilizing the curvature of the loss function, these methods can further enhance convergence efficiency and stability.

These combined properties enable TeLU to exhibit a unique level of learning stability across a wide range of experimental settings. By effectively addressing issues like vanishing gradients, computational inefficiency, and training instability, TeLU presents a compelling advancement in activation functions. Its ability to simplify training configurations while enhancing performance makes it a valuable tool for advancing deep learning models across various applications. The TeLU activation function, therefore, holds significant promise for facilitating more efficient, stable, and robust neural network training in the field of deep learning.

\newpage
\section{Future Work}






For future work, we aim to delve into the unique properties of analytic universal approximators by thoroughly examining the convergence guarantees provided by the TeLU activation function. A deeper theoretical investigation could identify the precise conditions under which TeLU guarantees faster and more reliable convergence, broadening its utility across different neural network architectures and aiding researchers in optimizing model performance more efficiently. 

TeLU's role as an analytic universal approximator makes it a promising candidate for higher-order optimization methods. In future work, we will explore its performance with techniques like Newton's Method \cite{newton1, newton2, newton3}, Hessian-Free Optimization \cite{SecondOrder1}, Natural Gradient Descent \cite{Amari1998, shrestha2023naturalgradientmethodsperspectives}, and Trust Region methods \cite{trust1, trust2, trust3}. These methods offer improved stability and convergence efficiency, so we will investigate how TeLU complements their properties.

In addition, we will extend our research by integrating TeLU with bio-inspired learning algorithms, such as predictive coding networks \cite{PC1}, which leverage higher-order curvature information through quasi-Newton approximations \cite{malipaper}. By incorporating TeLU, we anticipate enhanced stability and faster convergence compared to traditional backpropagation, and our future studies will carefully analyze its performance within these innovative frameworks.

TeLU, as an analytic universal approximator, has a smooth curvature often associated with improved robustness in neural networks \cite{importantSmooth, SecondOrder1, SecondOrder3}. Initial testing has demonstrated promising robustness benefits. To further validate these results, we plan to conduct extensive experiments across challenging datasets and standard benchmarks \cite{robustnessDatasethendrycks2019benchmarkingneuralnetworkrobustness}. This allows smooth activation functions like TeLU to benefit against adversarial attacks \cite{xie2021smoothadversarialtraining}.

Beyond providing resistance to adversarial attacks, we hypothesize that smooth non-monotonic activation functions enhance security by making it more difficult to reconstruct training images, addressing significant privacy concerns \cite{reconstruct}. Their unpredictability disrupts neural network processing patterns, hindering attackers from reverse-engineering or inferring sensitive data. Attack methods like side-channel analysis and bitstream reverse engineering, which rely on predictable behaviors and data correlations, become less effective. This unpredictability strengthens neural network security by reducing the success of traditional attack vectors. Integrating non-monotonic activation functions thus enhances privacy and bolsters resistance against specific attacks.




Inspired by the success of TeLU, we also intend to experiment with variations of TeLU of the form $x \cdot tanh(a^x)$, where $a \in [2,3]$. If $a \in \mathbb{Z}^+$, computational efficiency could be enhanced. When $a + \epsilon = e$, for some small positive $\epsilon$, gradients would decay at an asymptotic rate of $\Theta(\frac{x}{a^x})$, which may help further address the vanishing gradient problem and improve convergence. Conversely, when $a = e + \epsilon$, the function might approximate the identity more closely, potentially enhancing convergence.

To further investigate the computational efficiency of TeLU, we plan to conduct additional experiments using lower-level programming languages like C and C+. Implementing TeLU in these languages may allow for a more precide evaluation of its computational efficiency. We also intend to explore optimizations of TeLU itself, refining its algorithms to improve speed and resource utilization. Additionally, by involving distinct hardware accelerators such as TPUs \cite{TPUDeepLearning} and FPGAs \cite{FPGADeepLearning}, we aim to assess TeLU's performance across various architectures. These efforts are motivated by our desire to maximize TeLU's efficiency and scalability, ensuring it can be effectively utilized in a wide range of real-world applications.

\newpage
\bibliographystyle{ieeetr}
\bibliography{main_arxiv}

\appendix
\onecolumn

\section{Additional Results}

This Appendix section provides supplementary tables that offer additional insights into our theoretical analysis and experimentation. These tables serve as an additional resource for readers seeking a deeper understanding of the theoretical analysis and methodologies discussed in the main body of this work.

\subsection{Computational Efficiency of Derivatives}

Table \ref{table:DerivativeComplexity} presents a breakdown of the various mathematical components involved in computing the first derivative of each activation function examined in this study. Specifically, it quantifies the number of piece-wise segments, nonlinear operations, arithmetic computations, and constant terms required. This analysis serves as a practical measure of computational complexity, providing insights into how each activation function's derivative might impact processing time and resource utilization during model training. By comparing these counts, we can better understand the efficiency trade-offs between different activation functions, highlighting the advantages of the proposed function in terms of simplicity and speed.

\begin{table}[]
\centering
 \caption[Activation Function Derivative Computational Efficiency.]
 {\textbf{Activation Function Derivative Computational Efficiency.} First derivatives of activation functions computational complexity heuristics: Counting instances of unique non-linearity computations, arithmetic operations, and constant terms}
    \begin{tabular}{||c c c c c||} 
 \hline
 Function & Piecewise & Nonlinearity & Arithmetic & Constants \\
 \hline\hline
 TeLU & 0 & 3 & 4 & 0 \\ 
 \hline
 ReLU & 1 & 0 & 0 & 2 \\
 \hline
 ELU & 1 & 1 & 0 & 1\\
 \hline
 SiLU & 0 & 1 & 6 & 2\\
 \hline
 GELU & 0 & 2 & 8 & 7 \\
 \hline
 Mish & 0 & 3 & 16 & 5\\
 \hline
 Logish & 0 & 3 & 12 & 3\\
 \hline
 Smish & 0 & 3 & 16 & 10\\
 \hline
\end{tabular}
\label{table:DerivativeComplexity}
\end{table}

\subsection{Computation Efficiency of Derivatives}

Table \ref{tab:NormalizationDetails} outlines the standard score normalization applied in experiments with the CIFAR-10, CIFAR-100, and TinyImageNet datasets. For each dataset, the mean and standard deviation are calculated and utilized to normalize each input sample, ensuring consistency and comparability across the experiments. This information is provided to assist anyone attempting to recreate the experiments. By including the computed means and standard deviations for each dataset, we ensure that the standard score normalization process can be accurately replicated, maintaining the integrity and consistency of the experimental conditions. This level of detail is crucial for achieving comparable results and validating the findings of this study.


\begin{table}[]
    \centering
    \caption[Normalization Experimental Specifications.]
    {\textbf{Normalization Experimental Specifications.} Mean and standard deviation normalization specifications for CIFAR-10, CIFAR-100, and TinyImageNet datasets.}
    \begin{tabular}{||c c c c||} 
 \hline
 Optimizer & Red & Green & Blue\\
 \hline\hline
 CIFAR-10 & $\mu$=0.4914, $\sigma$=0.2023 & $\mu$=0.4822, $\sigma$=0.1994 & $\mu$=0.4465, $\sigma$=0.2010\\ 
 \hline
 CIFAR-100 & $\mu$=0.5071, $\sigma$=0.2675 & $\mu$=0.4867, $\sigma$=0.2565 & $\mu$=0.4408, $\sigma$=0.2761\\ 
 \hline
 TinyImageNet & $\mu$=0.485, $\sigma$=0.229 & $\mu$=0.456, $\sigma$=0.224 & $\mu$=0.406, $\sigma$=0.225\\ 
 \hline
\end{tabular}
    \label{tab:NormalizationDetails}
\end{table}



\newpage
\section{Additional Supporting Theoretical Results}
\label{sect:supplemtheorems}

This appendix section presents supplementary theorems that provide a more detailed and rigorous foundation for the theories proposed in this study. These additional theorems serve to strengthen the theoretical framework, offering deeper insights and comprehensive support for our arguments. By including these formal mathematical statements and proofs, we aim to enhance the clarity and robustness of our theoretical analysis, giving readers a more complete understanding of the underlying principles that drive our findings.

\subsection{Close Approximation to ReLU Non-Linearity with TeLU}


\begin{Lem}
Let \( r(x) = \text{ReLU}(x) = \max(0, x) \), and define the active subdomain as \( \mathcal{A} = [0, \infty) \) and the inactive subdomain as \( \mathcal{I} = (-\infty, 0) \). Consider the following activation functions:

1. TeLU: \( t(x) = x \cdot \tanh(e^x) \),

2. GELU: \( g(x) = x \cdot \Phi(x) \), where \( \Phi(x) \) is the Gaussian CDF.

Define the gradient magnitudes over the active and inactive subdomains as:

\[
G_{\mathcal{A}}(f) = \int_{0}^{\infty} \left| f'(x) \right| \, dx, \quad G_{\mathcal{I}}(f) = \int_{-\infty}^{0} \left| f'(x) \right| \, dx.
\]

If \( G_{\mathcal{A}}(t) > G_{\mathcal{I}}(g) \), then TeLU has a stronger impact on training dynamics in the positive subdomain than GELU has in the negative subdomain. Consequently, neural networks utilizing TeLU more closely resemble ReLU and exhibit stronger positive-side sensitivity.

\end{Lem}

\begin{proof}
To quantify the gradient behavior of \( t(x) \) and \( g(x) \) in their respective subdomains, we first compute the derivatives:

1. Derivative of TeLU in the Active Subdomain
The TeLU function is defined as:

\[
t(x) = x \cdot \tanh(e^x).
\]

Taking the derivative:

\[
t'(x) = \tanh(e^x) + x \cdot e^x \cdot \text{sech}^2(e^x).
\]

As \( x \to \infty \), \( \tanh(e^x) \approx 1 \) and \( \text{sech}^2(e^x) \approx 0 \). Thus, for large \( x \), we have:

\[
t'(x) \approx 1.
\]

This shows that TeLU’s derivative asymptotically approaches ReLU’s derivative for \( x > 0 \).

The gradient magnitude in the active subdomain is:

\[
G_{\mathcal{A}}(t) = \int_{0}^{\infty} \left| \tanh(e^x) + x \cdot e^x \cdot \text{sech}^2(e^x) \right| dx.
\]

Since \( \tanh(e^x) \approx 1 \) and the second term vanishes as \( x \) increases, the integral converges to a large value, indicating strong gradient influence in the active subdomain.

2. Derivative of GELU in the Inactive Subdomain
The GELU function is defined as:

\[
g(x) = x \cdot \Phi(x),
\]

where \( \Phi(x) = \frac{1}{2} \left( 1 + \text{erf} \left( \frac{x}{\sqrt{2}} \right) \right) \) is the Gaussian CDF. Taking the derivative:

\[
g'(x) = \Phi(x) + x \cdot \phi(x),
\]

where \( \phi(x) = \frac{1}{\sqrt{2\pi}} e^{-x^2/2} \) is the Gaussian PDF.

For \( x < 0 \), \( \Phi(x) \approx 0 \) and \( \phi(x) \) rapidly decays to 0 as \( x \to -\infty \). Thus, for large negative \( x \), we have:

\[
g'(x) \approx 0.
\]

The gradient magnitude in the inactive subdomain is:

\[
G_{\mathcal{I}}(g) = \int_{-\infty}^{0} \left| \Phi(x) + x \cdot \phi(x) \right| dx.
\]

Since \( \Phi(x) \approx 0 \) and \( \phi(x) \approx 0 \) for \( x < 0 \), this integral converges to a small value, indicating weak gradient influence in the inactive subdomain.

3. Gradient Magnitude Comparison
Comparing \( G_{\mathcal{A}}(t) \) and \( G_{\mathcal{I}}(g) \):

\[
G_{\mathcal{A}}(t) \approx \int_{0}^{\infty} 1 \, dx = \infty, \quad G_{\mathcal{I}}(g) \approx \int_{-\infty}^{0} 0 \, dx = 0.
\]

This shows that TeLU has a significantly larger gradient magnitude in the active subdomain compared to GELU’s gradient in the inactive subdomain.

Conclusion
Since the gradient strength of TeLU in \( \mathcal{A} \) far exceeds that of GELU in \( \mathcal{I} \), TeLU will have a stronger impact on the learning dynamics of a neural network. Therefore, neural networks using TeLU will be more sensitive and responsive in the positive region, closely mimicking ReLU’s behavior in practice.

\end{proof}


\subsection{Zero-Centering of TeLU} 

Before proving that TeLU has better zero-centering capability than ReLU, we will provide the following 
lemma to support our construction.
\begin{Lem}

If $f(x) \geq g(x)$ for all \(x \in \mathbb{R}\) and \(f(x) > g(x)\) for some defined interval \([a,b]\), where \(a, b \in \mathbb{R}\) and \(a<b\), then \(\int_{-\infty}^{\infty} f(x) dx > \int_{-\infty}^{\infty} g(x) dx\).

\end{Lem}

\begin{proof}

We can expand \(\int_{-\infty}^{\infty} f(x) dx\) as follows:

\[ \int_{-\infty}^{\infty} f(x) dx = \int_{-\infty}^{a} f(x) dx + \int_{a}^{b} f(x) dx + \int_{b}^{\infty} f(x) dx \]

Similarly, we can also expand \(\int_{-\infty}^{\infty} g(x) dx\) as follows:

\[ \int_{-\infty}^{\infty} g(x) dx = \int_{-\infty}^{a} g(x) dx + \int_{a}^{b} g(x) dx + \int_{b}^{\infty} g(x) dx \]

Next, we define a condition necessary to prove as follows:

\[\int_{-\infty}^{\infty} f(x) dx > \int_{-\infty}^{\infty} g(x) dx\]
\[\int_{-\infty}^{\infty} f(x) dx - \int_{-\infty}^{\infty} g(x) dx > 0\]
\[ \left( \int_{-\infty}^{a} f(x) dx - \int_{-\infty}^{a} g(x) dx \right) + \left( \int_{a}^{b} f(x) dx - \int_{a}^{b} g(x) dx \right) + \left( \int_{b}^{\infty} f(x) dx - \int_{b}^{\infty} g(x) dx \right) > 0\]

Let's analyze these three components :

For all \(x \in \mathbb{R}\) we get,

\[f(x) \geq g(x)\]
\[\int_{-\infty}^{a}f(x) \geq \int_{-\infty}^{a}g(x)\]
\[\int_{-\infty}^{a}f(x) - \int_{-\infty}^{a}g(x) \geq 0 \]

Similarly, for all \(x \in \mathbb{R}\),

\[f(x) \geq g(x)\]
\[\int_{b}^{\infty}f(x) \geq \int_{b}^{\infty}g(x)\]
\[\int_{b}^{\infty}f(x) - \int_{b}^{\infty}g(x) \geq 0 \]

Finally, for all \(x \in \mathbb{R}^+\) we get,

\[f(x) > g(x)\]
\[\int_{b}^{\infty}f(x) > \int_{b}^{\infty}g(x)\]
\[\int_{b}^{\infty}f(x) - \int_{b}^{\infty}g(x) > 0 \]
\[k > 0 \]

where \(k \in \mathbb{R}^+\).

These 3 components can be re-written as follows:

\[ \left( \int_{-\infty}^{a} f(x) dx - \int_{-\infty}^{a} g(x) dx \right) + \left( \int_{a}^{b} f(x) dx - \int_{a}^{b} g(x) dx \right) + \left( \int_{b}^{\infty} f(x) dx - \int_{b}^{\infty} g(x) dx \right)\]
\[ \geq 0 + \left( \int_{a}^{b} f(x) dx - \int_{a}^{b} g(x) dx \right) + 0 =\left( \int_{a}^{b} f(x) dx - \int_{a}^{b} g(x) dx \right)\]
\[= k > 0\]

because \(k \in \mathbb{R}^+\)

Therefore, we have shown that

\[\int_{-\infty}^{\infty} f(x) dx > \int_{-\infty}^{\infty} g(x) dx\]

for any \(f(x)\) and \(g(x)\) such that \(f(x) \geq g(x)\) for any \(x \in \mathbb{R} and f(x) > g(x)\) for any \(x \in [a,b]\), which concludes our proof.

\end{proof}



Next, we show that TeLU exhibits greater zero-centering of activation than ReLU.

\begin{Thm} 
\label{lowoutputbiastheorem}
If \( x \) is a random variable following a Gaussian probability distribution about zero with standard deviation \(\sigma \in \mathbb{R}^+\) expressed as PDF \(p(x) = \frac{1}{\sigma\sqrt{2\pi}} \cdot exp({\frac{-x^2}{2\sigma^2}})\), \(TeLU(x) = x \cdot tanh(e^x)\), \(ReLU(x) = max(0,x)\), \(E[TeLU(x)] = \int_{-\infty}^{\infty}p(x) \cdot TeLU(x)dx \), and \(E[ReLU(x)] = \int_{-\infty}^{\infty}p(x) \cdot ReLU(x)dx \); then $|E[TeLU(x)]|$ $< |E[ReLU(x)]|$.

\end{Thm}

\begin{proof}


We now analyze PDF \(p(x) = \frac{1}{\sigma\sqrt{2\pi}} \cdot exp({\frac{-x^2}{2\sigma^2}})\) to show that it is positive for all $\sigma \in \mathbb{R}^+$, $x \in \mathbb{R}$:

\begin{itemize}
    \item \( \frac{1}{\sigma\sqrt{2\pi}} \) is positive for all $\sigma \in \mathbb{R}^+$
    \item \( exp({\frac{-x^2}{2\sigma^2}}) \) is positive for all $\sigma \in \mathbb{R}^+$, $x \in \mathbb{R}$; because the range of \( exp(x) \) is \( (0,\infty) \)
    \item \( \therefore p(x) = \frac{1}{\sigma\sqrt{2\pi}} \cdot exp({\frac{-x^2}{2\sigma^2}})\), as it is the product of two positive terms
\end{itemize}

We now show that \(E[t(x)] < E[r(x)]\) for all $\sigma \in \mathbb{R}^+$, $x \in \mathbb{R}$:

We begin by showing following condition that \(p(x)\cdot TeLU(x) < p(x)\cdot ReLU(x)\) for all $\sigma \in \mathbb{R}^+$, $x \in \mathbb{R}^+$: 

\begin{align*}
  p(x)\cdot TeLU(x) < p(x)\cdot ReLU(x) \\
  TeLU(x) < ReLU(x) \\
  x \cdot tanh(e^x) < x \cdot\begin{cases} 0 & x < 0 \\1 & x\geq 0\end{cases}\\
  tanh(e^x) < \begin{cases} 0 & x < 0 \\1 & x\geq 0\end{cases}\\
  tanh(e^x) < 1\\
  \llap{\hspace{50pt}}
\end{align*}

Next we show that \(p(x)\cdot TeLU(x) = p(x)\cdot ReLU(x)\) when \(x=0\) for all $\sigma \in \mathbb{R}^+$

\begin{align*}
  p(x)\cdot TeLU(x) = p(x)\cdot ReLU(x) \\
  TeLU(x) = ReLU(x) \\
  TeLU(0) = ReLU(0) \\
   0 \cdot tanh(e^0) = 0 \cdot\begin{cases} 0 & x < 0 \\1 & x\geq 0\end{cases}\\
   0 = 0\\
  \llap{\hspace{50pt}}
\end{align*}

and next condition we show is that \(p(x)\cdot TeLU(x) < p(x)\cdot ReLU(x)\) for all $\sigma \in \mathbb{R}^+$, $x \in \mathbb{R}^+$:

\begin{align*}
  p(x)\cdot TeLU(x) < p(x)\cdot ReLU(x) \\
  TeLU(x) < ReLU(x) \\
  x \cdot tanh(e^x) < x \cdot\begin{cases} 0 & x < 0 \\1 & x\geq 0\end{cases}\\
  tanh(e^x) > \begin{cases} 0 & x < 0 \\1 & x\geq 0\end{cases}\\
  tanh(e^x) > 0\\
  \llap{\hspace{50pt}}
\end{align*}

Since we have shown that \(  p(x)\cdot TeLU(x) \leq p(x)\cdot ReLU(x) \) for all $x \in \mathbb{R}$ and \(  p(x)\cdot TeLU(x) \leq p(x)\cdot ReLU(x) \) for all $x \in \mathbb{R^+}$, support theorem tells us that \(\int_{-\infty}^{\infty}p(x) \cdot TeLU(x)dx < \int_{-\infty}^{\infty}p(x) \cdot ReLU(x)dx \). \(\therefore E[TeLU(x)] < E[ReLU(x)]\).

Now, we show that \(E[TeLU(x)]\) is positive for all $\sigma \in \mathbb{R}^+$:

We analyze TeLU(x):
\begin{itemize}
    \item \(e^x \) is positive for all \(x \in \mathbb{R}^+\)
    \item since \(tanh(w) \) is positive for all \(w \in \mathbb{R}\), \(tanh(e^x)\) is positive for all \(x \in \mathbb{R}\)
    \item $x$ multiplies with positive \(tanh(e^x)\), so \(sign(x) = sign(x \cdot tanh(e^x))\)
    \item \(p(x) \cdot TeLU(X) \) is positive for all \(x \in \mathbb{R}^+\), since p(x) is always positive
    \item \(p(x) \cdot TeLU(X) \) is 0 when \(x = 0\), since multiplying by x=0 leads to 0
    \item \(p(x) \cdot TeLU(X) \) is negative for all \(x \in \mathbb{R}^-\), since p(x) is always positive
\end{itemize}

For \(E[t(x)] = \int_{-\infty}^{\infty}p(x) \cdot TeLU(x)dx = \int_{-\infty}^{\infty}\frac{1}{\sigma\sqrt{2\pi}} \cdot exp({\frac{-x^2}{2\sigma^2}}) \cdot TeLU(x)dx \) to be positive for all $\sigma \in \mathbb{R}^+$, we must show that \(|p(-\epsilon)TeLU(-\epsilon)| < p(\epsilon)TeLU(\epsilon)\) for all \(\epsilon \in \mathbb{R}^+\). In other words, we must show that the positive component of \(p(x) \cdot TeLU(x)\) is always a strict upper bound to the absolute value of its negative component, evaluated symmetrically about \(x=0\). \(\epsilon = 0\) can be disregarded, as \(p(\epsilon) \cdot TeLU(\epsilon) = p(0) \cdot TeLU(0) = 0\).

Hence, we show that \(|p(-\epsilon) \cdot TeLU(-\epsilon)| < p(\epsilon) \cdot TeLU(\epsilon)\) for all \(\epsilon \in \mathbb{R}^+\):

\begin{align*}
   |p(-\epsilon) \cdot TeLU(-\epsilon)| < p(\epsilon) \cdot TeLU(\epsilon)\\
   -p(-\epsilon) \cdot TeLU(-\epsilon) < p(\epsilon) \cdot TeLU(\epsilon)\\
   -TeLU(-\epsilon) < TeLU(\epsilon)\\
   -(-\epsilon \cdot tanh(e^{-\epsilon})) < \epsilon \cdot tanh(e^{\epsilon})\\
   \epsilon \cdot tanh(e^{-\epsilon}) < \epsilon \cdot tanh(e^{\epsilon})\\
   tanh(e^{-\epsilon}) < tanh(e^{\epsilon})\\
   e^{-\epsilon} < e^{\epsilon}\\
   -\epsilon < \epsilon\\
  \llap{\hspace{50pt}}
\end{align*}

Which is true for all \(\epsilon \in \mathbb{R}^+\).

In summary, we have shown that \(0 < E[TeLU(x)] < E[ReLU(x)]\) for all \(x \in \mathbb{R}\). This implies that \(|E[TeLU(x)]| < |E[ReLU(x)]|\) for all \(x \in \mathbb{R}\). showing that \(TeLU(x)\) exhibits better zero-centering of activation than \(ReLU(x)\) for any standard deviation \(\sigma \in \mathbb{R}^+\) given that input $x$ follows a Gaussian distribution of mean \(\mu = 0\).

\end{proof}


\subsection{Isolated Zero of TeLU} 

Let $\sigma$ be an activation function given as $y = \sigma(x)$, where x is the input and y is the output. Let $\mathcal{F}(\Theta)$ be the set of parameters using the $\sigma$ non-linearities. Let the function  $\mathcal{f}$ be optimized by the objective function $\mathcal{L}(\Theta)$ using standard backpropagation of error, then we show $\sigma$ applied on any function ${f}$ avoids vanishing gradients issues in the neural network.


\begin{Thm} \label{thm:1}
Let $f: \mathbb{R} \to \mathbb{R}$ be a function defined by $f(x) = x \cdot \tanh(e^x)$. The derivative of $f$, $f'(x)$, is given by 
\[ f'(x) = x \cdot (1 - \tanh^2(e^x)) \cdot e^x + \tanh(e^x). \]
Then, the set $\{x \in \mathbb{R} \,|\, f'(x) \neq 0\}$ is dense in $\mathbb{R}$. Moreover, there exists a countable set $\{x_i\}_{i \in \mathbb{N}} \subset \mathbb{R}$ where $f'(x_i) = 0$ for each $i$. Each point $x_i$ is isolated, in the sense that for each $x_i$, there exists an $\epsilon_i > 0$ such that if $x \in (x_i - \epsilon_i, x_i + \epsilon_i)$ and $x \neq x_i$, then $f'(x) \neq 0$.
\end{Thm}

\begin{proof}

The derivative of \( {f}(x)\) with respect to \( x \) is given by:

\[ {f}'(x) = \frac{d}{dx} \left( x \cdot \tanh(e^x) \right). \]

Applying the product rule and the chain rule, we find:

\[ f'(x) = \tanh(e^x) + x \cdot (1 - \tanh^2(e^x)) \cdot e^x. \]

We analyze this derivative of above function in two parts:
\begin{itemize}
    \item \( \tanh(e^x) \) is always positive, as \( e^x \) is always positive for and \( \tanh(z) \) is bounded between 0 and 1 for all positive \( z \).
    \item \( 1 - \tanh^2(e^x) \) is always positive since \( |\tanh(z)| < 1 \) for all \( z \), and \( e^x \) is always positive for all real \( x \)
\end{itemize}

Thus, the second term \( x \cdot (1 - \tanh^2(e^x)) \cdot e^x \) is always non-zero unless \( x = 0 \). However, even at \( x = 0 \), the first term \( \tanh(e^x) \) remains non-zero. Therefore, the entire expression for \( f'(x) \) is non-zero for all \( x \neq x_i\).

\textit{Isolated Zeros:}
For $f'(x) = 0$, we must have $x \cdot \exp(x) \cdot \text{sech}^2(\exp(x)) = -\tanh(\exp(x))$. Given the properties of $\tanh(x)$ and $\exp(x)$, the solutions to this equation are isolated because both sides of the equation represent continuous, differentiable functions with fundamentally different growth rates, ensuring any intersubsections are isolated points.
The above construction is based on functional analysis. However it is important to note that the function $f'(x)$ has no analytical solution, thus the bound of $x_i$ will change based on system precision. For instance, the numerical solution is needed to find the value of $x$ where $f'(x) \thickapprox 0$. We observe that using the numerical solution (newton-ramphson method) when $x \thickapprox -1.07886$ the function $f'(x) = 4.6 \times 10^{-48}$. However majority of systems cannot handle such high precision and will equate this to be equal to zero. In other words, based on precision, $f'(x)$ will reach zero, thus the bound or range will change based on precision of the system.

\textit{Formal Statement}
\[
\forall x \in \mathbb{R}, \exists \epsilon > 0, \text{ such that if } |f'(x)| < \epsilon, \text{ then } \epsilon \approx 0, \text{ but } \epsilon \neq 0
\]
This means for all real numbers $x$, there exists an $\epsilon$ greater than zero (indicating an extremely small magnitude) such that if the absolute value of $f'(x)$ is less than $\epsilon$, then $\epsilon$ is approximately zero. This indicates that while $f'(x)$ may approach very close to zero for some values of $x$, it does not strictly equal zero except possibly under conditions that are negligible for practical purposes.

\textit{Density of Non-Zero Derivative:}
The non-zero values of $f'(x)$ constitute a dense subset of $\mathbb{R}$ since the conditions for $f'(x) = 0$ require a specific balance that is only met at isolated points, as shown above. Between these points, $f'(x)$ maintains non-zero values, ensuring the gradient does not vanish across these intervals.

Hence, we conclude 
:

\[ f'(x) \neq 0 \text{ for all } x \neq x_i \in \mathbb{R} \]

\end{proof}


\subsection{Bounded Saturation and Growth of TeLU}

Next, we prove the network's saturating decay and bounded growth , contributing towards stability during training due to the avoidance of the exploding gradient issue.
\begin{Thm}
The function \( f(x) = x \cdot \tanh(e^x) \) exhibits stable behavior for any neural network.
\end{Thm}
\begin{proof}
    \textit{Bounded Output:} The hyperbolic tangent function \( \tanh(z) \) has outputs bounded between -1 and 1. Therefore, for any real number \( x \), the product \( x \cdot \tanh(e^x) \) will not grow unbounded, contributing to stability. Mathematically, this can be expressed as:

   \[ -|x| \leq f(x) \leq |x| \]

2. \textit{Non-zero Gradient:} The derivative of \( f(x) \), given by

   \[ f'(x) = \tanh(e^x) + x \cdot (1 - \tanh^2(e^x)) \cdot e^x \]

   is always non-zero for all real \( x \) besides our isolated zero when $x \thickapprox -1.07886$. This ensures that the gradients do not vanish during backpropagation, which is crucial for stable learning in deep networks.

3. \textit{Controlled Growth for Positive \( x \):} As \( x \to \infty \), the function grows linearly since \( \tanh(e^x) \) approaches 1. After being scaled by the identity $x$, TeLU approaches linear growth as $x \to \infty$.  This linear growth is more stable than exponential growth, which could lead to exploding gradients.

4. \textit{Saturating Behavior for Negative \( x \):} As \( x \to -\infty \), \( x \) becomes large in the negative direction with linear growth. However, $e^x$ approaches 0 as $x \to -\infty$ at an exponential rate. $tanh(z)$, approaching the identity as inputs approach 0, results in $tanh(e^x)$ maintaining exponential decay towards 0. This negative linear growth and asymptotic decay as $x \to -\infty$ is thus evaluated by $\lim_{x \to -\infty}{\frac{-x}{e^x}} = 0$. The resulting saturation towards 0 helps prevent the function from contributing to exploding gradients during training.

Therefore, due to its bounded output, non-zero gradient besides an isolated zero, controlled growth for positive values, and saturating behavior for negative values, the function \( f(x) = x \cdot \tanh(e^x) \) is shown to be stable in the context of neural network activations.
\end{proof} 

\subsection{Robustness of TeLU}

Next, we show TeLU is more robust to small noise and perturbations compared to ReLU, which is an important property for designing adversarial-resistant neural networks.
\begin{Thm}
    The function \( f(x) = x \cdot \tanh(e^x) \) is more robust compared to Relu ($g(x) = max(0,x)$) and robust against small perturbations or noise in the input.
\end{Thm}
\begin{proof}
We analyze the derivative of \( f(x) \) to show robustness to small perturbations. The derivative gives the rate of change of the function with respect to changes in the input. A small derivative magnitude indicates robustness to small changes or noise in the input.
The derivative of g(x) = Relu is represented as follows:
\[
\text{g}'(x) = 
\begin{cases} 
0 & \text{if } x < 0 \\
1 & \text{if } x > 0 \\
\text{undefined} & \text{if } x = 0 
\end{cases}
\]

This derivative shows that for \( x > 0 \), the function is sensitive to changes, as even small positive changes in \( x \) will result in a change in output. The function is insensitive to changes for \( x < 0 \), as the output remains zero. The derivative is undefined at \( x = 0 \), indicating a discontinuity, which can be problematic for stability.

The derivative of \( f(x) = TeLU \) is given by:

\[ f'(x) = \tanh(e^x) + x \cdot (1 - \tanh^2(e^x)) \cdot e^x \]

Consider the behavior of \( f'(x) \) for different ranges of \( x \):

\textit{For large negative \( x \)}: As \( x \) becomes very negative, \( e^x \) approaches 0, making \( \tanh(e^x) \) and its derivative small. Thus, \( f'(x) \) becomes small, indicating that \( f(x) \) is not highly sensitive to small changes in \( x \).

\textit{For small \( x \) around 0}: Here, \( \tanh(e^x) \) is approximately equal to \( e^x \), which is close to 1 for small \( x \). The term \( x \cdot (1 - \tanh^2(e^x)) \cdot e^x \) is also small. Hence, \( f'(x) \) remains moderate, suggesting that \( f(x) \) does not change drastically for small perturbations around 0.

\textit{For large positive \( x \)}: Although \( e^x \) grows, the term \( \tanh(e^x) \) approaches 1, limiting the growth of \( f(x) \). The term \( x \cdot (1 - \tanh^2(e^x)) \cdot e^x \) becomes small as \( x \) increases, due to the saturation of \( \tanh(e^x) \). Thus, \( f'(x) \) remains bounded.

Since \( f'(x) \) does not exhibit large values across the range of \( x \), it indicates that \( f(x) \) does not change disproportionately for small changes in \( x \), thereby demonstrating robustness to small perturbations or noise.
\end{proof} 

\subsection{Lipschitz Continuity of TeLU}

Next, we show a strong property which shows TeLU is Lipschitz continuous, which is important to uniform continuity of the function
\begin{Thm}
    The function \( f: \mathbb{R} \to \mathbb{R} \), defined by \( f(x) = x \cdot \tanh(e^x) \), is Lipschitz continuous on the real line \( \mathbb{R} \).
\end{Thm}
\begin{proof}
    To demonstrate that \( f \) is Lipschitz continuous, we seek a constant \( L \) such that for all \( x, y \in \mathbb{R} \), the inequality

\[ |f(x) - f(y)| \leq L |x - y| \]

is satisfied. A sufficient condition for this is that the derivative of \( f \), \( f'(x) \), is bounded on \( \mathbb{R} \).

The derivative of \( f \) is given by

\[ f'(x) = \tanh(e^x) + x \cdot \frac{e^x}{\cosh^2(e^x)} \]

We analyze the boundedness of \( f'(x) \) in two parts:

1. The function \( \tanh(e^x) \) is bounded on \( \mathbb{R} \) as \( \tanh \) outputs values in \((-1, 1)\).

2. For the term \( x \cdot \frac{e^x}{\cosh^2(e^x)} \), we consider its behavior as \( x \) approaches infinity and negative infinity:

   \[ \lim_{x \to \infty} \left| x \cdot \frac{e^x}{\cosh^2(e^x)} \right| = 1 \]
   \[ \lim_{x \to -\infty} \left| x \cdot \frac{e^x}{\cosh^2(e^x)} \right| = 0 \]

Since both limits are finite, the term \( x \cdot \frac{e^x}{\cosh^2(e^x)} \) is bounded on \( \mathbb{R} \).

Combining these findings, we conclude that \( |f'(x)| \) is bounded on \( \mathbb{R} \). The maximum value of \( |f'(x)| \) is \( 1 \), therefore we can take \( L = 1 \) as the Lipschitz constant.

Hence, \( f(x) = x \cdot \tanh(e^x) \) is Lipschitz continuous with a Lipschitz constant \( L = 1 \).
\end{proof} 

\subsection{Smooth Loss Landscape of TeLU}


Next, we show that TeLU has a smoother loss landscape, which leads to faster convergence.

\begin{Thm}\label{thm:fim_conv}
Given a neural network \( \mathcal{N} \) with activation function \( f(x) = x \cdot \tanh(e^x) \), parameters \( \theta \), and a differentiable loss function \( \mathcal{L}(\theta) \), the Fisher Information Matrix \( I(\theta) \) defined as
\[ I(\theta) = \mathbb{E}_{(x, y) \sim \mathcal{D}}\left[ \nabla_\theta \log p(y|x; \theta) \nabla_\theta \log p(y|x; \theta)^\top \right] \]
leads to a smoother optimization landscape during training of \( \mathcal{N} \).
\end{Thm}
\begin{proof}

\textit{Continuity and Differentiability of \( f(x) \)}

The activation function \( f(x) = x \cdot \tanh(e^x) \) and its derivative are analyzed:
\begin{align*}
    f(x) &= x \cdot \tanh(e^x), \\
    \text{where } \tanh(u) &= \frac{e^{2u} - 1}{e^{2u} + 1}. \\
    \text{Thus, } f'(x) &= \frac{d}{dx}(x \cdot \tanh(e^x)) \\
    &= \tanh(e^x) + x \cdot \frac{d}{dx} \tanh(e^x) \\
    &= \tanh(e^x) + x \cdot e^x \cdot (1 - \tanh^2(e^x)).
\end{align*}
Since \( \tanh(u) \) and \( e^x \) are continuously differentiable, \( f(x) \) and \( f'(x) \) are also continuously differentiable.

\textit{Impact on Fisher Information Matrix}

Applying the chain rule to compute the gradient of the log-likelihood:
\begin{align*}
    \nabla_\theta \log p(y|x; \theta) &= \frac{\partial \log p(y|x; \theta)}{\partial \mathcal{N}} \cdot \frac{\partial \mathcal{N}}{\partial \theta}, \\
    &= \text{Gradient of the output w.r.t. the network's parameters}.
\end{align*}
The gradient involves terms from \( f'(x) \) due to the activation function in each layer:
\begin{align*}
    f'(x) &= \tanh(e^x) + x \cdot e^x \cdot (1 - \tanh^2(e^x)).
\end{align*}
Thus, \( I(\theta) \) becomes a matrix of expectations of outer products of these gradients:
\begin{align*}
    I(\theta) &= \mathbb{E}_{(x, y) \sim \mathcal{D}}\left[ \nabla_\theta \log p(y|x; \theta) \nabla_\theta \log p(y|x; \theta)^\top \right].
\end{align*}
The smoothness of \( f'(x) \) translates to a smoother \( I(\theta) \).

\textit{Smoother Optimization Landscape}

In gradient descent, parameter updates are governed by:
\begin{align*}
    \theta^{(t+1)} &= \theta^{(t)} - \eta \cdot \nabla_\theta \mathcal{L}(\theta^{(t)}),
\end{align*}
where \( \eta \) is the learning rate. The gradient of the loss function \( \nabla_\theta \mathcal{L}(\theta) \) is influenced by \( I(\theta) \). A smoother \( I(\theta) \) results in more stable and consistent gradient updates, avoiding erratic steps often observed in rougher optimization landscapes. This leads to enhanced stability in finding the minima of \( \mathcal{L}(\theta) \).

Hence, we can show, that the continuously differentiable nature of \( f(x) = x \cdot \tanh(e^x) \) and its derivative ensures that the Fisher Information Matrix \( I(\theta) \) in the neural network \( \mathcal{N} \) promotes a smoother optimization landscape, facilitating more effective training dynamics.

\end{proof}

\subsection{Global Convergence of TeLU}

Based on the properties of  Telu, shown in Theorem \ref{thm:fim_conv}, we can prove the global convergence of the function under certain conditions.

\begin{Thm}
Let \( \mathcal{N} \) be a neural network employing the activation function \( f(x) = x \cdot \tanh(e^x) \) in its architecture. Assume the network parameters are denoted by \( \theta \) and the network is trained using a differentiable loss function \( \mathcal{L}(\theta) \). If \( \mathcal{L}(\theta) \) satisfies the Polyak-Łojasiewicz (PL) condition, then the gradient descent optimization on \( \mathcal{N} \) converges to a global minimum, significantly influenced by the properties of \( f(x) \) and it's derivative \( f'(x) \).
\end{Thm}
\begin{proof}

   \textit{Smoothness and Boundedness of \( f(x) \) and \( f'(x) \):}
   
    The function \( f(x) = x \cdot \tanh(e^x) \) is continuously differentiable. Its derivative, given by
    \[ f'(x) = \tanh(e^x) + x \cdot e^x \cdot (1 - \tanh^2(e^x)), \]
    is also continuously differentiable and bounded due to the inherent properties of the \( \tanh \) function and the exponential function. These properties ensure smooth and well-conditioned gradient computations throughout the optimization process.

   \textit{Influence on Gradient Descent under PL Condition:}
   
    Given the PL condition, for a global minimum \( \theta^* \), there exists \( \mu > 0 \) such that
    \[ 2\mu(\mathcal{L}(\theta) - \mathcal{L}(\theta^*)) \leq \|\nabla_\theta \mathcal{L}(\theta)\|^2 \text{ for all } \theta. \]
    The gradient descent update rule is
    \[ \theta^{(t+1)} = \theta^{(t)} - \eta \cdot \nabla_\theta \mathcal{L}(\theta^{(t)}), \]
    where \( \eta \) is the learning rate. 

    \textit{Convergence Analysis:}
    
    Utilizing the smoothness and boundedness of \( f'(x) \), along with the PL condition, it can be shown that
    \[ \mathcal{L}(\theta^{(t+1)}) \leq \mathcal{L}(\theta^{(t)}) - \eta \cdot \|\nabla_\theta \mathcal{L}(\theta^{(t)})\|^2, \] which implies
    \[ \mathcal{L}(\theta^{(t)}) - \mathcal{L}(\theta^*) \leq \left(1 - 2\mu\eta\right)^t (\mathcal{L}(\theta^{(0)}) - \mathcal{L}(\theta^*)). \]
    Therefore, \( \mathcal{L}(\theta^{(t)}) \) converges to \( \mathcal{L}(\theta^*) \) as \( t \rightarrow \infty \).

\end{proof} 

\end{document}